\documentclass{article}

\usepackage{amsmath,amsfonts,bm}

\def\eqref#1{equation~\ref{#1}}

\def\1{\bm{1}}

\DeclareMathAlphabet{\mathsfit}{\encodingdefault}{\sfdefault}{m}{sl}
\SetMathAlphabet{\mathsfit}{bold}{\encodingdefault}{\sfdefault}{bx}{n}

\newcommand{\R}{\mathbb{R}}

\usepackage{microtype}
\usepackage{graphicx}
\usepackage{subcaption}
\usepackage{booktabs} %

\usepackage{hyperref}

\usepackage[accepted]{icml2026}

\usepackage{amsmath}
\usepackage{amssymb}
\usepackage{mathtools}
\usepackage{amsthm}

\usepackage[capitalize,noabbrev]{cleveref}

\usepackage{booktabs}       %
\usepackage{amsfonts}       %
\usepackage{nicefrac}       %
\usepackage{microtype}      %
\usepackage{xcolor}         %
\usepackage[table]{xcolor}
\usepackage{pifont}

\usepackage{wrapfig}
\usepackage{multirow}

\newcommand{\xmarkk}{\ding{55}}%
\newcommand*\colourcheck[1]{%
  \expandafter\newcommand\csname #1check\endcsname{\textcolor{#1}{\ding{52}}}%
}

\newcommand*\colourxmark[2]{%
  \expandafter\newcommand\csname #2check\endcsname{\textcolor{#2}{\ding{55}}}%
}

\newcommand{\xmark}{\textcolor{red}{\xmarkk}}

\definecolor{g}{rgb}{0.0, 0.5, 0.0}
\colourcheck{g}

\usepackage{listings}
\usepackage{adjustbox}
\usepackage{rotating}

\definecolor{codegreen}{rgb}{0,0.6,0}
\definecolor{codegray}{rgb}{0.5,0.5,0.5}
\definecolor{codepurple}{rgb}{0.58,0,0.82}
\definecolor{backcolour}{rgb}{0.95,0.95,0.92}

\newcommand{\githubrepo}[2]{%
  \href{#2}{%
    \includegraphics[height=1em]{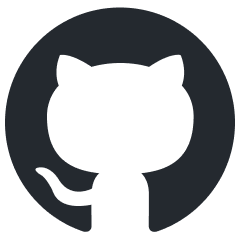}\hspace{0.3em}\texttt{#1}%
  }%
}

\lstdefinestyle{mystyle}{
    backgroundcolor=\color{backcolour},   
    commentstyle=\color{codegreen},
    keywordstyle=\color{magenta},
    numberstyle=\tiny\color{codegray},
    stringstyle=\color{codepurple},
    basicstyle=\ttfamily\footnotesize,
    breakatwhitespace=false,         
    breaklines=true,                 
    captionpos=b,                    
    keepspaces=true,                 
    numbers=left,                    
    numbersep=5pt,                  
    showspaces=false,                
    showstringspaces=false,
    showtabs=false,                  
    tabsize=2
}

\usepackage{pythonhighlight}
\usepackage{tcolorbox} %
\usepackage{caption}
\usepackage{subcaption}
\usepackage{cancel}
\usepackage{lipsum}

\crefname{figure}{Figure}{Figures}
\crefname{table}{Table}{Tables}
\crefname{section}{Section}{Sections}
\crefname{appendix}{Appendix}{Appendices}

\newtcolorbox{greenbox}[1][]{colback=green!4, colframe=black, boxrule=0.3pt, boxsep=-4pt #1}
\newtcolorbox{whitebox}[1][]{colback=white!4, colframe=black, boxrule=0.3pt, boxsep=0.1pt, left=0.1pt, right=2pt, top=0pt, bottom=0pt, #1}
\newcommand{\ours}{\textsc{Violin }}

\usepackage[colorinlistoftodos,prependcaption,textsize=tiny,textwidth=20mm]{todonotes}
\usepackage{soul}

\usepackage{thmtools, thm-restate}
\theoremstyle{plain}
\newtheorem{theorem}{Theorem}[section]
\theoremstyle{definition}
\newtheorem{definition}[theorem]{Definition}
\theoremstyle{definition}
\newtheorem{claim}[theorem]{Claim}

\theoremstyle{plain}

\theoremstyle{definition}

\theoremstyle{remark}

\crefname{equation}{equation}{equations}
\Crefname{equation}{Equation}{Equations}

\usepackage{dblfloatfix}

\icmltitlerunning{Spatial Priors via Space Filling Curves for Small and Limited Data Vision Transformers}

\begin{document}

\twocolumn[
  \icmltitle{Spatial Priors via Space Filling Curves for Small and Limited Data Vision Transformers}

  \icmlsetsymbol{equal}{*}

  \begin{icmlauthorlist}
    \icmlauthor{Leyla Naz Candogan}{equal,lions}
    \icmlauthor{Arshia Afzal}{equal,lions}
    \icmlauthor{Pol Puigdemont}{lions}
    \icmlauthor{Volkan Cevher}{lions}

  \end{icmlauthorlist}

  \icmlaffiliation{lions}{LIONS, \'Ecole Polytechnique F\'ed\'erale de Lausanne (EPFL), Lausanne, Switzerland}

  \icmlcorrespondingauthor{Leyla Naz Candogan}{leyla.candogan@epfl.ch}

  \icmlkeywords{Spatial priors, ViT, ICML}

  \vskip 0.3in
]

\printAffiliationsAndNotice{}  %

\begin{abstract}
Though Vision Transformers (ViTs) have become the dominant backbone in many computer vision tasks, due to permutation equivariance, their attention mechanism lacks explicit spatial inductive biases. This become particularly important in two settings: when model capacity is small or training data is limited. Inspired by the attention masking strategies in Linear Transformers and the scanning patterns of Vision SSMs, we introduce \ours, a lightweight masked attention mechanism that encodes spatial structure within attention via Space Filling Curves (SFCs) with less than $0.0015\%$ extra parameters and negligible computational overhead. \ours scans the image using multiple SFCs to construct curve-specific decay masks, which are then combined and multiplied with the attention matrix. Across a wide range of evaluations, \ours consistently improves performance. In limited data regimes such as fine-tuning on VTAB-1K, it boosts accuracy across all task groups and by up to $8.7\%$ on the tasks where spatial information is essential. It can be combined with parameter-efficient fine-tuning methods such as LoRA to further increase the performance. Beyond fine-tuning, \ours improves various small scale ViT architectures (e.g., DeiT, DINO) during pretraining on ImageNet-1K. Additionally, on pixel-level CIFAR-100 training, a task that is highly dependent on location information, \ours increases accuracy by up to $7.2\%$. Overall, \ours provides a computationally efficient yet effective way to inject spatial inductive bias into ViTs, especially benefiting small models and limited data settings.

\vspace{-2mm}
\begin{center}
    \githubrepo{\textbf{\ours Code}}
    {https://github.com/LIONS-EPFL/VIOLIN}
\end{center}
\end{abstract}

\section{Introduction}
Vision Transformers (ViTs)~\citep{vit} have rapidly become a dominant architecture in computer vision, achieving strong performance across tasks by capturing global dependencies through self-attention. However, unlike Convolutional Neural Networks (CNNs)~\citep{lecun1998gradient}, ViTs lack inherent \textit{spatial priors} such as locality~\citep{fan2024rmt}. This limitation partially comes from the permutation equivariance of attention, which treats image patches as an unordered set of tokens. As a result, ViTs become \textit{data-hungry} and \textit{dependent on larger model sizes}.\footnote{We define models with $\leq30$M parameters as small-scale and those with $\sim$86M+ as large scale.} While large models and datasets allow ViTs to learn these biases directly~\citep{lu2022bridging,sun2017jft}, many downstream tasks require adapting a pretrained backbone with limited data. In such cases, \textit{even large ViTs struggle to specialize}, making stronger inductive biases essential across scales. Prior works tried to address this limitation with convolutions~\citep{guo2022cmt}, novel positional encodings~\citep{wu2021rethinking}, or masking strategies~\citep{fan2024rmt}.

\begin{figure*}[t]
    \centering
    \includegraphics[width=0.9\linewidth]{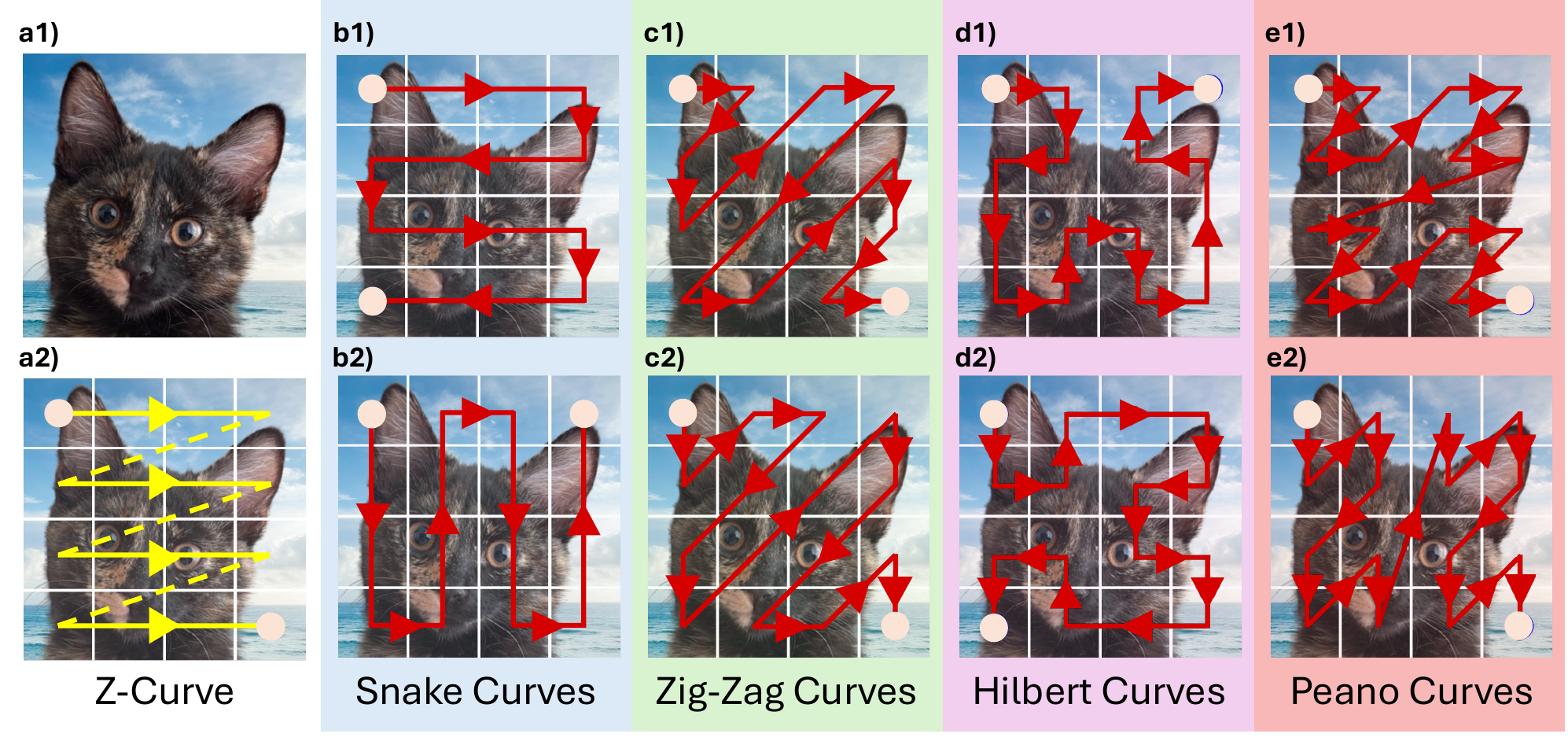}
    \caption{\textit{Space Filling Curve paths:} Examples of traversal paths used in \ours on a $4\times4$ patched image. \textbf{(a1)} Original image. \textbf{(a2)} Z-curve 
        \textbf{(b1)} Snake curve, 
    \textbf{(b2)} Transposed Snake curve, 
        \textbf{(c1)} Zig-zag curve, 
    \textbf{(c2)} Transposed Zig-zag curve, 
        \textbf{(d1)} Hilbert curve, 
    \textbf{(d2)} Transposed Hilbert curve, 
        \textbf{(e1)} Peano curve, 
    \textbf{(e2)} Transposed Peano curve.}
    \label{fig:golge_curves}
\end{figure*}

Concurrently, in natural language processing, State Space Models (SSMs) and Linear Transformers have emerged as efficient alternatives to standard transformers~\citep{mamba,mamba2,retnet}, and their vision adaptations have achieved strong results~\citep{alkin2024visionlstm,liu2024vmamba,zhu2024visionmambaefficientvisual}. Through recurrence and a decay factor on attention scores, these models can capture the relative spatial order of image patches. However, this information depends entirely on the chosen scanning order, and to capture both vertical and horizontal relations, they typically require multiple directional scans~\citep{li2024videomamba}.  

Scanning an image converts its 2D patch layout into a 1D sequence, with the order of patches determined by a traversal path. This process can be viewed as a Space Filling Curve (SFC): a continuous path that passes through every point in a multidimensional grid while systematically covering the entire image~\citep{sagan1994space}. Many vision backbones, including vanilla ViT~\citep{vit}, Vision x-LSTM~\citep{alkin2024visionlstm}, and Vim~\citep{zhu2024visionmambaefficientvisual}, use the simple Z-curve, or row-by-row scan, for this linearization (see \cref{fig:golge_curves} \textbf{(a2)}). Given that other SFCs, such as Snake, Zig-zag, Peano, and Hilbert curves, preserve locality in different ways, we ask the following question:
\begin{center}
\textit{Can SFC-inspired structure in attention enhance the spatial understanding of ViTs and improve their performance in small models and data scarce settings?}
\end{center}
In this work, we answer this question affirmatively by introducing \ours \footnote{As a subtle homage to Giuseppe Peano, the creator of space filling curves, we named our model in a way that also reflects a musical instrument, just like Peano’s family name resembles “Piano”.}, a lightweight attention mechanism for global attention models that injects spatial priors via SFC-guided decay masks. \ours integrates multiple SFC-based scans into a single mask, $\mathbf{M}_\text{\ours}$, capturing relative patch locations \textit{without modifying the rest of the architecture}. This yields an efficient, plug-and-play way to introduce locality into ViTs, \textit{particularly benefiting small models and data scarce regimes}. \cref{fig:golge_curves} \textbf{(b - e)} shows the SFCs used in \ours, with their linearized sequences in \cref{fig:flat_golge}.

We evaluate \ours across a broad set of settings:  
\vspace{-2mm}
\begin{itemize}
    \item Fine-tuning DeiT, DeiT-III, and DINO~\citep{deit,touvron2022deit,dino} on VTAB~\citep{vtab}, across scales \textit{from Tiny (5M) to Huge (632M)}, where \ours consistently improves baselines by up to \textbf{8.7\%} on individual tasks and \textbf{4.7\%} on average. \ours also combines seamlessly with parameter-efficient fine-tuning methods, further boosting adaptability.  
    \vspace{7mm}
    \item Pretraining small-scale models on ImageNet-1K~\citep{russakovsky2015imagenet} increases the performance by up to \textbf{0.9\%}, and on pixel-level CIFAR-100~\citep{cifar}, achieves a notable \textbf{7.2\%} improvement.  
    \item Additional analyses, including the complementary roles of different curves, performance on the Structured group, and extensions to dense prediction tasks such as object detection on COCO~\citep{lin2015microsoftcococommonobjects} and semantic segmentation on ADE20K~\citep{zhou2017scene}, further highlight the versatility of \ours and the importance of explicit spatial priors.  
\end{itemize}

\section{Background}
\textbf{Notations and preliminaries} \label{notation}
We denote a patched image as $\mathcal{I} \in \R^{H \times W \times d}$, where $H$ and $W$ are the number of patches along height and width, and $d$ is the embedding dimension. Its flattened form is $\mathbf{X} \in \R^{N \times d}$ with $N = H \times W$ as the sequence length. For single head attention, the query, key, and value matrices $\mathbf{Q}, \mathbf{K}, \mathbf{V} \in \R^{N \times d}$ are computed using learnable weights $\mathbf{W}_Q,\mathbf{W}_K,\mathbf{W}_V \in \R^{d \times d}$, and the standard ViT attention is computed as
\begin{equation}\label{eq:qkv_attention}
\begin{split}
\mathbf{Q} = \mathbf{X}\mathbf{W}_Q,\; 
\mathbf{K} = \mathbf{X}\mathbf{W}_K,\; 
\mathbf{V} = \mathbf{X}\mathbf{W}_V, \\
\mathbf{Y} = \text{Softmax}\!\left(\frac{\mathbf{Q}\mathbf{K}^\top}{\sqrt{d}}\right)\mathbf{V}.
\end{split}
\end{equation}
where $\mathbf{Y} \in \R^{N \times d}$ is the attention output. We use $h$ and $L$ for the number of attention heads and transformer layers respectively. Elements of matrices and vectors are accessed by $[\cdot]$, and $\odot$ denotes the Hadamard product. A full list of notations is provided in \cref{ap:notation}.

\textbf{Vision Transformers and spatial priors}
After dividing an image into patches (tokens), ViTs process them as a 1D sequence, typically flattened with a Z-curve~\citep{vit}, which discards information about neighboring patches. To reintroduce spatial information, most ViTs add positional embeddings before transformer blocks. Recent works have further improved performance through self-supervised learning (e.g., DINO~\citep{dino}) and optimized training strategies (e.g., DeiT and DeiT-III~\citep{deit,touvron2022deit}). In this study, we show how \ours improves upon these models and training recipes.

By processing patches independently, ViTs lack the strong spatial inductive bias of architectures like CNNs, which inherently encode locality~\citep{yuan2021t2t}. Although ViTs capture global interactions, they struggle with fine-grained local structures, making training data-hungry~\citep{d2021convit}. Sufficiently large models and datasets can mitigate this by learning locality from data, but when model size or data is limited, ViTs struggle to achieve strong performance~\citep{lu2022bridging}, see \cref{bckg:vit} for details.

\textbf{Linear Transformers}
Linear attention was introduced as an alternative to softmax attention, reducing quadratic complexity to linear time via a recurrent formulation \cref{linatt}~\citep{trans_rnn}. 
Instead of relying on positional embeddings to capture the order within a sequence, most modern Linear Transformers~\citep{retnet} incorporate a decay factor ($\textcolor{black}{\gamma}$), 
\begin{equation} \label{linatt}
\mathbf{S}_i = \textcolor{black}{\gamma}\mathbf{S}_{i-1} + \mathbf{k}_i^\top\mathbf{v}_i,\quad 
\mathbf{y}_i = \mathbf{q}_i^\top\mathbf{S}_{i} \end{equation}
\begin{equation} \label{linatt_matrix}
\hspace{-2mm}
\scalebox{0.92}{$
\mathbf{Y} = (\mathbf{Q}\mathbf{K}^\top \odot \mathbf{M}_{\text{causal}})\mathbf{V},\;
\mathbf{M}_{\text{causal}}[i,j] =
\begin{cases}
\gamma^{i-j} & i \ge j, \\
0 & i < j.
\end{cases}$}
\end{equation}
where $\mathbf{S}_i \in \mathbb{R}^{d \times d}$ is the hidden state. This recurrent form can be parallelized using matrix multiplication with a Toeplitz decay mask $\mathbf{M}$~\citep{tnn,retnet} as in \cref{linatt_matrix}. Though linear masked attention was initially proposed for causal NLP tasks, it is later adapted to non-causal tasks using full Toeplitz masks~\citep{lion}. The decay mask naturally extends context length, supports variable sequence lengths, and  provides locality information that inspired \ours. %

\textbf{Scans in Linear Vision Transformers and SSMs}
Linear Transformers and SSMs have been applied to vision tasks~\citep{alkin2024visionlstm, liu2024vmamba, zhu2024visionmambaefficientvisual, ren2025autoregressive, hu2024zigma, zhang2024surveyvisualmamba}. To enhance spatial representation, these models often traverse image patches using a Z-curve, typically scanning in both vertical and horizontal directions. Each scan acts as a separate recurrence, capturing distinct spatial patterns through their own decay factors.
\newpage
\textbf{Space Filling Curves} 
\begin{definition} \label{def:sfc}
A \textit{Space Filling Curve (SFC)} is a continuous mapping from a closed unit interval \( S = [0,1] \) to a closed unit hypercube \( Q = [0,1]^N \), passing through every point in \(Q\) exactly once~\citep{peano1990courbe}. In this work, we focus on the 2D Euclidean case \( Q = [0,1]^2 \), corresponding to images.
\end{definition}
Based on \cref{def:sfc}, many SFCs can been defined, including the \textbf{Snake}, \textbf{Peano} (also known as the Morton curve)~\citep{peano1990courbe}, \textbf{Hilbert}~\citep{hilbert1935stetige}, \textbf{Z} (or Sweep), and \textbf{Zig-zag}~\citep{zigzag} curves as illustrated in \cref{fig:golge_curves}. Additionally, other curves include the Sierpinski \citep{sierpinski1915curve}, and Lebesgue curves \citep{lebesgue1904}.

Flattening or scanning can be viewed as applying an SFC \(c\) to a 2D patched image \(\mathcal{I}\) with \(N\) total patches, mapping it to a 1D sequence \(\mathbf{X}_c \in \mathbb{R}^{N}\) via a flattening function  \( F_c(\mathcal{I}): \mathbb{R}^{H \times W} \mapsto \mathbb{R}^{N} \) such that
\begin{equation}\label{flat_w_curve}
\hspace{-2mm}
\scalebox{0.92}{$
F_c(i,j): (i,j) \mapsto n, \;
i \in \mathbb{Z}_{[0,H)},\;
j \in \mathbb{Z}_{[0,W)},\;
n \in \mathbb{Z}_{[0,N)}. $}
\end{equation}
\begin{equation} \label{flat_w_curve_2}
\hspace{-2mm}
\scalebox{0.92}{$
\mathbf{X}_c = F_c(\mathcal{I}),\quad \mathbf{X}_c[n] = \mathcal{I}[i,j]\;\; \text{where}\;\; n=F_c(i,j). $}
\end{equation}
This flattening can be applied independently across each dimension \(d\) for \(\mathcal{I} \in \mathbb{R}^{H \times W \times d}\). While SFCs 
have diverse applications in other domains, their role in image classification remains underexplored~\citep{zhao2024rethinkingzigzagflatteningimage, kutscher2025REOrder}. Additional details are provided in \cref{ap:sfc}.

\section{Methodology}
In this section, we introduce decay-masked attention (\cref{decmask}), extend it to diverse scanning patterns (\cref{sfcprinc,meet}), and finally present \ours attention (\cref{ours}).
\subsection{Attention with Decay Mask} \label{decmask}
As shown in \cref{permutation_equivariance}, attention (\cref{eq:qkv_attention}) is permutation equivariant: changing the order of tokens in the sequence results in the same reordering in the output. Therefore, standard attention does not encode relative spatial priors within an image. To introduce locality, we take inspiration from Linear Transformers and multiply a decay mask to the attention to break this equivariance:
\begin{equation}\label{newattention}
\begin{split}
\mathbf{Y} = \text{Softmax}\!\left(\frac{\mathbf{Q}\mathbf{K}^\top}{\sqrt{d}} \odot \mathbf{M}\right)\mathbf{V},\\ %
\mathbf{M}[i,j] = \gamma^{|i-j|},\quad 0 < \gamma \le 1.
\end{split}
\end{equation} 
This decay mask $\mathbf{M}$, also known as the Kac–Murdock–Szegö matrix~\citep{kms}, extends the causal decay mask to full attention~\citep{lion}. It dampens the attention score between tokens $i$ and $j$ by $\gamma^{|i-j|}$, enforcing locality in the flattened sequence $\mathbf{X}$. However, both the token order in $\mathbf{X}$ and the notion of distance in $\mathbf{M}$ depend entirely on how the original image $\mathcal{I}$ is flattened. This raises a natural question: \textit{What are alternative ways to flatten an image?}
\subsection{SFCs as Principled Way of Image Flattening} \label{sfcprinc}
Following \cref{flat_w_curve_2}, scanning an image along a path $c$ yields the sequence $\mathbf{X}_c = F_c(\mathcal{I})$. Many ViTs use the Z-curve as the default scanning method.

\textbf{Z-Curve}
The Z-curve, also called sweep, row-major order, or raster scan, traverses the image row by row, top to bottom, and left to right within each row and defined by $F_z(i,j) = iW + j$. See \cref{ap:sfc} for other curves used in \ours

Although flattening with different curves usually requires reprocessing the image, we propose a simpler and significantly more efficient alternative: \textit{applying a permutation to the flattened sequence}.

\textbf{Permutation of a flattened image} 
Given a sequence $\mathbf{X}_{c_1}$ flattened via SFC $c_1$, and noting that flattening is one-to-one, we define a permutation $\pi_{c_1 \mapsto c_2}:\mathbb{Z}_{[0,N)}\!\mapsto\!\mathbb{Z}_{[0,N)}$ that maps it to $\mathbf{X}_{c_2}$ from curve $c_2$
\begin{equation} \label{perm_function}
\mathbf{X}_{c_2} = \pi_{c_1 \mapsto c_2}(\mathbf{X}_{c_1}).
\end{equation}
Note that since each index in \( \mathbf{X}_{c_1} \) uniquely corresponds to one in \( \mathbf{X}_{c_2} \), $\pi_{c_1 \to c_2}$ is invertible. Alternatively, we can represent it as a permutation matrix $\mathbf{P}_{c_1 \mapsto c_2} \in \{0,1\}^{N \times N}$
\begin{align} \label{perm_matrix}
\mathbf{P}_{c_1 \mapsto c_2}[n,m] &=
\begin{cases}
1 & \text{if } m = \pi_{c_1 \mapsto c_2}(n), \\
0 & \text{otherwise},
\end{cases}
\notag \\ %
\mathbf{X}_{c_2} &= \mathbf{P}_{c_1 \mapsto c_2}\mathbf{X}_{c_1}.
\end{align}
Since $\mathbf{P}_{c_1 \mapsto c_2}$ is a permutation matrix,
\begin{equation}
\mathbf{P}_{c_2 \mapsto c_1}
= \mathbf{P}_{c_1 \mapsto c_2}^{-1}
= \mathbf{P}_{c_1 \mapsto c_2}^\top .
\end{equation}
Thus, by flattening the image once using the Z-curve, it is possible to obtain
$\mathbf{X}_c$ for other curves by applying $\pi_{z\mapsto c}(\cdot)$.
\subsection{SFCs Meet Attention} \label{meet}
With the naive approach, using $\mathbf{X}_{c}$ for each curve individually and following \cref{newattention}, the output of masked attention $\mathbf{Y}_c$ can be calculated such that
\begin{equation} \label{naive_reorder}
     \mathbf{Y}_c = \text{Softmax}\left(\frac{\mathbf{Q}_c \mathbf{K}_c^\top}{\sqrt{d}} \odot \mathbf{M}_c\right)\mathbf{V}_c, \;%
     \mathbf{M}_c[i,j]  = \gamma_c^{|i-j|},
\end{equation}
where $\mathbf{Q}_c,\mathbf{K}_c,\mathbf{V}_c$ are calculated with $\mathbf{X}_{c}$. As the token order of $\mathbf{Y}_c$ depends on the curve $c$, when multiple curves are used, the outputs (e.g $\mathbf{Y}_{c_1}$ and $\mathbf{Y}_{c_2}$) will have mismatched positions. To overcome this issue we define a basis curve.

\textbf{Basis Curve}
After computing the attention output $\mathbf{Y}_c$ for each curve $c$, we permute them into a common basis to align all outputs. This preserves the spatial locality of each curve while ensuring they share a consistent reference order. Following standard ViT flattening, we use the Z-curve as the basis and perform all permutations relative to it, simplifying notation as $\pi_{z \mapsto c} = \pi_c$, $\pi_{c \mapsto z} = \pi_c^{-1}$ and $\mathbf{P}_{z \mapsto c} = \mathbf{P}_c$, $\mathbf{P}_{c \mapsto z} = \mathbf{P}_c^{-1}$. The output aligned to the basis is
\begin{align}
\widetilde{\mathbf{Y}_c} = \pi_c^{-1}(\mathbf{Y}_c) = \mathbf{P}_c^\top \mathbf{Y}_c.
\end{align}
\textbf{Permutation of Decay Mask}
The aligned output $\widetilde{\mathbf{Y}_c}$ of the masked attention in \cref{naive_reorder} is
\begin{equation}
\widetilde{\mathbf{Y}_c} = \mathbf{P}_c^\top\mathbf{Y}_c 
= \mathbf{P}_c^\top\text{Softmax}\!\left(\frac{\mathbf{Q}_c\mathbf{K}_c^\top}{\sqrt{d}} \odot \mathbf{M}_c\right)\mathbf{V}_c.
\end{equation}
Equivalently, we can permute the decay mask $\mathbf{M}_c$ to the basis order as $\widetilde{\mathbf{M}_c} = \pi_c^{-1}(\mathbf{M}_c) = \mathbf{P}_c^\top \mathbf{M}_c \mathbf{P}_c$, allowing attention to be computed directly in the basis, see \cref{Mc_equality} for proof. The attention output then becomes
\begin{equation} \label{viatt_pre}
\begin{split} 
\widetilde{\mathbf{Y}_c} = \text{Softmax}\!\left(\frac{\mathbf{QK}^\top}{\sqrt{d}} \odot \widetilde{\mathbf{M}_c}\right)\mathbf{V}, \\
\widetilde{\mathbf{M}_c} = \pi_c^{-1}(\mathbf{M}_c),\quad \mathbf{M}_c[i,j] = \gamma_c^{|i-j|}.
\end{split}
\end{equation}
This approach is more efficient than the naive one, as $\mathbf{Q},\mathbf{K},\mathbf{V}$ are computed only once with the basis curve, and, crucially, \textit{a single} $\mathbf{Q}\mathbf{K}^\top \in \mathbb{R}^{N\times N}$ is shared across all.
\subsection{\ours Attention} \label{ours}
For a single head, we define \ours attention as a decay-masked attention guided by multiple SFCs.
\begin{figure}[h] \centering \vspace{-1mm} 
\begin{minipage}{0.45\textwidth} 
\begin{tcolorbox}[colback=orange!30, colframe=orange!50, boxrule=0.2mm, arc=0mm ,boxsep=0.1pt, left=2pt, right=6pt, top=-8pt, bottom=2pt] \begin{align} \label{violinattention} \mathbf{Y} = \text{Softmax}&\left(\alpha \frac{\mathbf{QK}^\top}{\sqrt{d}} \odot \mathbf{M}_\text{\ours} \right) \mathbf{V}, \notag \\ \mathbf{M}_\text{\ours} &= \frac{1}{|\mathcal{C}|} \sum_{c \in \mathcal{C}} \widetilde{\mathbf{M}_c}. \end{align} \end{tcolorbox}
\end{minipage}
\vspace{-3mm}
\end{figure}

Here, $\mathbf{M}_\text{\ours}$ is the average of decay masks from all curves $c \in \mathcal{C}$, each first aligned to the basis (Z-curve) order. The matrices $\mathbf{Q},\mathbf{K},\mathbf{V}$ are computed from the input $\mathbf{X}$ flattened with respect to the basis. The learnable scalar $\alpha \in \mathbb{R}$ controls how strongly the mask influences attention. 

For \ours, we use Snake, Zig-zag, Peano, and Hilbert curves together with their transposed variants (\cref{fig:golge_curves} \textbf{(b2-e2)}) to capture diverse scanning patterns in both row and column major order. This gives the curve set
\begin{multline}
\mathcal{C} = \{\text{Snake, Zig-Zag, Peano, Hilbert,} \\ \text{Snake$^{\top}$, Zig-Zag$^{\top}$, Peano$^{\top}$, Hilbert$^{\top}$}\}.
\end{multline}

Each curve $c$ has a decay factor $\gamma_c \in [0,1]$ for its mask $\mathbf{M}_c$, parameterized as $\gamma_c=\text{sigmoid}(\beta_c)$ with learnable $\beta_c \in \mathbb{R}$ for stability~\citep{lru}. In multi-head attention, each head $k$ has its own $\beta_c^k$ and $\alpha^k$, yielding head specific masks $\mathbf{M}_c^k$ and $\mathbf{M}_\text{\ours}^k$. 

In practice, permutations can be applied efficiently via indexing, see code in \cref{curve_codes}. The full \ours block is shown in \cref{fig:vilblock}, with theoretical motivation for averaging in \cref{ap:sfc_avg}, further design choices and ablations in \cref{sec:design} and \cref{ap:ablations}.

\begin{figure}[t]
\centering
\includegraphics[width=0.9\linewidth]{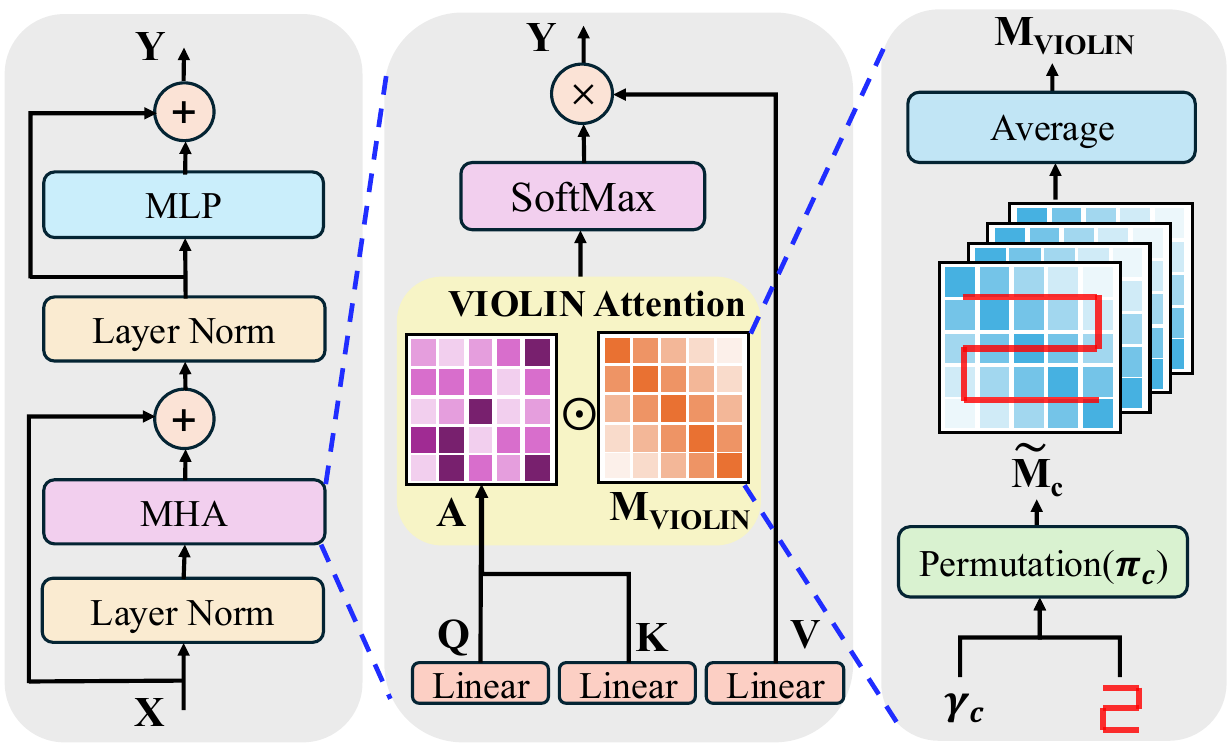}
\caption{\ours: \textbf{(Left)} ViT block with \ours multi-head attention. \textbf{(Middle)} Single-head \ours attention. \textbf{(Right)} Decay mask $\mathbf{M}_\text{\ours}$ formed by averaging masks from curves in $\mathcal{C}$.}
\label{fig:vilblock}
\end{figure}
\begin{table}[b]
    \centering
    \caption{\textit{Parameter and computational overhead of \ours:} calculated relative to DeiT-B (86M parameters, 55.4G FLOPs). }
    \begin{tabular}{lcc}
        \noalign{\hrule height 1.5pt}
        \multirow{2}{4em}{Metric} &{Theoretical} & {\% Change} \\ 
        &{Computation} & {(over DeiT-B)} \\ \hline
        $\#$ Param. & $Lh(|\mathcal{C}|+1)$ & $\textbf{0.0015\%}$ \\ 
        FLOPs & $\mathcal{O}(LhdN^2)$ & $\textbf{0.64\%}$ \\ 
        \noalign{\hrule height 1.5pt}
    \end{tabular}%
    \label{tab:comp_overhead}
\end{table}

\textbf{Parameter and computational overhead}
A key advantage of \ours is its minimal parameter and computational overhead. As shown in \cref{tab:comp_overhead}, \ours adds only \textbf{0.0015\% parameters} and \textbf{0.64\% FLOPs} compared to the baseline DeiT-B model which is effectively negligible in practice. 

We further evaluate GPU memory and inference runtime using a DeiT-S backbone with a batch size of 256 at resolutions of $224\times224$  (classification) and $512\times512$ (dense prediction). As shown in \cref{tab:gpu_time}, \ours closely matches vanilla DeiT model in both runtime and memory usage, confirming the minimal overhead predicted by our analysis.

\section{Experiments}
We evaluate \ours across diverse settings to assess its impact on ViTs’ spatial awareness. Experiments include fine-tuning on small datasets (\cref{sec:finetune}), pretraining small-scale models on ImageNet-1K and pixel-level CIFAR-100 (\cref{sec:pretrain}), and ablations on curve configurations and decay factors (\cref{sec:extra}). Beyond classification, we analyze gains on the Structured VTAB group and extend evaluation to dense prediction tasks such as detection and segmentation. Overall, \ours consistently improves performance, particularly for small models and data-scarce regimes.
\subsection{VTAB-1K Fine-tuning} \label{sec:finetune}
The Visual Task Adaptation Benchmark (VTAB)~\citep{vtab} evaluates the adaptability of learned representations to diverse unseen tasks with limited data. It consists of three groups, Natural, Specialized, and Structured, covering 19 datasets from varied domains and semantic categories. In our experiments we use VTAB-1K, a subset with 1,000 examples per task, specifically designed to test model adaptation in data-scarce settings. 

On small datasets, we test \ours under two configurations: full and parameter-efficient fine-tuning (PEFT). In both cases, we compare fine-tuning results of the original pretrained models (\colorbox{gray!20}{Baseline}), and \colorbox{orange!30}{$\text{Baseline}\odot \mathbf{M}_\text{\ours}$} where pretrained models are combined with freshly initialized mask before fine-tuning and then optimized jointly with the backbone during fine-tuning. For all models, baselines and \ours, we use the fine-tuning implementation from \cite{vtab1k-pytorch} as described in \cref{ap:vtab_hyper} with the complete set of training hyperparameters in \cref{tab:vtab_hyper}, and per-dataset results in \cref{ap:vtab}.
\begin{table}[t]
\centering
\caption{\textit{GPU memory and inference time comparison:} for DeiT-S and \ours-S at different input resolutions. Measurements done on the same hardware with batch size 256.}
    \begin{tabular}{lcc}
    \noalign{\hrule height 1.5pt}
    \multirow{2}{4em}{Model} & GPU Memory & Runtime \\
     & (GB) & (ms/batch) \\
    \midrule
    \rowcolor{gray!20} DeiT ($224\times224$)   & 0.80 & 206.1 \\
    \rowcolor{orange!30} \ours ($224\times224$) & 0.81 & 233.1 \\
    \hline
    \rowcolor{gray!20} DeiT ($512\times512$)   & 13.88 & 1739.3 \\
    \rowcolor{orange!30} \ours ($512\times512$) & 13.90 & 1789.7 \\
    \noalign{\hrule height 1.5pt}
    \end{tabular}%
\label{tab:gpu_time}
\end{table}
\begin{table*}[t]
    \centering
    \caption{\textit{Full fine-tuning results on VTAB-1K}: 
    Comparison of the top-1 accuracies of \colorbox{gray!20}{baseline models} and their \colorbox{orange!30}{$\text{Baseline}\odot \mathbf{M}_\text{\ours}$} counterparts across the VTAB-1K benchmark. The three task groups are abbreviated as NAT. = Natural, SPE. = Specialized, and STR. = Structured. The values in parentheses $(\cdot)$ indicate the accuracy difference compared to the baseline. The best performance within each model pair is highlighted in \textbf{bold}. \textcolor{g}{Green} highlights the improvement.}    \label{tab:finetune2}
    \resizebox{\textwidth}{!}{
    \begin{tabular}{lc>{\columncolor{gray!20}}c>{\columncolor{gray!20}}c>{\columncolor{gray!20}}c>{\columncolor{gray!20}}c>{\columncolor{orange!30}}c>{\columncolor{orange!30}}c>{\columncolor{orange!30}}c>{\columncolor{orange!30}}c}
         \noalign{\hrule height 1.5pt}
        \multirow{3}{*}{Model} & \multirow{3}{*}{Param.} & \multicolumn{8}{c}{Top-1 Accuracy ($\%$)}  \\ 
        & & \multicolumn{4}{{>{\columncolor{gray!20}}c}}{Baseline} & \multicolumn{4}{{>{\columncolor{orange!30}}c}}{$\text{Baseline}\odot \mathbf{M}_\text{\ours}$} \\ 
        & & NAT. & SPE. & STR. & Avg.& NAT. & SPE. & STR. & Avg.\\
        \hline
        DeiT-T & 5M  & 69.56 & 82.34 & 53.57 & 65.52 & \textbf{71.90} \textbf{\textcolor{g}{(+2.34)}} & \textbf{83.75} \textbf{\textcolor{g}{(+1.41)}} & \textbf{57.50} \textbf{\textcolor{g}{(+3.93)}} & \textbf{68.33} \textbf{\textcolor{g}{(+2.81)}} \\
        DeiT-S & 22M & 73.64 & 84.30 & 53.44 & 67.38 & \textbf{76.06} \textbf{\textcolor{g}{(+2.42)}} & \textbf{85.05} \textbf{\textcolor{g}{(+0.75)}} & \textbf{58.26} \textbf{\textcolor{g}{(+4.82)}} & \textbf{70.46} \textbf{\textcolor{g}{(+3.08)}} \\
        DeiT-B & 86M & 76.93 & 85.52 & 57.00 & 70.35 & \textbf{77.96} \textbf{\textcolor{g}{(+1.03)}} & \textbf{86.29} \textbf{\textcolor{g}{(+0.77)}} & \textbf{61.89} \textbf{\textcolor{g}{(+4.89)}} & \textbf{72.95} \textbf{\textcolor{g}{(+2.60)}} \\
        \hline
        DeiT-III-S & 22M  & 75.13 & 83.63 & 52.92 & 67.57 & \textbf{77.03} \textbf{\textcolor{g}{(+1.90)}} & \textbf{85.46} \textbf{\textcolor{g}{(+1.83)}} & \textbf{61.61} \textbf{\textcolor{g}{(+8.69)}} & \textbf{72.31} \textbf{\textcolor{g}{(+4.74)}} \\
        DeiT-III-B & 86M  & 78.19 & 85.26 & 56.71 & 70.63 & \textbf{79.24} \textbf{\textcolor{g}{(+1.05)}} & \textbf{86.47} \textbf{\textcolor{g}{(+1.21)}} & \textbf{63.03} \textbf{\textcolor{g}{(+6.32)}} & \textbf{73.94} \textbf{\textcolor{g}{(+3.31)}} \\
        DeiT-III-L & 304M & 88.68 & 84.38 & 51.40 & 67.41 & \textbf{90.39} \textbf{\textcolor{g}{(+1.71)}} & \textbf{84.68} \textbf{\textcolor{g}{(+0.30)}} & \textbf{54.95} \textbf{\textcolor{g}{(+3.55)}} & \textbf{69.51} \textbf{\textcolor{g}{(+2.10)}} \\
        DeiT-III-H & 632M &  88.15	& 84.18	& 50.70 & 66.91   & \textbf{89.10}\textbf{\textcolor{g}{(+0.95)}} & \textbf{84.43} \textbf{\textcolor{g}{(+0.25)}}& \textbf{53.65} \textbf{\textcolor{g}{(+2.95)}}& \textbf{68.50} \textbf{\textcolor{g}{(+1.41)}}\\
        \hline
        DINO-S & 22M & 75.35 & 85.09 & 60.65 & 71.21 & \textbf{76.26} \textbf{\textcolor{g}{(+0.91)}} & \textbf{85.32} \textbf{\textcolor{g}{(+0.23)}} & \textbf{61.24} \textbf{\textcolor{g}{(+0.59)}} & \textbf{71.84} \textbf{\textcolor{g}{(+0.63)}} \\
        DINO-B & 86M & 77.50 & 85.77 & 58.47 & 71.23 & \textbf{78.65} \textbf{\textcolor{g}{(+1.15)}} & \textbf{86.44} \textbf{\textcolor{g}{(+0.67)}} & \textbf{60.84} \textbf{\textcolor{g}{(+2.37)}} & \textbf{72.79} \textbf{\textcolor{g}{(+1.56)}} \\
        \noalign{\hrule height 1.5pt}
    \end{tabular}}
\end{table*}

\textbf{Full fine-tuning}
In the first setting, we test the plug-in capability of \ours by fully fine-tuning pretrained DeiT, DeiT-III, and DINO models across scales ranging from 5M to 630M parameters. Prior to fine-tuning, we freshly initialize the $\mathbf{M}_\text{\ours}$ mask and scaling factor $\alpha$ as defined in \cref{violinattention}. Next, we jointly fine-tune the model and the mask with accuracies reported in \cref{tab:finetune2}. 
Freshly initialized mask enables fast adaptation by learning task-specific structural biases, which is critical in data-scarce setting. We also fine-tune the \ours pretrained models from \cref{sec:pretrain} and observe that masks learned only during downstream fine-tuning consistently outperform pretrained ones, full results and discussions are provided in \cref{ap:finetune}.

This highlights a key advantage, \textit{\ours can improve any pretrained global attention model when applied only at fine-tuning,} allowing models to specialize on the downstream task better and avoiding costly pretraining from scratch. The gains are substantial, up to \textbf{4.7\%} on average and \textbf{8.7\%} on individual groups, showing that the spatial bias introduced by \ours enables more effective learning in data-scarce regimes. Moreover, \ours introduces negligible overhead and generalizes across datasets, training setups, and model scales, including models larger than 600M parameters.
\begin{table}[t]
    \centering
    \caption{\textit{PEFT results on VTAB-1K with DeiT-B:} \# Param. denotes the number of learnable parameters per method. The baseline uses PEFT alone, while \ours combines PEFT with mask fine-tuning.}
    \begin{tabular}{lc>{\columncolor{gray!20}}c>{\columncolor{orange!30}}c}
        \noalign{\hrule height 1.5pt}
        \multirow{2}{*}{Method} & \multirow{2}{*}{$\#$ Param.} & \multicolumn{2}{c}{Avg. Accuracy ($\%$)}  \\ 
        & & Baseline & $\text{Baseline}\odot \mathbf{M}_\text{\ours}$ \\ 
        Full-FT& 86 M & 70.35 &\textbf{72.95 \textcolor{g}{(+2.60)}}\\
        \hline
        LoRA&$\sim$0.3M&71.04 &\textbf{72.55 \textcolor{g}{(+1.51)}}\\
        \hline
        DoRA&$\sim$0.6M&70.75&\textbf{71.90 \textcolor{g}{(+1.15)}}\\        
        \noalign{\hrule height 1.5pt}
    \end{tabular}%
    \label{tab:peft}
\end{table}

\textbf{PEFT with \ours} 
Secondly, we use the PEFT methods LoRA~\citep{hu2022lora} and DoRa~\citep{liu2024dora} to fine-tune DeiT-B, with results in \cref{tab:peft}. The \ours mask is freshly initialized and updated alongside the PEFT weights. The extra cost introduced by \ours remains insignificant, only 0.0015\% additional parameters compared to 0.35\% introduced by LoRA. These results show that \ours can integrate seamlessly with different PEFT methods, further highlighting its applicability and generalizability.
\subsection{Pretraining} \label{sec:pretrain}
\textbf{ImageNet-1K pretraining}
We pretrain \ours on small-scale models~\footnote{We observed that for ImageNet pretraining with larger models, the performance gains are smaller, which is expected. See \cref{ap:base} for numerical results and a detailed explanation.} under both supervised and self-supervised paradigms, as shown in \cref{tab:pretrain}. For supervised training, we follow the DeiT training recipe for tiny and small models, a strong baseline for data-efficient supervised learning. In all DeiT based pretraining experiments, we adopt only the training recipe and do not use distillation. \ours consistently improves performance without additional tuning, with gains of \textbf{0.8\%} and \textbf{0.9\%} on DeiT-T and DeiT-S models, demonstrating strong compatibility. For these models, we replace the class token with Global Average Pooling (GAP)~\citep{lin2013network,lu2022bridging}, which is more compatible with \ours, see \cref{abl:gap} for details. 

For self-supervised training, we adopt DINO, a state-of-the-art teacher-student framework for label free representation learning, known for its stable training dynamics and strong downstream performance. In our experiments, both teacher and student networks are equipped with \ours attention. Across model scales and training durations, \ours consistently improves performance, yielding gains in both KNN and linear evaluations on ImageNet. For all models, \textit{we strictly follow the original training recipes without modifying any hyperparameters for \ours}. Baseline accuracies are taken directly from the reported values.  
\begin{figure*}[t]
    \centering
    \captionof{table}{\textit{Pretraining results on ImageNet-1K}: 
    Comparison of the top-1 accuracies of \colorbox{gray!20}{baseline models} with their \colorbox{orange!30}{\ours} counterparts. The values in parentheses $(\cdot)$ indicate the accuracy difference compared to the baseline. The best performance between each pair of models is highlighted in \textbf{bold}. For DINO models, both KNN and linear evaluations are reported and (100), (300) indicate the number of training epochs of the models. \textbf{(Left)} Supervised training with similar sized CNN baselines, \textbf{(Right)} Self-supervised training.}
    \label{tab:pretrain}
\begin{minipage}[t]{0.44\textwidth}
    \begin{tabular}{lc>{\columncolor{gray!20}}c>{\columncolor{orange!30}}c}
         \noalign{\hrule height 1.5pt}
        \multirow{2}{*}{Model} & \multirow{2}{*}{$\#$ Param.} & \multicolumn{2}{c}{Top-1 Accuracy ($\%$)}  \\ 
        & & Baseline & \textbf{\ours}  \\ 
        \hline
        DeiT-T  &  5M & 72.2 & \textbf{73.0 \textcolor{g}{(+0.8)}}  \\
        DeiT-S  & 22M & 79.8 &\textbf{ 80.7  \textcolor{g}{(+0.9)}}  \\
        \hline
        ResNet-18  & 12M & 69.8  \\
        ResNet-50  & 25M & 76.2   \\
         \noalign{\hrule height 1.5pt}
    \end{tabular}
\end{minipage}
\hfill
\begin{minipage}[t]{0.54\textwidth}
        \begin{tabular}{lcc>{\columncolor{gray!20}}c>{\columncolor{orange!30}}c}
         \noalign{\hrule height 1.5pt}
        \multirow{2}{*}{Model} &  \multirow{2}{*}{} & \multirow{2}{*}{$\#$ Param.} & \multicolumn{2}{c}{Top-1 Accuracy ($\%$)}  \\ 
        & & & Baseline & \textbf{\ours}  \\ 
        \hline
        \multirow{2}{*}{DINO-S (100)} & KNN    & \multirow{2}{*}{22M} & 69.3 & \textbf{70.0  \textcolor{g}{(+0.7)}}  \\
                                            & Linear & & 74.0 & \textbf{74.6  \textcolor{g}{(+0.6)}} \\
        \multirow{2}{*}{DINO-S (300)} & KNN    & \multirow{2}{*}{22M} & 72.8 & \textbf{73.4 \textcolor{g}{(+0.6)}}  \\
                                            & Linear & & 76.1 & \textbf{76.4  \textcolor{g}{(+0.3)}} \\
         \noalign{\hrule height 1.5pt}
    \end{tabular}
\end{minipage}
\end{figure*}

\textbf{Ablation studies}
In \cref{ap:ablations}, we provide comprehensive ablations on \ours, using the same pretraining setup. \cref{abl:gap,abl:pos_embd} examine the effects of global average pooling and positional embeddings, while \cref{ap:curve_configs} explores curve configurations, including using a single curve, all combinations in $\mathcal{C}$, Z-curve only, Manhattan distance-based masking (similar to RMT \citep{fan2024rmt}), random curve orderings and variants without transposed curves. \cref{ap:masking} compares alternative masking strategies, and \cref{ap:other_ablations} analyzes key design choices such as initialization, the scaling factor $\alpha$, and fixed versus learnable decay parameters. Together, these ablations clarify the contribution of each component. Additionally, in \cref{abl:multi_res}, we evaluate the context extrapolation capability of \ours using multi-resolution classification and video generation with a pretrained \ours DINO model, leveraging the extrapolation property of the KMS decay mask $\mathbf{M}_\ours$. 

\newpage
\textbf{Pixel-level CIFAR-100 pretraining} 
Recent work has explored pixel-level tokenization for ViTs~\citep{nguyen2025an, wang2025scaling}, which provides detailed image representations and avoids hand-crafted choices around patching. However, since eliminating patching also removes a key source of locality bias and it makes models even more data-hungry and harder to optimize on smaller datasets such as CIFAR-100~\citep{cifar}. This setting aligns perfectly with \ours, which introduces locality into the model independently of the patching process.  

On CIFAR-100, when ViT-T is trained using the DeiT ImageNet training recipe, \ours achieves a striking improvement of over \textbf{7\%} compared to the vanilla pixel-level baseline, as shown in \cref{tab:cifar}. This demonstrates that our locality mechanism provides a powerful inductive bias, enabling effective learning in small-data, small-model regimes where standard ViTs fail. These results highlight both the effectiveness of \ours and the importance of locality awareness for pixel-level ViTs, particularly in resource-constrained scenarios where large-scale pretraining or very long training schedules are impractical.
\begin{table}[t]
    \centering
    \caption{ \textit{Pixel level CIFAR-100 pretraining}: Comparison of the top-1 accuracies of \colorbox{gray!20}{baseline} and \colorbox{orange!30}{\ours} models. }    
    \begin{tabular}{lc>{\columncolor{gray!20}}c>{\columncolor{orange!30}}c}
        \noalign{\hrule height 1.5pt}
        \multirow{2}{*}{Model} & \multirow{2}{*}{$\#$ Param.} & \multicolumn{2}{c}{Avg. Accuracy ($\%$)}  \\ 
        & & Baseline & \textbf{\ours}  \\ 
        \hline
        DeiT-T & 5 M & 60.8 & \textbf{68.0} \textcolor{g}{(\textbf{+7.2)}}\\    
        \noalign{\hrule height 1.5pt}
    \end{tabular}%
    \label{tab:cifar}
\end{table}
\subsection{Understanding Spatial Awareness in \ours} \label{sec:extra}
\textbf{Performance gain on the Structured group}
The Structured category of VTAB includes tasks that require understanding the spatial  structure of the images such as object counting and 3D depth prediction, many of which are derived from simulated environments. These scenes often consist of rendered geometric objects that are simple to humans but differ significantly from images in datasets like ImageNet. As a result, success in these tasks often depends on \textit{recognizing positional, orientational, or shape-based information, making local spatial layout especially important.}

\begin{figure}[h]
  \begin{center}
    \includegraphics[width=0.65\linewidth]{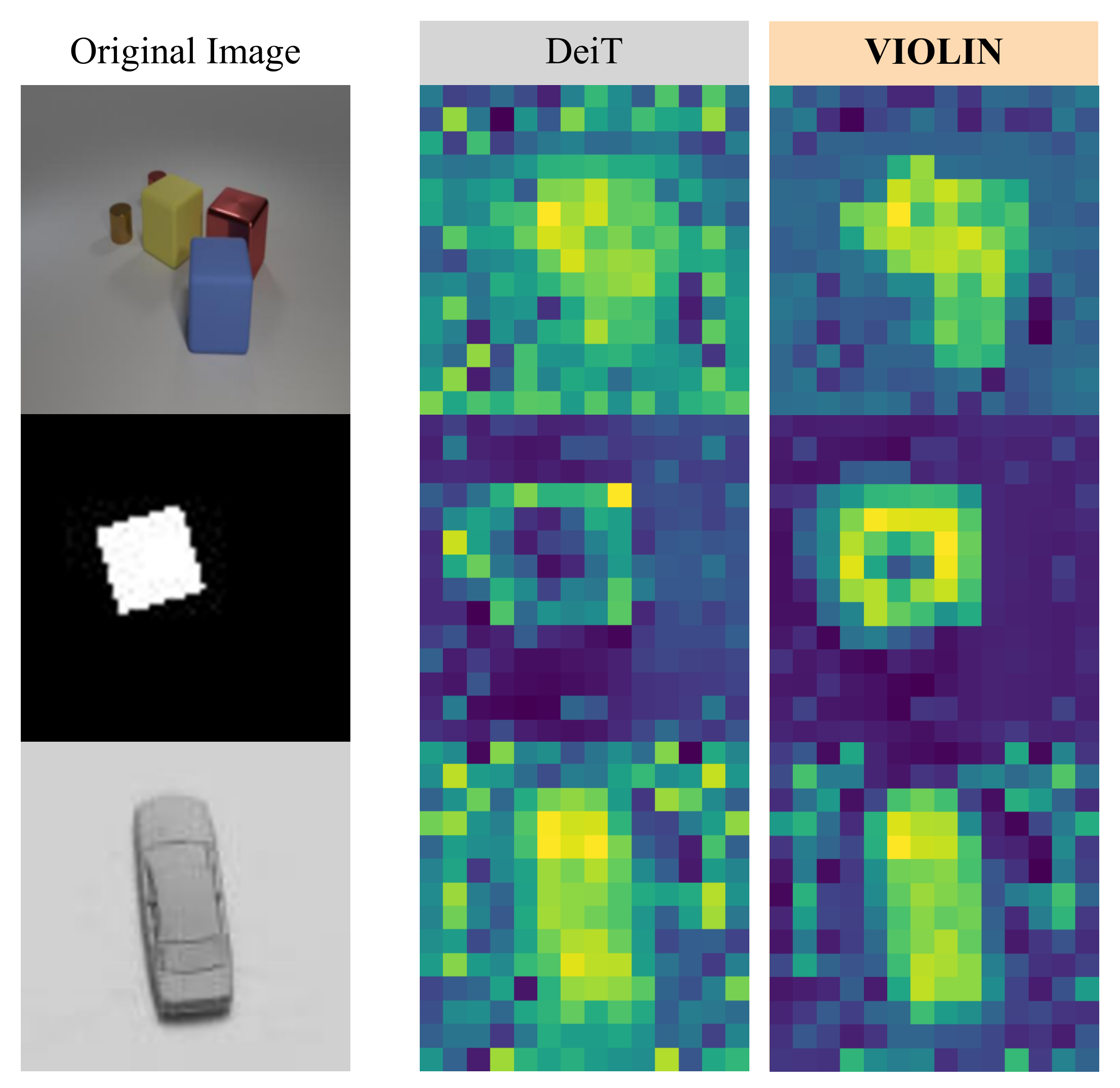}
  \end{center}
  \caption{\textit{Attention heatmaps on Structured tasks:} Examples are drawn from three datasets in the Structured group: CLEVR-Count, dSprites-Location, and SmallNORB-Azimuth. They are taken from layer 12, using the same attention head for each image. }
  \label{fig:structured}
\end{figure}

As shown in \cref{tab:finetune2}, the \ours mask provides the largest improvements in this category, with gains of up to \textbf{8.69\%}, a \textbf{16\%} relative increase over the baseline. These results highlight the \ours's ability to enhance spatial capabilities, and generalize effectively to tasks that depend heavily on spatial structure. In \cref{fig:structured}, we illustrate images from three datasets in the Structured group with attention heatmaps of DeiT-B models fine-tuned with and without $\mathbf{M}_\text{\ours}$. The comparisons show that models fine-tuned with \ours attend to objects more accurately, suppress noise on irrelevant patches, and produce more uniform responses in background regions, further demonstrating its benefit for spatial understanding. \cref{ap:visuals} provides additional visualizations of per-head attention heatmaps across different layers.

\begin{figure*}[t]
    \centering
    \captionof{table}{\textit{Results on dense prediction tasks:} \textbf{(Left)} mIoU scores on semantic segmentation on ADE20K with DeiT-B model. \textbf{(Right)} box AP and mask AP scores on object detection and instance segmentation on COCO with Swin-T. }
\begin{minipage}[t]{0.4\textwidth}
    \begin{tabular}{l>{\columncolor{gray!20}}c>{\columncolor{orange!30}}c}
        \noalign{\hrule height 1.5pt}
        \multirow{2}{*}{Backbone} & \multicolumn{2}{c}{mIoU}  \\ 
         & Baseline & {$\text{Baseline}\odot \mathbf{M}_\text{\ours}$}  \\
        \hline
        DeiT-B & 45.24 &  \textbf{45.80 \textcolor{g}{(+0.56)}}\\    
        \noalign{\hrule height 1.5pt}
    \end{tabular}
    \label{tab:segment}
\end{minipage}
\hfill
\begin{minipage}[t]{0.53\textwidth}
    \begin{tabular}{l>{\columncolor{gray!20}}c>{\columncolor{gray!20}}c>{\columncolor{orange!30}}c>{\columncolor{orange!30}}c}
        \noalign{\hrule height 1.5pt}
        \multirow{2}{*}{Backbone}  & \multicolumn{2}{c}{Baseline} & \multicolumn{2}{c}{{$\text{Baseline}\odot \mathbf{M}_\text{\ours}$}}  \\ 
        & box AP & mask AP & box AP & mask AP \\
        \hline
        Swin-T & 42.7 & 39.3 &  \textbf{42.8 \textcolor{g}{(+0.1)}} & \textbf{39.7 \textcolor{g}{(+0.4)}}\\    
        \noalign{\hrule height 1.5pt}
    \end{tabular}
\end{minipage}
\end{figure*}

\textbf{Curve configurations}
We examine the individual contribution of each curve by pretraining DeiT-S with all $2^4=16$ combinations of four curves (including their transposed variants), with accuracies reported in \cref{tab:curves}. While some combinations yield larger gains, every curve contributes meaningfully, motivating the use of all four in \ours to leverage their complementary spatial information. To illustrate this, \cref{fig:curves_sumbul} visualizes the decay masks for three reference patches (top-left, center, bottom-right) across all curves and their transposes. Lighter regions indicate stronger attention, and the distinct patterns show how different curves bias the model toward diverse spatial regions.  

\begin{figure}[t] 
    \centering
    \includegraphics[width=0.95\linewidth]{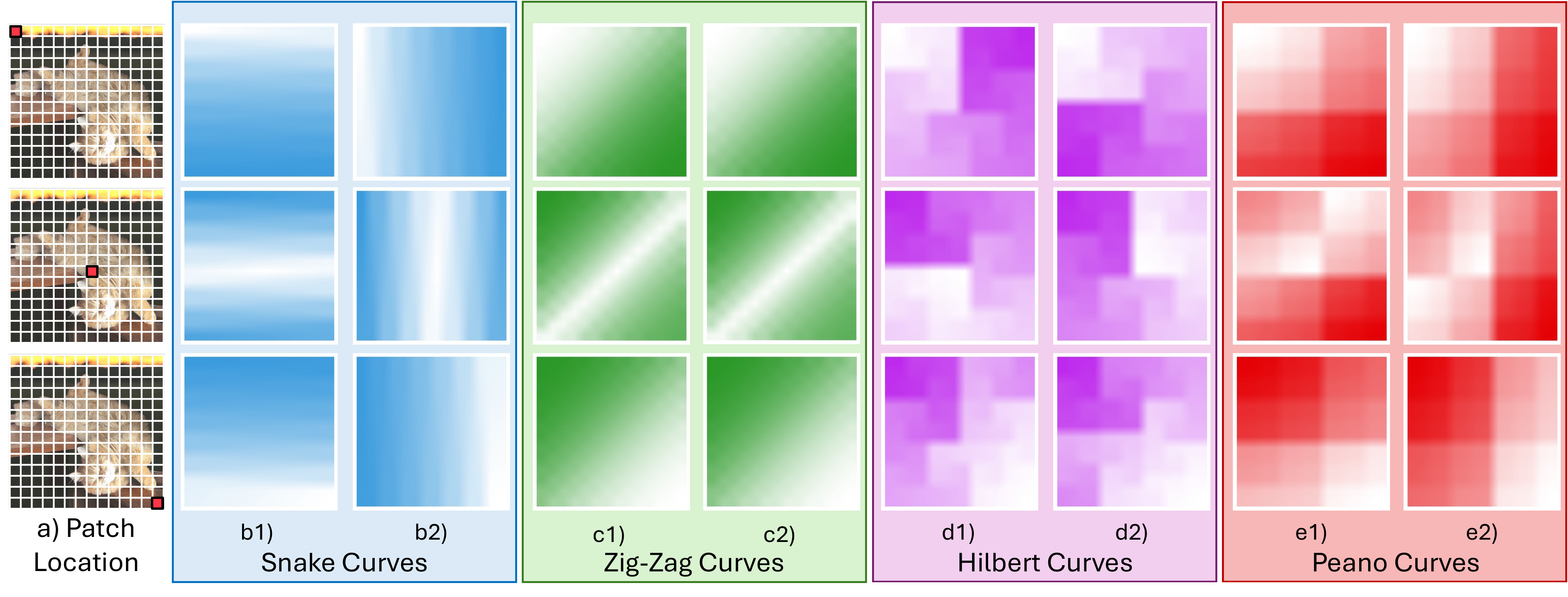}
   \caption{\textit{Mask patterns for different patches:} Visualization of decay mask patterns for three reference patches, top-left, center, and bottom-right, (1st, 2nd and 3rd rows) across all curves. 
   Lighter values indicate stronger spatial relevance, showing more attended regions. 
   \textbf{a1)} Reference patch locations, 
        \textbf{b1)} Snake, 
    \textbf{b2)} Snake$^{\top}$, 
        \textbf{c1)} Zig-zag, 
    \textbf{c2)} Zig-zag$^{\top}$, 
        \textbf{d1)} Hilbert, 
    \textbf{d2)} Hilbert$^{\top}$, 
        \textbf{e1)} Peano, 
    \textbf{e2)} Peano$^{\top}$ curves.}
    \label{fig:curves_sumbul}
\end{figure}

We further analyze the learned decay parameters $\gamma_c$ for DeiT-B in \cref{fig:gamma}, observing that most remain close to one. This is consistent with findings in Linear Transformers where $\gamma \approx 1$ is associated with preserving long-range reasoning~\citep{lru}. With resolution 224 (sequence length 196), even a moderate decay changes the effective receptive field significantly (e.g., $0.9^{196} < 10^{-9}$), making values near one necessary to keep both local and global spatial information. Smaller $\gamma_c$ values, on the other hand, act as implicit curve selection, as their corresponding masks contribute minimally to the weighted average, with certain layers and heads emphasizing particular curves. \cref{fig:gamma_values} visualizes how different $\gamma$ values change the effective receptive field, with additional attention heatmaps and curve-flattening visualizations provided in \cref{ap:visuals}.

\begin{figure}[b] 
    \centering
    \includegraphics[width=0.9\linewidth]{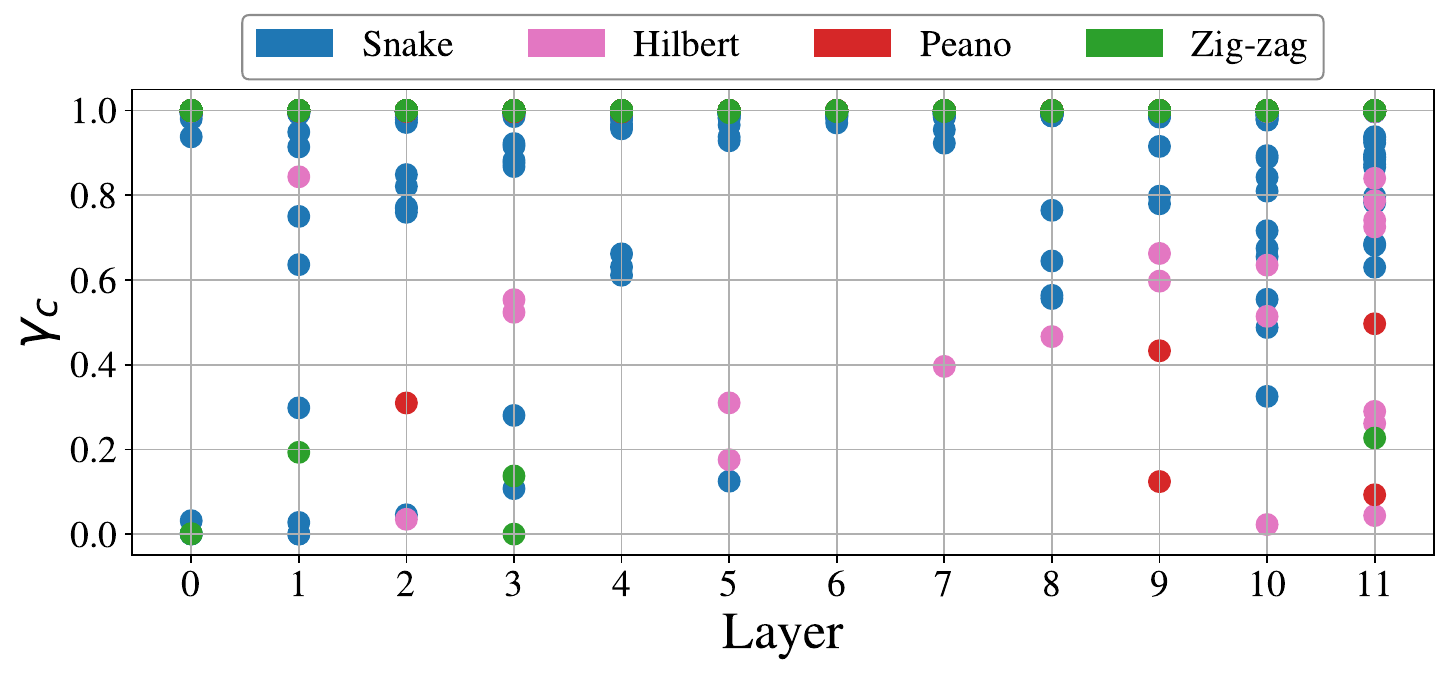}
    \caption{\textit{$\gamma_c$ values:} $\gamma_c$ values of \ours DeiT-B model are presented across layers, heads and curves. Most remain close to one, indicating active use of long-range spatial information.}
    \label{fig:gamma}
\end{figure}

\textbf{Dense prediction tasks}
To assess the capabilities of \ours beyond classification, we evaluate it on semantic segmentation and object detection. For both tasks, baseline and \ours enhanced models are trained under identical setups to ensure fair comparison, with results reported in \cref{tab:segment}. These experiments also highlight the flexibility of $\mathbf{M}_\text{\ours}$, which naturally generalizes to arbitrary input shapes, enabling resolution expansion and non-square images.

For semantic segmentation, we use ADE20K~\citep{zhou2017scene,zhou2019semantic}, a challenging scene parsing dataset, implemented in the \texttt{mmsegmentation} framework~\citep{mmseg2020}. The backbone is an ImageNet pretrained DeiT-B model combined with UPerNet~\citep{xiao2018unified}. The $\mathbf{M}_\text{\ours}$ mask is freshly initialized at fine-tuning, and trained for 80k iterations with batch size 16. 

For object detection, we evaluate on COCO~\citep{lin2015microsoftcococommonobjects} using the \texttt{mmdetection} framework~\citep{mmdetection} with an ImageNet pretrained Swin-T~\citep{liu2021swin} backbone and Mask R-CNN~\citep{maskrcnn} as the detector. $\mathbf{M}_\text{\ours}$ is freshly initialized at fine-tuning, and models are trained with a $1\times$ schedule and batch size 16. \ours yields improvements of \textbf{+0.56} mIoU on semantic segmentation and a \textbf{+0.4} mAP over the baseline (see \cref{tab:segment}), showing how spatial priors can improve dense prediction tasks.

\begin{table}[b] 
\centering
\vspace{-4mm}
\caption{\textit{Comparison of locality methods:} The pretrained DeiT-B model fine-tuned with different locality methods on the VTAB Structured group. Best result is highlighted on \textbf{bold}.}
\begin{tabular}{lcc}
\noalign{\hrule height 1.5pt}
\multirow{2}{*}{Method} & \# Extra  & Structured  \\
& Param. & Avg. (\%) \\
\midrule
\rowcolor{gray!20} Baseline (DeiT-B)        & --      & 57.00 \\
\rowcolor{orange!30}\ours                    & $\sim$1.3K   & \textbf{61.89} \\
Single SFC ($\mathbf{M}_{\text{Peano}}$)       & $\sim$0.4K   & 61.63 \\
Additive $\mathbf{M}_{\ours}$          & $\sim$1.3K   & 61.34 \\
Swin RPB                 & $\sim$105K   & 61.58 \\
i-RPE QKV                & $\sim$115K   & 61.45 \\
LocalViT                 & $\sim$6.2M   & 61.50 \\
Manhattan Mask           & $\sim$0.4K   & 58.37 \\
\bottomrule
\end{tabular}%
\label{tab:structured_locality}
\end{table}

\textbf{Comparison against other inductive bias methods}
In \cref{tab:structured_locality}, we present an extended comparison of locality-enforcing methods on the Structured group with fine-tuning. All methods use the same pretrained DeiT-B backbone with their locality mechanisms initialized on top, ensuring that all models start identically. All methods are then fine-tuned under the same protocol, described in \cref{ap:vtab_hyper}.

These results show that while most locality priors offer some improvement, \textit{\ours achieves the strongest gains with minimal overhead}. This indicates that the gains come specifically from the usage of multiple SFC curves, rather than from the presence of any local bias. Moreover, they highlight \ours's effectiveness as a plug-and-play spatial prior in small-data finetuning regimes. Full implementation details, initialization choices, further details on chosen baselines and per-dataset results are provided in \cref{ap:locality}. 

\section{Conclusion and Future Directions}
In this work, we introduced \ours, a novel masked attention mechanism inspired by the decay masks of Linear Transformers and the perspective of flattening via space filling curves. By integrating diverse spatial patterns into a unified mask, \ours enhances the understanding of relative spatial relationships without altering the training recipe, or introducing a significant computational cost.  

Experiments show that \ours is especially effective in small models and data-scarce settings, where spatial inductive bias is most critical. It serves as a plug-and-play module applicable only during fine-tuning, combining seamlessly with PEFT methods. More broadly, \ours emphasizes the overlooked role of patch ordering and spatial priors in ViT design, offering a lightweight and practical approach to strengthen locality in global attention ViTs.  

\textbf{Future directions} Since \ours operates directly on attention scores, it can be used in any setting that uses global attention and benefits from spatial priors. This opens up many exciting future directions. For instance, in dense tasks such as depth estimation, super-resolution and object tracking, explicit spatial priors are critical. In data-scarce applications like medical imaging or satellite analysis where learning spatial structure from scratch is costly, \ours's plug-and-play nature can allow us to inject strong locality priors at fine-tuning. Similarly, video understanding or multimodal learning presents promising opportunities to further explore the impact of \ours in various vision backbones.

\section*{Acknowledgements}
We thank the reviewers for their valuable feedback. This work was partially sponsored by the Army Research Office and was accomplished under Grant Number W911NF-24-1-0048, and partially funded by the Swiss National Science Foundation (SNSF) under grant number 200021-205011. Additionally, this work was supported under project ID \#37 as part of the Swiss AI Initiative, through a grant from the ETH Domain and computational resources provided by the Swiss National Supercomputing Centre (CSCS) under the Alps infrastructure and by the Swiss AI Initiative 2025 Fellowship Program.

\section*{Impact Statement}

This work aims to advance machine learning by improving the spatial inductive biases of vision transformers, particularly in small-model and data-scarce settings. By enabling more data-efficient and lightweight vision models, our approach may reduce dependence on large-scale pretraining and computational resources and improve resource constrained domains such as medical imaging, remote sensing, and scientific analysis. We do not foresee any negative ethical implications beyond those commonly associated with vision models.

\bibliography{references}
\bibliographystyle{icml2026}

\newpage
\onecolumn
\appendix
\makeatletter
\renewcommand\partname{}  %
\renewcommand\thepart{}   %
\makeatother
\newpage
\part{Appendix} %

\begin{itemize}
  \item \hyperref[ap:notation]{\textbf{A} Notations}
  \item \hyperref[bckg]{\textbf{B} Extended Background}
  \begin{itemize}
    \item \hyperref[bckg:vit]{B.1 ViTs and spatial priors}
    \item \hyperref[bckg:linear]{B.2 Linear Transformers}
    \item \hyperref[ap:sfc]{B.3 Space Filling Curves}
    \item \hyperref[ap:relpos]{B.4 Locality via decay mask}
    \item \hyperref[ap:toeplitz]{B.5 Efficiency of Toeplitz decay mask}
    \item \hyperref[ap:connections]{B.6 Connections of \ours to other models}
  \end{itemize}
  \item \hyperref[ap:proofs]{\textbf{C} Proofs}
  \begin{itemize}
    \item \hyperref[permutation_equivariance]{C.1 Attention is permutation equivariant}
    \item \hyperref[ap:sfc_distance]{C.2 SFCs in decay mask are a distance metric}
    \item \hyperref[Mc_equality]{C.3 \ours SFC flattening only reflects in decay mask}
    \item \hyperref[ap:sfc_avg]{C.4 Averaging multiple SFC decay masks}
  \end{itemize}
  \item \hyperref[sec:design]{\textbf{D} Further design details}
  \begin{itemize}
    \item \hyperref[ap:init]{D.1 Initialization}
    \item \hyperref[ap:other_arch]{D.2 Adaptation of \ours to various architectures}
  \end{itemize}
  \item \hyperref[ap:ablations]{\textbf{E} Ablation studies}
  \begin{itemize}
    \item \hyperref[abl:pos_embd]{E.1 Positional embeddings}
    \item \hyperref[ap:curve_configs]{E.2 Alternative curve configurations}
    \item \hyperref[ap:masking]{E.3 Alternative masking strategies}
    \item \hyperref[ap:other_ablations]{E.4 Other design elements}
    \item \hyperref[abl:gap]{E.5 Global Average Pooling (GAP)}
  \end{itemize}
  \item \hyperref[ap:additional]{\textbf{F} Additional results}
  \begin{itemize}
    \item \hyperref[ap:base]{F.1 Pretraining of larger models}
    \item \hyperref[ap:finetune]{F.2 Fine-tuning of \ours pretrained models}
    \item \hyperref[abl:multi_res]{F.3 Multi-resolution classification}
    \item \hyperref[ap:visuals]{F.4 Additional visualizations}
    \item \hyperref[ap:vtab]{F.5 Details and individual results on VTAB-1K dataset}
    \item \hyperref[ap:locality]{F.6 Comparison against other locality methods}
    \item \hyperref[ap:learned]{F.7 Learned curve order}
    \item \hyperref[ap:rpe_comp]{F.8 Comparison with relative positional encodings}
  \end{itemize}
  \item \hyperref[ap:codes]{\textbf{G} Codes and implementation details}
  \begin{itemize}
    \item \hyperref[ap:compute]{G.1 Compute resources}
    \item \hyperref[ap:vtab_hyper]{G.2 VTAB-1K hyperparameters}
    \item \hyperref[curve_codes]{G.3 Codes for curves}
    \item \hyperref[ap:decay_code]{G.4 Code of efficient decay mask}
  \end{itemize}
\end{itemize}

\clearpage

\newpage

\section{Notations} \label[appendix]{ap:notation}
In \cref{tab:notation}, we summarize the notations used in the paper.
\begin{table}[ht]
    \centering
    \caption{\textit{Notations}: Summary of notations used throughout the paper.}
    \label{tab:notation}
    \begin{tabular}{ll}
        \noalign{\hrule height 1.5pt}
         \textbf{Definition}& \textbf{Notation} \\
         \hline 
         Image & $\mathcal{I} \in \R^{H\times W \times d} $ \\
         Curves set & $\mathcal{C}$\\
         Curve ID& $c \in \mathcal{C}$\\
         Flattening operator with curve $c$& $F_c(\mathcal{I}): \R^{H\times W} \rightarrow \R^{N}$\\
         Flattened image with curve $c$& $\mathbf{X}_c \in \R^{N \times d }$ \\
         Permutation from curve $c_1$ to $c_2$& $\pi_{c_1\rightarrow c_2}(i)$  \\
         Permutation matrix from curve $c_1$ to $c_2$& $\mathbf{P}_{c_1\rightarrow c_2} \in \R^{N \times N}$ \\
         Decay mask for basis curve (Z-curve)& $\mathbf{M} \in \R^{N \times N}$ \\
         Decay mask for curve $c$ & $\mathbf{M}_c \in \R^{N \times N}$\\
         Permuted decay mask for curve $c$ & $\widetilde{\mathbf{M}_c} \in \R^{N \times N}$\\
         Average of all decay masks for all curves & $\mathbf{M}_\text{\ours} \in \R^{N \times N}$\\
         Average mask scaling parameter & $\alpha \in \R$\\
         Decay parameter for mask $\mathbf{M}_c$& $\gamma_c \in \R$ \\
         Queries, keys, values & $\mathbf{Q},\mathbf{K},\mathbf{V} \in \R^{N \times d}$ \\
         Integer index set & $\mathbb{Z}_{[0,N)} = \{\, i \in \mathbb{Z} \mid 0 \le i < N \,\}$ \\

         \noalign{\hrule height 1.5pt}
    \end{tabular}
\end{table}

\section{Extended Background} \label{bckg}

\subsection{ViTs and spatial priors} \label[appendix]{bckg:vit}
ViTs are powerful alternatives to Convolutional Neural Networks (CNNs)~\citep{oshea2015introductionconvolutionalneuralnetworks}, but their design comes with a fundamental limitation: a lack of inherent spatial inductive bias. Unlike CNNs, where convolutions naturally encode locality and translation equivariance, ViTs treat images as sequences of independent patches. Spatial relations must therefore be inferred entirely from data, with positional embeddings and patching serving as the primary source of spatial information~\citep{vit,yuan2021t2t}. This design provides ViTs with flexibility in modeling global dependencies, however it also removes the strong inductive priors that are especially critical in data-scarce settings~\citep{d2021convit, wu2021rethinking}.

The absence of spatial inductive bias makes ViTs particularly fragile and data hungry when model capacity or training data is limited. Small ViTs trained on large datasets often underperform compared to CNNs, since they cannot rely on built-in locality to efficiently capture low-level spatial features~\citep{deit, yuan2021t2t}. In contrast, when both models and datasets are sufficiently large, and training is long enough, ViTs can learn these biases directly from data. For instance, large-scale training on ImageNet-21k~\citep{ridnik2021imagenet21k} or JFT~\citep{sun2017jft} demonstrates that ViTs can eventually match or surpass CNNs, but this comes at considerable computational and data cost~\citep{vit, deit}. Therefore, spatial inductive bias is highly beneficial in practice, especially for downstream tasks, resource-constrained scenarios and small scale models.

Motivated by this tradeoff, various approaches have emerged to reintroduce spatial priors into transformer architectures. Hierarchical models such as Swin Transformer~\citep{liu2021swin,liu2022swinv2} and Pyramid Vision Transformer (PVT)~\citep{pvt, wang2022pvtv2} adopt CNN-like multi-scale processing, enabling more efficient capture of local and global dependencies. Similarly, T2T-ViT~\citep{yuan2021t2t} progressively aggregates tokens to embed local structure. These designs restore the inductive biases of locality and scale, improving performance in regimes where pure ViTs struggle.

Another line of work incorporates convolutions directly into the transformer pipeline. Convolutional hybrids such as CvT~\citep{wu2021rethinking}, ConViT~\citep{d2021convit}, and CMT~\citep{guo2022cmt} explicitly embed local connectivity into the attention mechanism or token embedding process, bridging the gap between CNNs and ViTs. Other methods explore novel locality-aware mechanisms, including vicinity attention~\citep{zhang2023vicinity}, shuffle-based spatial mixing~\citep{huang2021shuffle}, and localized attention modules~\citep{li2021localvit, chu2023conditional}. Even more recent innovations, such as RMT~\citep{fan2024rmt}, propose decay masks inspired by RetNet~\citep{retnet} to enforce local inductive constraints.

Despite their effectiveness, most of these approaches achieve improved spatial priors by directly modifying the ViT architecture such as embedding convolutions into tokenization, or restructuring the model into hierarchical stages. While such changes enhance locality, they also increase design complexity, reduce modularity, and often require pretraining from scratch on large datasets to fully realize their benefits. This makes them less practical in settings where one wishes to reuse widely available pretrained vanilla ViTs. In contrast, methods that can inject spatial inductive bias without altering the base architecture, for instance, during fine-tuning, offer a more lightweight and flexible alternative, enabling broader applicability to downstream tasks and smaller models without sacrificing compatibility with existing pretrained checkpoints.

What remains missing is a simple mechanism to bridge this gap: an approach that can utilize already trained ViTs while still strengthening their spatial priors, which can be achieved via \ours with close to zero additional cost. 

\subsection{Linear Transformers} \label[appendix]{bckg:linear}
Linear attention is mathematically equivalent to an RNN~\citep{trans_rnn}
\begin{align} \label{linatt2}
    \quad \mathbf{S}_i = \mathbf{S}_{i-1} + \mathbf{k}_i^\top\mathbf{v}_i, \quad \mathbf{y}_i = \mathbf{q}_i^\top\mathbf{S}_{i}  \quad \Leftrightarrow \quad \mathbf{Y} = (\mathbf{Q}\mathbf{K}^\top \odot \mathbf{L_\text{Causal}})\mathbf{V},
\end{align}
where $\mathbf{S}_i \in \mathbb{R}^{d \times d}$ represents the hidden state of the Linear Transformer in its equivalent RNN form and $\mathbf{L_\text{Causal}} \in \R^{N \times N}$ is lower triangular matrix of ones. 

Building on that, Linear Transformers with a scalar decay factor commonly take the following recurrent form:
\begin{equation}
    \mathbf{S}_i = \mathbf{\Lambda}_i\mathbf{S}_{i-1} + \mathbf{k}_i^\top\mathbf{v}_i, \quad \mathbf{u}_i = \mathbf{q}_i^\top\mathbf{S}_{i}
\end{equation}
with hidden state $\mathbf{S}_i$ and output $\mathbf{y}_i$. Here, the behavior of the model is determined by the choice of the decay parameter \( \mathbf{\Lambda}_i \). It is also standard practice to apply a non-linearity to the queries and keys, such that \( \mathbf{Q}, \mathbf{K} = \phi(\mathbf{W}_Q \mathbf{X}), \phi(\mathbf{W}_K \mathbf{X}) \), and to scale attention in relation to past tokens, as discussed in \citep{trans_rnn}.

\paragraph{No decay} In vanilla Linear Transformers (\cref{linatt}), there is no decay term, or equivalently \( \mathbf{\Lambda}_i = \mathbf{I} \) where $\mathbf{I}$ is the identity matrix. As a result, these models do not encode relative positional information. Performer~\citep{performer} is a representative example, using Random Fourier Features (RFF)~\citep{peng2021random} as the non-linear function \( \phi(\cdot) \), without any form of decay mechanism.

\paragraph{Non input-dependent decay} A key example in this category is RetNet~\citep{retnet}, which employs a fixed scalar decay parameter \( \mathbf{\Lambda}_i = \gamma \). This introduces a locality bias in the attention computation, but the decay remains constant and independent of the input sequence.

\paragraph{Input-dependent decay} Several recent linear transformers in the NLP domain fall into this category, where the decay parameter \( \mathbf{\Lambda}_i = g(\mathbf{x}_i) \) is a function of the input and thus varies across tokens. For example, DeltaNet~\citep{deltanet} defines the decay using the Delta Rule~\citep{schlag2021linear} as \( \mathbf{\Lambda}_i = \mathbf{I} - \mathbf{k}_i \mathbf{k}_i^\top \), while Gated RFA~\citep{peng2021random} uses an input-dependent scalar decay of the form \( \mathbf{\Lambda}_i = \sigma(\mathbf{W}\mathbf{x}_i) \), where \( \sigma(\cdot) \) is the sigmoid function and \( \mathbf{W} \in \mathbb{R}^d \), resulting in a scalar decay value per token.

\paragraph{Selective SMMs} This category of models is closely related to linear transformers with input-dependent decay. A prominent example is Mamba~\citep{mamba}, which can be interpreted as a linear transformer with an input-dependent diagonal matrix as the decay parameter \( \mathbf{\Lambda}_i \)~\citep{yang2023gated}. Mamba-2~\citep{mamba2}, a simplified variant, further refines this by using an exponential formulation for the decay factor: \( \mathbf{\Lambda}_i = \exp(-\exp(\mathbf{W}\mathbf{x}_i)) \), enabling a more stable and expressive modeling of token-wise recurrence.

\subsection{Space Filling Curves} \label[appendix]{ap:sfc}

SFCs have diverse applications across various domains, including image compression and generation~\citep{wang2022neural, dafner2000context}, point cloud processing~\citep{deepspace}, data mining~\citep{böhm2020spacefillingcurveshighperformancedata}, and data movement~\citep{walker2023impactspacefillingcurvesdata}. In this section, we define the curves used in this study as flattening operation $F_c$ for each curve. The definitions are adapted from~\citep{sagan1994space,peano1990courbe,hilbert1935stetige, zhao2024rethinkingzigzagflatteningimage}.
\paragraph{Z-curve}
The Z-curve, also known as sweep, row-major order, or raster scan, is the simplest and most widely used method for flattening a 2D image into a 1D sequence. It scans the image row by row, from top to bottom and left to right within each row. More concretely, for an image with width $W$, the flattening function can be defined as 
\begin{align} \label{z-curve}
    F_z(i,j) = i W +j.
\end{align}
This flattening order is the default scanning method in many vision models, including ViTs. As a result, we use it as our basis in the paper.

\paragraph{Snake Curve}
The snake curve, also known as boustrophedon order~\citep{snake_curve}, is a variation of the Z-curve that alternates the scanning direction across rows. Even-indexed rows are traversed left to right, while odd-indexed rows are traversed right to left, creating a continuous snake path through the image. The flattening function is given by:
 \begin{align} \label{snake-curve}
F_{\text{snake}}(i, j) =
\begin{cases}
i \cdot W + j & \text{if } i \bmod 2 = 0 \\
i \cdot W + (W - 1 - j) & \text{if } i \bmod 2 = 1
\end{cases}
\end{align}
This curve has a simplicity similar to the Z-curve while reducing long jumps between the end of one row and the beginning of the next. It is utilized in various applications, including image processing and path planning, due to its efficiency in covering areas without unnecessary repositioning.

\paragraph{Zig-zag Curve}
The Zig-zag curve~\citep{zigzag} is a diagonal scanning pattern that visits patches of an image along consecutive diagonals, alternating direction at each level. More concretely, with an image of size $H \times W$, for each diagonal $g \in \{0,\dots, H + W - 2\}$, it scans the elements where $ i+j=g$, from top-right to bottom-left on odd-numbered diagonals and from bottom-left to top-right on even-numbered ones. In other words, for each diagonal $g$, let the set of valid coordinates on that diagonal be
$D_g = \left\{ (i, j) \mid i + j = g, \; 0 \leq i < H, \; 0 \leq j < W \right\}$.
Then the ordering of $F_{\text{zigzag}}(i, j)$ can be defined by
\begin{equation} \label{zigzag}
F_{\text{zigzag}}(i, j) = \left( \sum_{k=0}^{g-1} |D_k| \right) + \operatorname{offset}_g(i, j),
\end{equation}
where $|D_k|$ is the length of the diagonal and $\operatorname{offset}_g(i, j)$ is
\[
\operatorname{offset}_g(i, j) =
\begin{cases}
\#\{(i', j') \in D_g \mid j' < j\} & \text{if } g \bmod 2 = 0, \\
\#\{(i', j') \in D_g \mid j' > j\} & \text{if } g \bmod 2 = 1.
\end{cases}
\]

The zig-zag curve is most commonly used in applications where frequency components are spatially grouped such as the JPEG compression standard to serialize the block of discrete cosine transform (DCT) coefficients, to ensure that low-frequency components that carry the most information appear early in the sequence.

\paragraph{Hilbert Curve}
The Hilbert curve~\citep{hilbert1935stetige} recursively divides the space into quadrants and connects them in a continuous path that fills the entire 2D grid. Similar to Peano curve, the Hilbert curve is most naturally defined on square images of size  $2^p\times2^p$ where the recursive quadrant-based construction aligns with the binary structure of the coordinates. The flattening function $F_{\text{hilbert}}(i, j)$ does not have a simple closed-form expression, but can be computed via recursive or bitwise algorithms, for example, Butz or Moore methods~\citep{BUTZ1969128, moore_curve}.

For an image of size $H \times W$ with $H = W = 2^p$, we can define the Hilbert curve flattening function as
\begin{equation} \label{hilbert-curve}
F_{\text{hilbert}}(i, j) = \sum_{k=1}^{n} q_k \cdot 4^{n - k}
\end{equation}
where $q_1 q_2 \cdots q_n$ is the base-4 Hilbert index corresponding to the normalized pixel center:
\begin{equation}
\left( \frac{i}{2^n} + \frac{1}{2^{n+1}}, \quad \frac{j}{2^n} + \frac{1}{2^{n+1}} \right)
\in [0, 1)^2
\end{equation}
Each digit $q_k \in \{0, 1, 2, 3\}$ represents the quadrant at level $k$ in the recursive Hilbert construction.

Points that are close in 2D space tend to remain close in 1D, which makes it especially valuable in image processing, spatial indexing, and contexts where locality is significant.

\paragraph{Peano Curve}
The Peano curve, also called Z-order curve or Morton curve,~\citep{peano1990courbe} is a recursive scanning approach that preserves spatial locality by interleaving the binary representations of the row and column indices. It is particularly well-suited to square grids of size $2^p\times2^p$ as the bit structure of the coordinates aligns naturally with the recursive subdivisions of the curve. 

For $H=W=2^p$, let $(i, j) \in \{0, \dots, 2^p - 1\}^2$ be the pixel coordinates, and we can write their binary expansions:
\begin{equation}
i = \sum_{k=0}^{n-1} i_k \cdot 2^k, \quad
j = \sum_{k=0}^{n-1} j_k \cdot 2^k
\quad \text{with } i_k, j_k \in \{0, 1\}
\end{equation}
\begin{equation} \label{morton}
F_{\text{peano}}(i, j) = \text{interleave\_bits}(i, j) = \sum_{k=0}^{p-1} \left( j_k \cdot 2^{2k + 1} + i_k \cdot 2^{2k} \right)
\end{equation}
As it can be constructed bitwise, it is computationally efficient and commonly used in applications like image tiling, spatial databases, and quadtree indexing. 

\paragraph{Remark:} While the Peano and Hilbert curves are most naturally defined on square grids with power-of-two dimensions, they can be easily extended to arbitrary image sizes by truncating higher-order bits, using padding, clipping, or floating-point mapping techniques~\citep{gilbert,general_sfc}. In \cref{fig:hilbert_morton}, we visually show how to extend these curves to non-power-of-2 cases with codes provided in \cref{curve_codes}.

\begin{figure}[t]
  \centering
  \begin{subfigure}[b]{0.24\textwidth}
    \includegraphics[width=\linewidth]{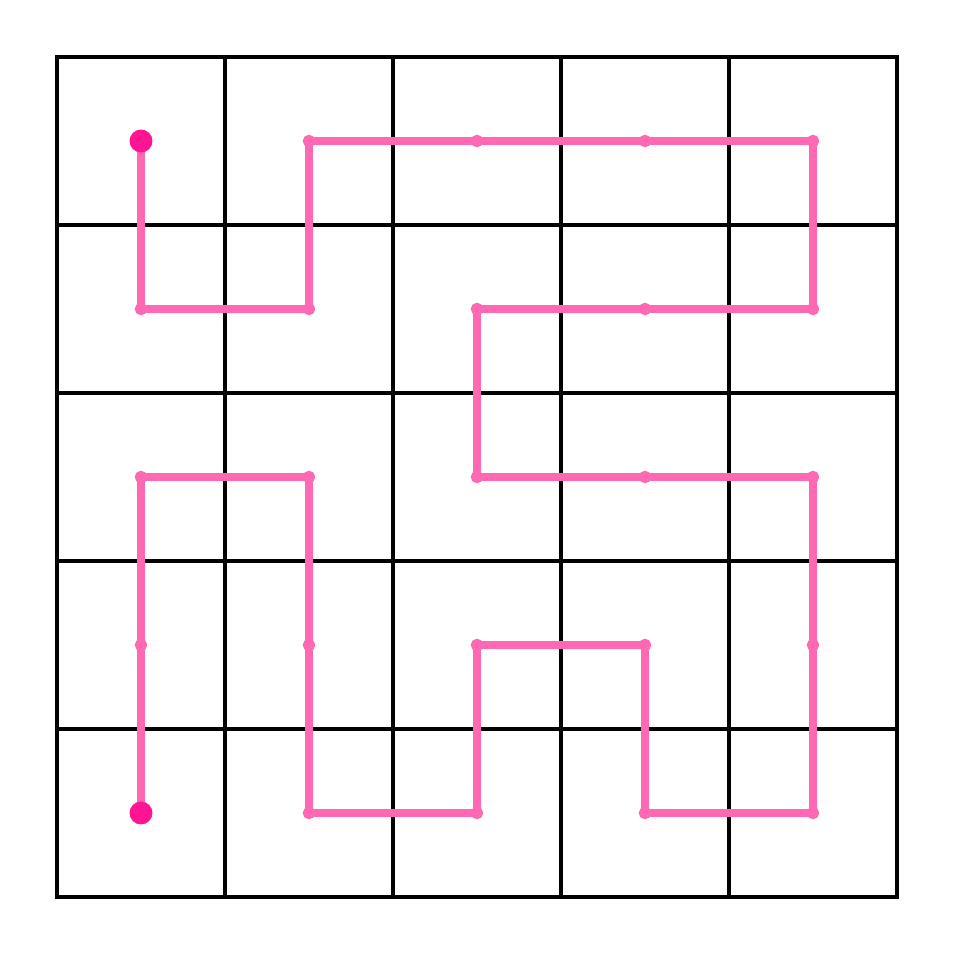}
    \caption{Hilbert on $5\times5$ grid.}
    \label{fig:sub1}
  \end{subfigure}
  \hfill
  \begin{subfigure}[b]{0.24\textwidth}
    \includegraphics[width=\linewidth]{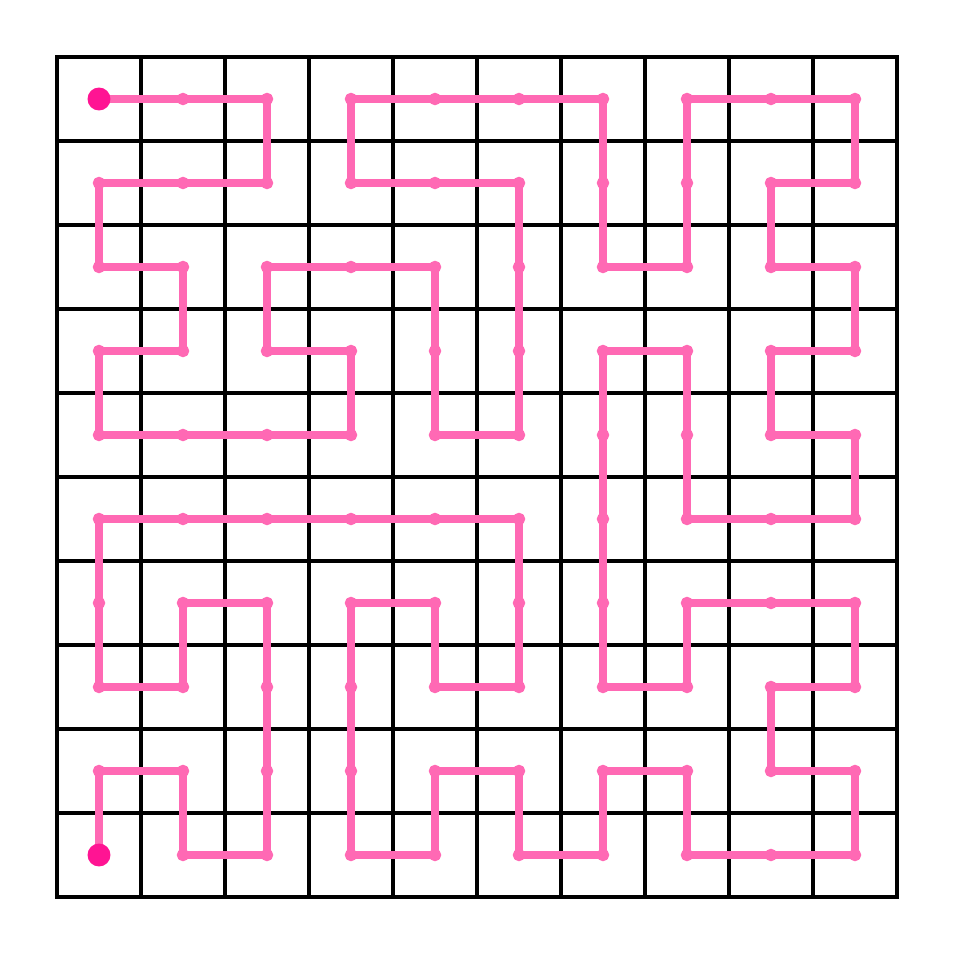}
    \caption{Hilbert on $10\times10$ grid.}
    \label{fig:sub2}
  \end{subfigure}
  \hfill
  \begin{subfigure}[b]{0.24\textwidth}
    \includegraphics[width=\linewidth]{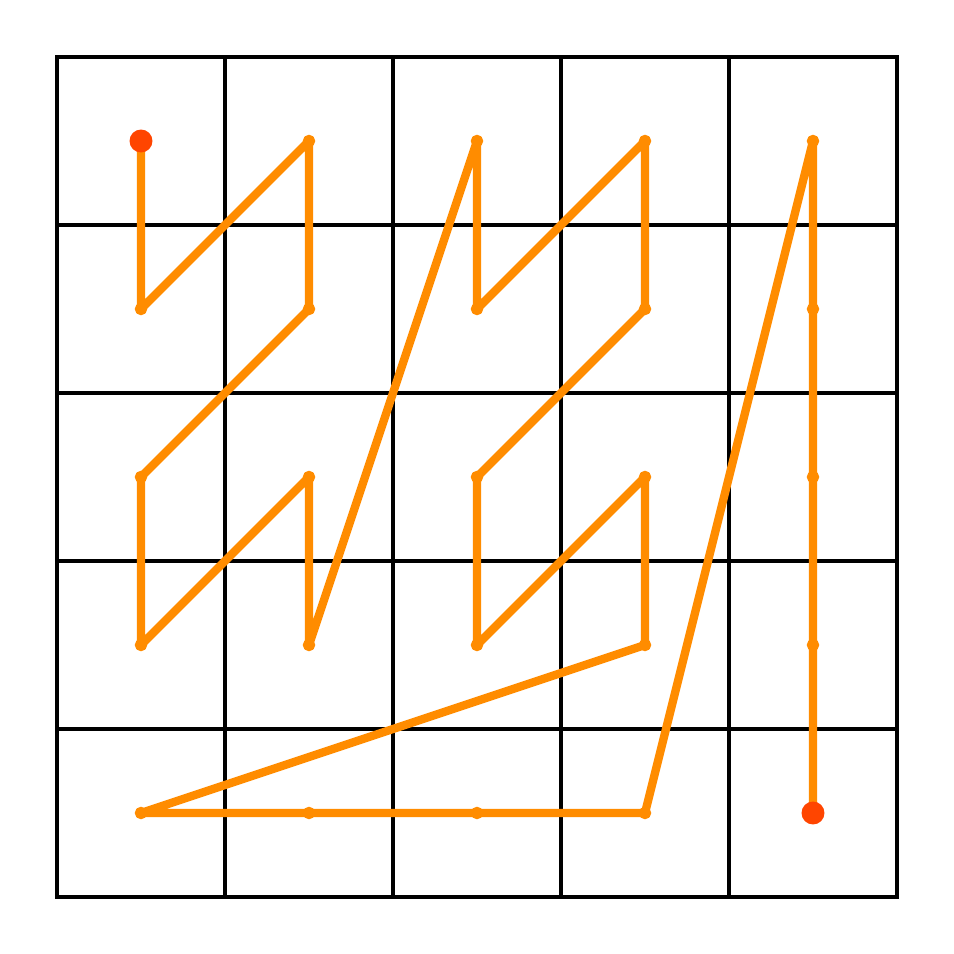}
    \caption{Peano on $5\times5$ grid.}
    \label{fig:sub3}
  \end{subfigure}
  \hfill
  \begin{subfigure}[b]{0.24\textwidth}
    \includegraphics[width=\linewidth]{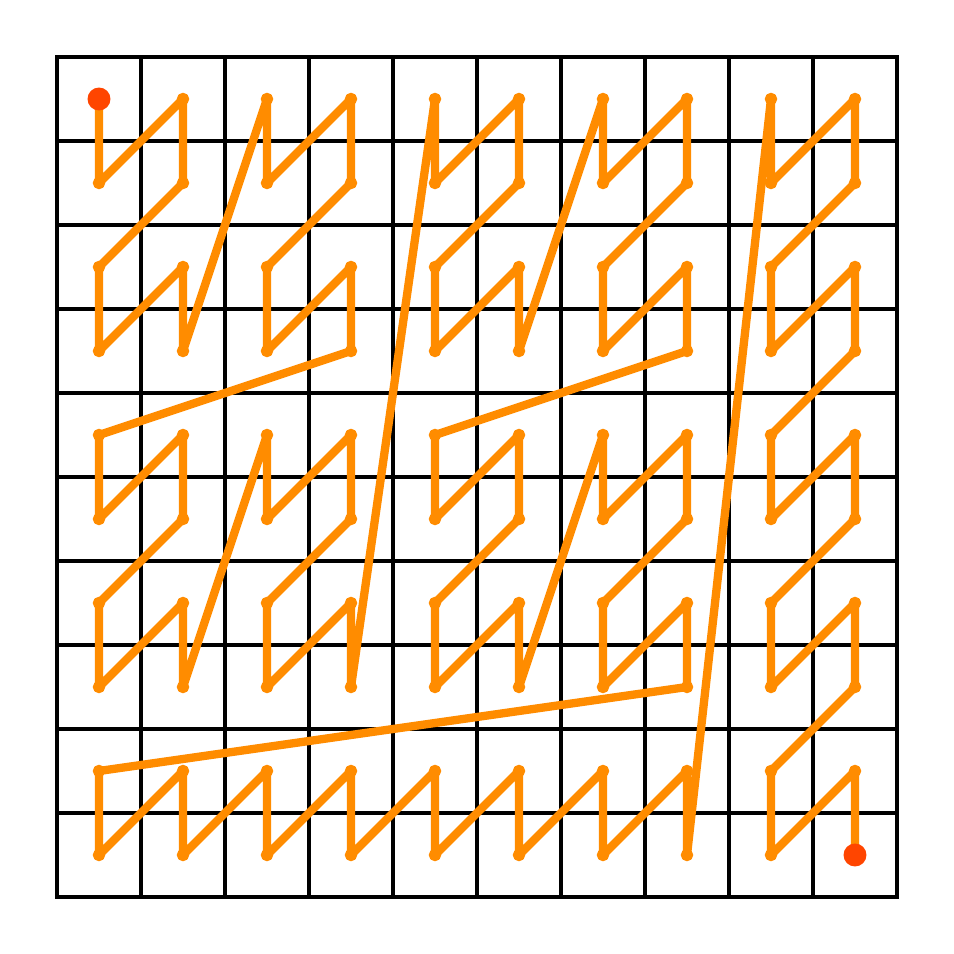}
    \caption{Peano on $10\times10$ grid.}
    \label{fig:sub4}
  \end{subfigure}
    \caption{\textit{Extension of Hilbert and Peano curves:} Visualization of how Hilbert and Peano curves extend to non-power-of-2 grids.}
  \label{fig:hilbert_morton}
\end{figure}

\paragraph{Flattening with transposed curves} \label{sfc_trans}
Standard SFCs are typically defined over fixed scans using row-major or column-major orderings. To increase the diversity of locality preserving patterns without incurring additional cost, we introduce transposed variants of standard SFCs—such as column-major Snake or vertical Zig-Zag. These variants simply swap coordinates during traversal. We define the flattened image under a transposed curve as:
\begin{equation} \label{rot_F}
   \mathbf{X}_{c^\top}[n] = \mathcal{I}[i, j] \quad \text{where} \quad n=F_{c\top}(i, j) = F_c(j,i).
\end{equation}
Accordingly, we expand our curve set to include these rotated versions, resulting in the final \ours curve set:
\begin{align}
    \mathcal{C} = \{\text{Snake, Zig-Zag, Peano, Hilbert, Snake$^{\top}$, Zig-Zag$^{\top}$, Peano$^{\top}$, Hilbert$^{\top}$}\}
\end{align}

\subsection{Locality via decay mask} \label[appendix]{ap:relpos}
\paragraph{Decay mask structure}  An example of a \(4 \times 4\) causal decay mask with non-input-dependent decay factor, as used in RetNet~\citep{retnet}, is
\begin{align}
    \mathbf{M_\text{Causal}} =
\begin{bmatrix}
1           &     &   & \\
\gamma      & 1            &        &  \\
\gamma^2    & \gamma       & 1            &  \\
\gamma^3    & \gamma^2     & \gamma        & 1
\end{bmatrix}, \qquad
    \mathbf{M_\text{Causal}}[i,j] =  \begin{cases} 
    \gamma^{i-j} & i \geq j  \\
    0 & i < j
\end{cases}
\end{align}
As seen in the causal decay mask above, the decay masking the attention \( \mathbf{M_\text{Causal}}[i,j] \) depends only on the difference between \( i \) and \( j \), specifically \( \mathbf{M_\text{Causal}}[i,j] = \gamma^{|i-j|} \). which reflects the locality information in the causal decay mask.

As an extension for bidirectional tasks, such as image classification, the causal mask can be extended to a full Toeplitz decay mask, as shown in~\citep{lion}:
\begin{align}
    \mathbf{M} =
\begin{bmatrix}
1           &    \gamma  & \gamma ^2  & \gamma ^3 \\
\gamma      & 1            &   \gamma      &   \gamma ^2\\
\gamma^2    & \gamma       & 1            &\gamma   \\
\gamma^3    & \gamma^2     & \gamma        & 1
\end{bmatrix}, \qquad
    \mathbf{M}[i,j] =  \gamma^{|i-j|}
\end{align}
 in this case, the attention between each pair of tokens \( i \) and \( j \) is masked based on their distance \( |i-j| \). Additionally, the decay factor \( 0 < \gamma < 1 \) is bounded between to ensure that \( \mathbf{M}[i,j] \) does not overflow and remains stable~\citep{lru}. %

\paragraph{Extrapolation capabilities of decay mask} The decay mask \( \mathbf{M} \) can easily be extrapolated beyond the context length~\citep{mamba2, retnet} because \( \mathbf{M}[i,j] = \gamma^{|i-j|} \) is independent of the sequence length. This is especially useful since we can change the resolution of images during inference without needing to interpolate or extrapolate the position embeddings~\citep{vit, dino}. This capability is particularly valuable when generating videos for object tracking in \ours DINO.

\subsection{Efficiency of Toeplitz decay mask} \label{ap:toeplitz}

As mentioned in the background \cref{bckg:linear}, the decay parameter $\gamma$ can be input dependent as well, which means that it is extracted for each token as:
\begin{align} \label{eqfox}
    \gamma_i = g(\mathbf{W}_\gamma \mathbf{x}_i) , \quad \mathbf{M}[i,j] = \gamma_j\gamma_{j+1}...\gamma_{i} =\prod_{k=j}^{i} \gamma_k
\end{align}
with $g(.)$ being a bounded function such that $0<g(x)<1$ (i.e. sigmoid). This results in each element of the decay mask \( \mathbf{M}[i,j] \) representing the cumulative product of decay contributions from all tokens between positions \( i \) and \( j \) leading to input-dependent decay masks. While these type of masks can offer finer-grained control, they are slower to train, requiring \( \mathcal{O}(\log(N)) \) time points to compute~\citep{mamba, mamba2}, consume more memory, and must be dynamically constructed during inference. In contrast, input-independent decay masks such as the one used in \ours are much more efficient. We adopt the decay mask in \ours as it is faster to train, memory-efficient (requiring only a single learned scalar \( \gamma \) per curve), and eliminates the need for recomputation during inference. This simple scalar-based design still performs effectively and achieves strong results in practice~\citep{lion}.

\subsection{Connections of \ours to other models} \label{ap:connections}

As \ours is inspired by the forget gate (also known as the decay mask) in Linear Transformers, it shares strong connections with these models and their adaptations for vision tasks. Below, we highlight some of the most relevant connections:

\paragraph{RMT} RMT~\citep{fan2024rmt} also introduces a decay mask (via Manhattan distance) to enhance the spatial awareness of ViTs, addressing a similar challenge. However, it differs from \ours in key ways. RMT uses only a single flattening strategy and applies a fixed distance metric (Manhattan), while \ours generates multiple masks based on different SFCs and defines a Kac–Murdock–Szegő (KMS) matrix for the decay. Architecturally, \ours is a modular attention mechanism that can be plugged into various ViT backbones, whereas RMT is a standalone model. We also conducted an ablation using the Manhattan distance decay as in RMT, and found it underperforms compared to \ours. Detailed results are provided in \cref{tab:extra_curves}.

\paragraph{FoX} FoX, or Forgetting Transformer~\citep{fox}, is designed for causal sequence modeling, specifically to capture long-range dependencies in the NLP domain. It uses an input-dependent causal decay mask, as shown in \cref{eqfox}, which differs significantly from \ours in both application domain and mask design. Moreover, the perspective central to \ours, based on flattening and scanning via space-filling curves, does not appear in FoX, as it operates in the NLP setting rather than vision tasks.

\paragraph{Vision Linear Transformer} This class includes models such as Vision LSTM~\citep{alkin2024visionlstm}, Vision Mamba~\citep{zhu2024visionmambaefficientvisual}, and VMamba~\citep{liu2024vmamba}, which are related to \ours due to their use of different scanning strategies primarily based on the Z-curve in both standard and transposed (horizontal and vertical) directions. However, these models significantly differ from \ours in architecture, as they are based on SSMs like Mamba~\citep{mamba} or other linear attention mechanisms, rather than softmax-based Transformers. In contrast, \ours is a softmax-based masked attention module that can be easily integrated into various ViT backbones. In this study, we apply \ours to DeiT, DeiT-III, and DINO as representative examples.

\paragraph{MAE} Masked Auto Encoders (MAE)~\citep{He2022MAE} apply random input masking as a pretraining objective, dropping patches and training the model to reconstruct them. This masking affects only the input and does not influence attention computation. In contrast, \ours applies structured masking within the attention mechanism, using decay masks based on space-filling curves to rescale attention scores, without dropping tokens or reconstructing inputs. It serves as a spatial inductive bias, guiding the model to attend more to nearby regions without altering the input or training objective.

\section{Proofs}\label{ap:proofs}

\subsection{Attention is permutation equivariant} \label[appendix]{permutation_equivariance}
\begin{claim}
    Attention without positional embeddings is permutation-equivariant. That is,
    \begin{align}
        A(\pi(\mathbf{X})) = \pi(A(\mathbf{X}))
    \end{align}
    where \( A(\cdot) \) is the output of the attention mechanism, and \( \pi(\cdot) \) denotes a permutation of the sequence.
\end{claim}
\begin{proof} \label{permuteequi}
Let \( \mathbf{X} \in \mathbb{R}^{N \times d} \) be the input sequence with \( N \) tokens and model dimension \( d \). The attention is defined as
\begin{align}
    \mathbf{Q} = \mathbf{X} \mathbf{W}_Q, \quad \mathbf{K} = \mathbf{X} \mathbf{W}_K, \quad \mathbf{V} = \mathbf{X} \mathbf{W}_V, \qquad
    A(\mathbf{X}) = \text{Softmax}\left(\frac{\mathbf{Q} \mathbf{K}^\top}{\sqrt{d}}\right) \mathbf{V}.
\end{align}
Let \( \pi \) be a permutation of the input sequence, represented by a permutation matrix \( \mathbf{P} \in \mathbb{R}^{N \times N} \) such that \( \pi(\mathbf{X}) = \mathbf{P} \mathbf{X} \) and $\mathbf{P}\mathbf{P}^\top = \mathbf{I}$. Then
\begin{align} \label{permqkv}
    \pi(\mathbf{Q}) = \mathbf{P} \mathbf{X} \mathbf{W}_Q = \mathbf{P} \mathbf{Q}, \quad \pi(\mathbf{K}) = \mathbf{P} \mathbf{K}, \quad \pi(\mathbf{V}) = \mathbf{P} \mathbf{V}.
\end{align}
Now compute the attention on the permuted input
\begin{align}
    A(\pi(\mathbf{X})) &= \text{Softmax}\left( \frac{(\mathbf{P} \mathbf{Q})(\mathbf{P} \mathbf{K})^\top}{\sqrt{d}} \right)(\mathbf{P} \mathbf{V}) 
    = \text{Softmax}\left( \frac{\mathbf{P} \mathbf{Q} \mathbf{K}^\top \mathbf{P}^\top}{\sqrt{d}} \right)\mathbf{P} \mathbf{V}
\end{align}
Since softmax is applied row-wise and permutation matrices preserve row-wise operations, we can factor \( \mathbf{P} \) out
\begin{equation}
    \scalebox{1}{$
A(\pi(\mathbf{X})) = \mathbf{P} \, \text{Softmax}\left( \frac{\mathbf{Q} \mathbf{K}^\top}{\sqrt{d}} \right)
\underset{\mathbf{I}}{\cancel{\mathbf{P}^\top\mathbf{P}}}
\mathbf{V}=
 \mathbf{P} \, \text{Softmax}\left( \frac{\mathbf{Q} \mathbf{K}^\top}{\sqrt{d}} \right)\mathbf{V} =
 \mathbf{P} A(\mathbf{X}) = \pi(A(\mathbf{X}))$}
\end{equation}
Thus, attention is permutation-equivariant in the absence of positional embeddings.
\end{proof} 

\subsection{SFCs in decay mask are a distance metric} \label{ap:sfc_distance}
\begin{claim}
Let \( \mathbf{X}_{c_1} \in \mathbb{R}^{N \times d} \) be the flattened image using a space-filling curve \( c_1 \), with the sequence indexed by \( i,j,k \in \{0, \dots, N-1\} \). Any permutation \( \pi_{c_2} \), corresponding to a new flattening order defined by a different curve \( c_2 \), when applied to \( \mathbf{X}_{c_1} \), induces a new sequence order. In this new order, the term \( |\pi(i) - \pi(j)| \) satisfies the non-negativity, identity of indiscernibles, symmetry and triangle inequality properties of a distance metric between tokens \( i \) and \( j \).
\end{claim}

\begin{proof}
To show that \( |\pi(i) - \pi(j)| \) is a valid distance metric, we verify that it satisfies the standard properties of a metric:

\hspace*{2em}\textit{Non-negativity:}  
    For all \( i, j \), we have  
    \begin{equation}
    |\pi(i) - \pi(j)| \geq 0
    \end{equation} 
    \hspace*{2em}since absolute values are always non-negative.

\hspace*{2em}\textit{Identity of indiscernibles:}  
    \begin{equation}
    |\pi(i) - \pi(j)| = 0 \iff \pi(i) = \pi(j) \iff i = j
    \end{equation}  
    \hspace*{2em}because \( \pi \) is a permutation (i.e., a bijective function), so \( \pi(i) = \pi(j) \) implies \( i = j \).

\hspace*{2em}\textit{Symmetry:}  
    \begin{equation}
    |\pi(i) - \pi(j)| = |\pi(j) - \pi(i)|
    \end{equation}  
    \hspace*{2em}by the symmetry of absolute value.

\hspace*{2em}\textit{Triangle inequality:}  
    For any \( i, j, k \in \{0, \dots, N-1\} \),
    \begin{equation}
    |\pi(i) - \pi(j)| \leq |\pi(i) - \pi(k)| + |\pi(k) - \pi(j)|
    \end{equation}
    \hspace*{2em}holds due to the triangle inequality property of absolute values.

Therefore, \( |\pi(i) - \pi(j)| \) satisfies all the conditions of a distance metric. This property is particularly interesting because the term \( |\pi(i) - \pi(j)| \) appears as the exponent in the decay mask, leading to \( \mathbf{M}_{c_2}[i,j] = \gamma^{|\pi(i) - \pi(j)|} \). As a result, taking the logarithm of the decay mask yields a distance matrix,
$
\log(\mathbf{M}_{c_2}[i,j]) = |\pi(i) - \pi(j)| \cdot \log(\gamma)
$
thus, \( \log(\mathbf{M}_{c_2}) \) is a scaled distance matrix, encoding relative positional distances under the permutation induced by curve \( c_2 \).
\end{proof}

\subsection{\ours SFC flattening only reflects in decay mask}\label{Mc_equality}
\begin{claim} 
    Let the input sequence flattened using a base space-filling curve (e.g., Z-curve) be denoted by \( \mathbf{X} \in \mathbb{R}^{N \times d} \), and let the output of \ours attention be \( \mathbf{Y} \in \mathbb{R}^{N \times d} \), computed as:
    \begin{align} \label{viatt}
        \mathbf{Y} = \text{Softmax}\left(\alpha \frac{\mathbf{QK}^\top}{\sqrt{d}} \odot \mathbf{M} \right) \mathbf{V}
    \end{align}
    where \( \mathbf{M} \in \mathbb{R}^{N \times N} \) is the base decay mask with entries \( \mathbf{M}[i,j] = \gamma^{|i - j|} \).

    Now, let \( \mathbf{X}_c = \pi_c(\mathbf{X}) \) be the input sequence reordered using a space-filling curve \( c \), with permutation \( \pi_c \). Then, the output of the \ours attention for the permuted input \( \mathbf{X}_c \), re-ordered back to the original (basis) input order, is given by:
    \begin{align} \label{viatt_pre2}
        \mathbf{\widetilde{Y}} = \text{Softmax}\left(\alpha \frac{\mathbf{QK}^\top}{\sqrt{d}} \odot \pi_c(\mathbf{M}) \right) \mathbf{V}
    \end{align}
    where \( \pi_c(\mathbf{M}) = \mathbf{M}[\pi_c(i), \pi_c(j)] \) denotes the decay mask permuted along both rows and columns according to the curve \( c \).
\end{claim}

\begin{proof}
It is easy to see that flattening the input \( \mathcal{I} \) into a sequence \( \mathbf{X}_{c_1} \) using any space-filling curve \( c_1 \) defines a one-to-one mapping from the 2D grid to a 1D sequence. Therefore, there exists a permutation \( \pi_{c_1 \rightarrow c_2} \) and an associated permutation matrix \( \mathbf{P}_{c_1 \rightarrow c_2} \) such that the sequence obtained by flattening with another curve \( c_2 \) is given by:
\begin{equation}    
\mathbf{X}_{c_2} = \mathbf{P}_{c_1 \rightarrow c_2} \, \mathbf{X}_{c_1}
\end{equation}
Now, considering \( c_1 \) as the z-Curve (our basis flattening), and renaming \( c_2 \) simply as \( c \), we simplify the notation as follows:
\begin{equation} 
\pi_{c_1 \rightarrow c_2} = \pi_c, \quad \mathbf{P}_{c_1 \rightarrow c_2} = \mathbf{P}_c, \quad \mathbf{X}_c = \pi_c(\mathbf{X}) = \mathbf{P}_c \mathbf{X}
\end{equation}
From  \cref{permqkv} we know that permuting the input $\mathbf{X}$ will result in permutation of query, key and value matrices so for the input $\mathbf{X}_c$ the  attention presented at \cref{viatt} is re-written as:

    \begin{align}  
        \mathbf{Y}_c &= \text{Softmax}\left(\alpha \frac{\mathbf{\pi_c(Q)\pi_c(K)}^\top}{\sqrt{d}} \odot \mathbf{M} \right) \mathbf{\pi_c(V)} \notag  \\ &=  \text{Softmax}\left(\alpha \frac{\mathbf{\mathbf{P}_c Q(\mathbf{P}_c K)}^\top}{\sqrt{d}} \odot \mathbf{M} \right) \mathbf{{P}_c V}\notag  \\ &= \text{Softmax}\left(\alpha \frac{\mathbf{\mathbf{P}_c(QK^\top)\mathbf{P}_c}^\top}{\sqrt{d}} \odot \mathbf{M} \right) \mathbf{{P}_c V }
    \end{align}
by multiplying $\mathbf{P}_c\mathbf{P}_c^\top$ to both sides of $\mathbf{M}$ we have:
    \begin{align}  \label{eq:vilatt_per}
        \mathbf{Y}_c &= \text{Softmax}\left(\alpha \frac{\mathbf{\mathbf{P}_c(QK^\top)\mathbf{P}_c}^\top}{\sqrt{d}} \odot \mathbf{P}_c\mathbf{P}_c^\top\mathbf{M}\mathbf{P}_c\mathbf{P}_c^\top \right) \mathbf{{P}_c(V)} \\&= \text{Softmax}\left(\alpha \frac{\mathbf{\mathbf{P}_c(QK^\top)\mathbf{P}_c}^\top}{\sqrt{d}} \odot \mathbf{P}_c(\mathbf{P}_c^\top\mathbf{M}\mathbf{P}_c)\mathbf{P}_c^\top \right) \mathbf{{P}_c(V)}
    \end{align}
Since the multiplication with the decay mask and the softmax operation are element-wise (i.e., applied row-wise for each query), the permutation matrices \( \mathbf{P}_c \) and \( \mathbf{P}_c^\top \) can be factored out of the attention computation. This results in the following expression:
\begin{equation} \label{eq:vilatt_final}
    \scalebox{1}{$ 
     \mathbf{Y}_c = \mathbf{P}_c  \text{Softmax}\left( \alpha \frac{\mathbf{QK}^\top}{\sqrt{d}} \odot \mathbf{P}_c^\top \mathbf{M} \mathbf{P} _c\right)\underset{\mathbf{I}}{\cancel{\mathbf{P}_c^\top\mathbf{P}_c}} \mathbf{V} =  \mathbf{P}_c  \text{Softmax}\left( \alpha \frac{\mathbf{QK}^\top}{\sqrt{d}} \odot \underbrace{\mathbf{P}_c^\top \mathbf{M} \mathbf{P}_c}_{\pi_c^{-1}(\mathbf{M)}} \right) \mathbf{V} $}
\end{equation}
Since the order of \( \mathbf{Y}_c \) corresponds to the permuted input \( \mathbf{X}_c \), we can recover the output in the original (basis) order by applying the inverse permutation, i.e., multiplying by \( \mathbf{P}_c^\top \). Therefore, the final output \( \mathbf{\widetilde{Y}_c} \) aligned with the original input \( \mathbf{X} \) is:
\begin{align} \label{eq:vilatt_recovered}
    \mathbf{\widetilde{Y}_c}  = \mathbf{P}_c^\top \mathbf{Y}_c = \text{Softmax}\left( \alpha \frac{\mathbf{QK}^\top}{\sqrt{d}} \odot \mathbf{P}_c^\top \mathbf{M} \mathbf{P} _c\right) \mathbf{V}
\end{align}
This confirms that applying attention to a permuted input using the base decay mask is equivalent to applying attention to the original input with a permuted (reordered) decay mask \( \pi_c^{-1}(\mathbf{M}) = \mathbf{P}_c^\top \mathbf{M} \mathbf{P}_c \).
    
\end{proof}

\begin{figure}[t]
    \centering
    \includegraphics[width=1\linewidth]{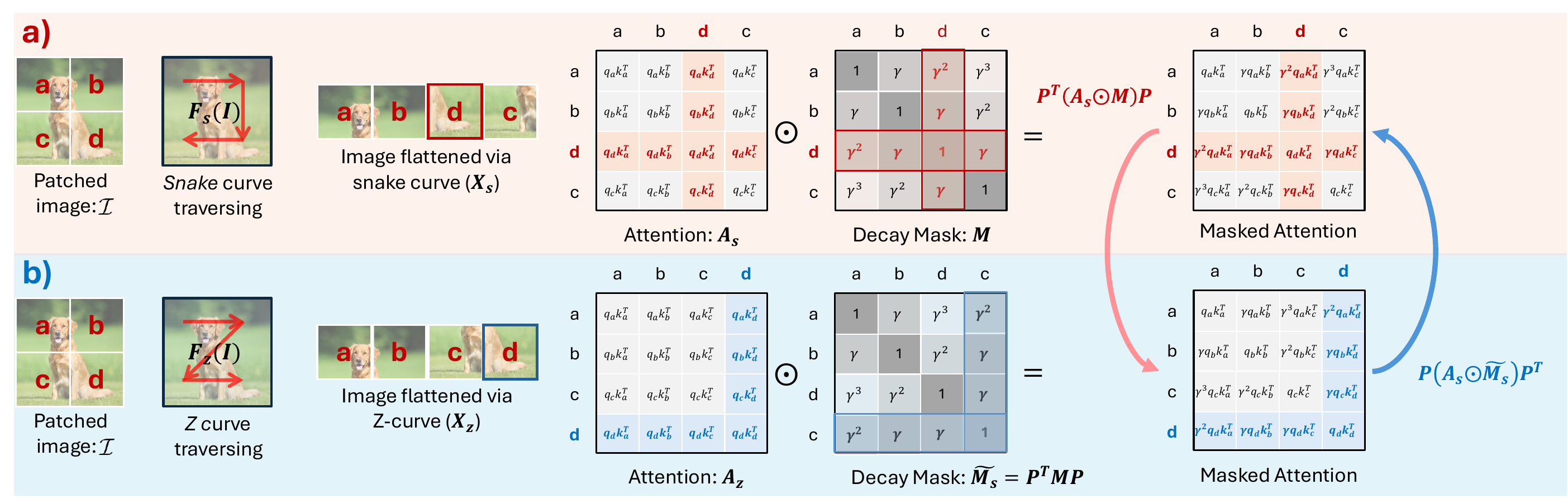}
    \caption{\textit{Effect of SFCs on flattened Image:} Visually showing the equivalence between \textbf{a)} Permuting the input sequence according to \( c \) (e.g., the snake curve) to get $\mathbf{X}_S$, multiplying the attention $\mathbf{A}_S$ with the original decay mask defined in the basis curve $\mathbf{M}$ (e.g., Z-curve in our study), and then reordering the output back to the original and \textbf{b)} Calculating attention $\mathbf{A}_z$ with basis curve ordered $\mathbf{X}_z$, using a permuted decay mask \( \widetilde{\mathbf{M}_c} \). }
    \label{fig:equl_permute}
\end{figure}

This proof is also visualized in \cref{fig:equl_permute}, illustrating that applying attention using a permuted decay mask based on curve \( c \) (e.g., the snake curve in the figure) is equivalent to permuting the input sequence according to \( c \), computing attention with the original decay mask defined in the basis curve (e.g., Z-curve in our study), and then reordering the output back to the original sequence order. 

\paragraph{Disclaimer} In practice, it is unnecessary to explicitly define a permutation function $\pi$ or construct a matrix $\mathbf{P}$. The reordering can be efficiently achieved by simply storing the corresponding indices. $\mathbf{P}$ and $\pi$ are used for mathematical clarity and formalism only.

\subsection{Averaging multiple SFC decay masks}\label[appendix]{ap:sfc_avg}
\begin{claim}
Let $C=\{c_1,\dots,c_m\}$ be a fixed set of space-filling curves, each inducing a permutation $\pi_c$ over $N$ tokens and a decay mask
\begin{equation}
\mathbf{M}_c[i,j] \;=\; \gamma^{|\pi_c(i)-\pi_c(j)|}, \qquad \gamma\in(0,1).
\end{equation}
Define the averaged decay mask
\begin{equation}
\overline{\mathbf{M}}[i,j] \;=\; \frac{1}{m}\sum_{c\in C}\mathbf{M}_c[i,j].
\end{equation}
Then, for any token pair $(i,j)$, averaging over $\mathcal{C}$ yields:
\begin{enumerate}
    \item \textbf{(Reduced sensitivity to individual curves)}  
    The influence of any single curve on $\overline{\mathbf{M}}[i,j]$ is bounded by $1/m$.
    \item \textbf{(Robust preservation of local interactions)}  
    If at least a fraction $p\in(0,1]$ of curves satisfy
    \begin{equation}
    |\pi_c(i)-\pi_c(j)| \le r,
    \end{equation}
    then the averaged mask obeys the lower bound
    \begin{equation}
    \overline{\mathbf{M}}[i,j] \;\ge\; p\,\gamma^{r}.
    \end{equation}
\end{enumerate}
Consequently, $\overline{\mathbf{M}}$ provides a more stable and expressive affinity prior than any single-curve mask.
\end{claim}

\begin{proof}
We prove the two statements.

\hspace*{2em}\textit{(1) Reduced sensitivity to individual curves:}  
Fix $(i,j)$ and define $x_c := \mathbf{M}_c[i,j]\in(0,1]$. By definition,
\begin{equation}
\overline{\mathbf{M}}[i,j] \;=\; \frac{1}{m}\sum_{c\in C} x_c.
\end{equation}
If the value of one curve $c^\star$ is perturbed from $x_{c^\star}$ to $x_{c^\star}'$, while all others remain fixed, then
\begin{equation}
\left|\overline{\mathbf{M}}'[i,j]-\overline{\mathbf{M}}[i,j]\right|
\;=\;
\frac{1}{m}\left|x_{c^\star}'-x_{c^\star}\right|
\;\le\;
\frac{1}{m},
\end{equation}
since $x_c\in(0,1]$ for all $c$. Thus, no single curve can dominate the averaged mask, and the effect of any outlier curve is suppressed by a factor $1/m$.

\hspace*{2em}\textit{(2) Robust preservation of local interactions:}  
Assume that for at least $pm$ curves in $\mathcal{C}$ we have $|\pi_c(i)-\pi_c(j)|\le r$. For each such curve,
\begin{equation}
\mathbf{M}_c[i,j] = \gamma^{|\pi_c(i)-\pi_c(j)|} \;\ge\; \gamma^{r},
\end{equation}
since $\gamma\in(0,1)$ and $\gamma^t$ is monotonically decreasing for $t\ge 0$.

Summing over all curves yields
\begin{equation}
\sum_{c\in C}\mathbf{M}_c[i,j]
\;\ge\;
pm\cdot \gamma^{r}.
\end{equation}
Dividing by $m$ gives
\begin{equation}
\overline{\mathbf{M}}[i,j]
\;\ge\;
p\,\gamma^{r},
\end{equation}
which establishes the claimed lower bound.

Finally, observe that each single-curve mask $\mathbf{M}_c$ depends on one induced one-dimensional distance $|\pi_c(i)-\pi_c(j)|$, whereas the averaged mask $\overline{\mathbf{M}}$ aggregates multiple such distances. Hence, $\overline{\mathbf{M}}$ encodes interactions that are consistently local across several curves while attenuating interactions that appear local only under a single permutation. This yields a more stable and expressive affinity structure.
\end{proof}

\paragraph{Remark}
Although the above result is stated for a fixed set of curves $\mathcal{C}$, it also has a natural probabilistic interpretation. If the curves in $\mathcal{C}$ are viewed as samples from an underlying distribution over space-filling curve orderings, then the averaged mask $\overline{\mathbf{M}}[i,j]$ corresponds to the empirical mean of the random variable $\mathbf{M}_c[i,j]=\gamma^{|\pi_c(i)-\pi_c(j)|}$. In this case, standard results imply that the variance of the empirical mean decreases proportionally to $1/|C|$. This interpretation provides additional intuition: averaging multiple curves reduces the variability induced by any single ordering and yields a more stable estimate of spatial affinity.

\newpage
\section{Futher design details} \label[appendix]{sec:design} 
In this section, we outline key design choices made in the implementation of \ours models. 
\subsection{Initialization} \label[appendix]{ap:init} 
Since $\gamma_c =\text{sigmoid}(\beta_c)$ is exponentiated over the sequence length in the \ours decay mask, it is important to initialize it close to 1, which is also highlighted in the Linear Transformer literature~\citep{lru,retnet}. For pretraining \ours models, we initialize $\beta_c$ uniformly in the range $[5, 9]$, which corresponds to $\gamma_c \in (0.9820, 0.9998)$. This ensures that the initial mask values $\mathbf{M}_c[i,j] \in (0.03, 0.962)$ for $N = 196$, maintaining a stable and controlled decay. For numerical results on the effect of initialization, see \cref{ap:other_ablations}.

During full fine-tuning, we initialize the model using the pretrained baseline. In this setting, since the query/key/value weights $\mathbf{W_Q}, \mathbf{W_K}, \mathbf{W_V}$ are already trained during pretraining and \ours attention is introduced and used only at fine-tuning, we initialize the scaling factor $\alpha$ using a Gaussian distribution centered at 1 to allow for smooth adaptation. For $\beta_c$, we use a uniform initialization in the range $[15, 20]$. This setup avoids a steep drop in attention scores while allowing the model to gradually adapt to the newly introduced decay mask $\mathbf{M}_\ours$.  All other initialization settings in \ours exactly follow those of the original baselines without any modification.

\textit{All other configurations, such as data augmentation, optimizer, initialization, model parameters, and training setups are kept exactly the same as in the original baselines, with no modifications.}

\subsection{Adaptation of \ours to various architectures} \label[appendix]{ap:other_arch} 
\ours attention supports both the use of a classification token and Global Average Pooling (GPA)~\citep{lin2013network,lu2022bridging}. For pretraining of DeiT models, we remove the classification token and instead apply Global Average Pooling (GAP). The attention module is replaced with \ours attention, while the rest of the model, including positional embeddings, layer normalization, and other components, remains unchanged, see \cref{abl:gap} for details. For fine-tuning the classification token remains intact.

In the DINO setting, both teacher and student models are initialized with \ours attention, with all other weights handled as usual. Due to the multi-crop training, the attention module encounters varying sequence lengths. However, since the construction of \( \mathbf{M}_\text{\ours} \) naturally adapts to any sequence length, this poses no issue.  

To accommodate the classification token, we modify the corresponding rows and columns of \( \mathbf{M}_\text{\ours} \) by setting \( \gamma_\text{cls} = 1 \). We also experimented with a learnable \( \gamma_\text{cls} \in [0,1] \) but observed no significant performance gains. The rest of the model structure follows the original DINO architecture.  

\paragraph{\ours with hierarchical and convolutional architectures}
Hierarchical transformer architectures such as Swin~\citep{liu2021swin} and convolutional-transformer hybrids like PVT~\citep{pvt} differ fundamentally from vanilla ViTs in how attention is computed. Instead of applying full attention across the entire sequence, they restrict the receptive field by using windowed or spatially localized attention, often combined with hierarchical feature maps. This design introduces locality explicitly into the architecture, reducing the need for additional spatial priors such as those provided by SFCs.

In such settings, applying SFC-guided decay masks becomes problematic for two main reasons. First, SFCs are meaningful when attention spans the \textit{entire} sequence of image patches, since the curve defines a global traversal order. In hierarchical models, however, attention is restricted to local windows or pyramid levels, where the notion of a global SFC ordering no longer applies. Second, many of these architectures already incorporate inductive biases (through localized windows, shifting strategies, or convolutional layers), so introducing additional SFC-based priors could interfere with rather than complement their design.

Thus, \ours is best suited for standard ViTs and related architectures where attention is fully global, the sequence is flattened in a fixed order (commonly the Z-curve), and inductive biases are otherwise minimal. In contrast, hierarchical or convolutional variants already bake spatial priors directly into their architecture, making SFC-based masking redundant or ill defined.

Consistent with our analysis, when we integrated \ours into Swin at tiny and small scales during pretraining, we achieved minimal accuracy improvements of $0.2\%$ and $0.1\%$, respectively, as shown in \cref{tab:swin}. The \ours mask is applied at every stage and layer, with each mask being independently learned and unique to its respective layer. The remaining architecture follows the original Swin model structure.  
\begin{table}[h]
\centering
\caption{\textit{Pretraining of Swin models:} The performance of baseline model is compared against \ours for ImageNet pretraining. Changes with respect to the baseline are shown inside $(\cdot)$ next to the accuracies.}
\label{tab:swin}
\begin{tabular}{l>{\columncolor{gray!20}}c>{\columncolor{orange!30}}c}
\noalign{\hrule height 1.5pt}
\multirow{2}{*}{Model}  & \multicolumn{2}{c}{Top-1 Accuracy ($\%$)} \\
& Baseline & \ours  \\ 
\hline
Swin-T & 81.3 & \textbf{81.5} \textcolor{g}{(\textbf{+0.2)}} \\
Swin-S & 83.0 & \textbf{83.1} \textcolor{g}{(\textbf{+0.1})} \\
\noalign{\hrule height 1.5pt}
\end{tabular}
\end{table}

\paragraph{\ours with video transformers}
Video transformers operate on spatiotemporal tokens, and \ours can be incorporated into these models in a straightforward way because it only rescales the attention scores between tokens. There are two natural ways to extend \ours:
\begin{enumerate}
    \item \textbf{Spatial-only SFCs (2D per frame) } 
    The same 2D SFCs used for images can be applied independently to the $(H, W)$ grid of each frame, while keeping the temporal dimension unchanged. This provides a per-frame spatial prior and mirrors the image setting.
    \item \textbf{Full spatiotemporal SFCs (3D) } 
    Following \cref{def:sfc}, SFCs naturally generalize to arbitrary dimensions. Thus, we can define 3D SFCs over the full $(T, H, W)$ grid (e.g., 3D Hilbert or 3D Morton curves) and compute distances based on each token's original spatiotemporal position. The resulting decay masks encourage locality across both space and time. Masks can be computed once over the full grid and then indexed to the visible token subset, similar to how positional embeddings are handled in VideoMAE~\citep{wang2023videomaev2}.
\end{enumerate}
Both approaches are fully compatible with video MAE-style training: they require no changes to masking or reconstruction objectives, they can be applied to both encoder and decoder, and they provide a meaningful structural prior, especially under high masking ratios where positional structure becomes crucial. Overall, extending \ours to video models is a promising direction for future work, as spatiotemporal SFCs may offer strong inductive bias with minimal additional cost.

\section{Ablation studies} \label[appendix]{ap:ablations}
In this section, we provide comprehensive ablation studies on various elements of \ours. For all ablations, we utilize different scales of DeiT models and we keep the training recipe the same. We use a patch size of 16 and a resolution $224\times224$ for each one of the models.
\subsection{Positional embeddings} \label[appendix]{abl:pos_embd}
To evaluate the impact of positional embeddings, we pretrain the \ours DeiT-B model both with and without them, see \cref{tab:pos}. The results indicate that positional embeddings provide a performance boost, leading us to retain the original positional embedding configurations of the base models. 
\begin{table}[h]
\centering
\caption{\textit{Ablation on positional embeddings (PE):} The performance of the baseline model with PE is compared against \ours with (w) and without (wo) PE. Changes with respect to the baseline are shown inside $(\cdot)$ next to the accuracies.}
\label{tab:pos}
\begin{tabular}{l>{\columncolor{gray!20}}c>{\columncolor{orange!30}}cc}
\noalign{\hrule height 1.5pt}
\multirow{2}{*}{Model}  & \multicolumn{3}{c}{Top-1 Accuracy ($\%$)} \\
& Baseline & \ours w PE & \ours wo PE \\ 
\hline
DeiT-B & 81.8 & \textbf{81.9 \textcolor{g}{(+0.1)}} & 81.5 \textcolor{red}{(-0.3)}\\
\noalign{\hrule height 1.5pt}
\end{tabular}
\end{table}
\subsection{Alternative curve configurations} \label[appendix]{ap:curve_configs}
We examine the individual contribution of each curve to the overall performance. To do so, we pretrain DeiT-S using all possible combinations of the four curves, resulting in $ 2^4 = 16$ variations. The accuracies of each configuration are presented in \cref{tab:curves}. Note that whenever a curve has is used, the transposed version is also included. In other words, if the snake curve is included, its transposed variant Snake$^\top$ is also utilized. 
\begin{table}[t]
\centering
\caption{\textit{Ablation on the effect of each curve:} The performance of the baseline model is compared against \ours with different curve combinations. \gcheck ~indicates the curse is in the set, whereas \xmark ~means it is not. Changes with respect to the baseline are shown inside $(\cdot)$ next to the accuracies.}
\label{tab:curves}
\begin{tabular}{@{}lccccl@{}}
\noalign{\hrule height 1.5pt}
Model  & Snake Curve & Zig-Zag Curve & Hilbert Curve & Peano Curve & Top-1 Acc (\%) \\
\midrule
\rowcolor{gray!20}
DeiT-S (Baseline) & \xmark & \xmark & \xmark & \xmark & 79.9 \\
 & \gcheck& \xmark & \xmark & \xmark & 80.0 {\textcolor{g}{(+0.1)}} \\
 & \xmark &\gcheck & \xmark & \xmark & 80.2 {\textcolor{g}{(+0.3)}}\\
  & \xmark & \xmark &\gcheck & \xmark& 79.9  {\textcolor{g}{  ---}}\\
 & \xmark & \xmark& \xmark& \gcheck& 80.4 {\textcolor{g}{(+0.5)}}\\
 & \gcheck& \gcheck& \xmark& \xmark & 80.3 {\textcolor{g}{(+0.4)}}\\
 & \gcheck& \xmark& \gcheck& \xmark& 80.4 {\textcolor{g}{(+0.5)}}\\
 & \gcheck& \xmark& \xmark& \gcheck& 80.3 {\textcolor{g}{(+0.4)}}\\
 & \xmark&\gcheck & \gcheck& \xmark& 80.3 {\textcolor{g}{(+0.4)}}\\
 & \xmark& \gcheck& \xmark&\gcheck & 80.5 {\textcolor{g}{(+0.6)}}\\
 & \xmark& \xmark& \gcheck&\gcheck & 80.2 {\textcolor{g}{(+0.3)}} \\
 & \gcheck& \gcheck& \gcheck& \xmark& 80.4 {\textcolor{g}{(+0.5)}} \\
 & \gcheck& \gcheck& \xmark&\gcheck & 80.4 {\textcolor{g}{(+0.5)}} \\
 & \gcheck& \xmark &\gcheck &\gcheck & 80.5 {\textcolor{g}{(+0.6)}} \\
  & \xmark & \gcheck &\gcheck &\gcheck & 80.5 {\textcolor{g}{(+0.6)}} \\
 \rowcolor{orange!30}
\ours DeiT-S (Ours) & \gcheck& \gcheck& \gcheck& \gcheck& \textbf{80.7} \textbf{\textcolor{g}{(+0.8)}} \\
\noalign{\hrule height 1.5pt}
\end{tabular}
\end{table}

The results reveal that while certain curve combinations yield more substantial improvements than others, each curve contributes meaningfully to the overall performance. Thus, we retain all four curves in the \ours configuration, leveraging their complementary spatial information.

Additionally, we explore several alternative configurations, as detailed in \cref{tab:extra_curves}. For instance, we evaluate the use of only the four original curves referred as $\mathcal{C}_\text{normal}$ (snake, zig-zag, Hilbert, and Peano) and only their rotated counterparts $\mathcal{C}_\text{transposed}$ (snake$^\top$, zig-zag$^\top$, Hilbert$^\top$, and Peano$^\top$). We also test using only the default Z-curve ordering, which results in a $0.7\%$ accuracy gain.  

\begin{table}[t]
\centering
\caption{\textit{Ablation on different curve configurations:} The performance of the baseline model is compared against \ours with different curve configurations: only original curves ($\mathcal{C}_\text{normal}$), only transposed curves ($\mathcal{C}_\text{transposed}$), only Z-curve, Manhattan distance-based mask and random curves. Changes with respect to the baseline are shown inside $(\cdot)$ next to the accuracies.}
\label{tab:extra_curves}
\begin{tabular}{l>{\columncolor{gray!20}}c>{\columncolor{orange!30}}cccccc}
\noalign{\hrule height 1.5pt}
\multirow{2}{*}{Model}  & \multicolumn{7}{c}{Top-1 Accuracy ($\%$)} \\
& Baseline & \ours & $\mathcal{C}_\text{normal}$ & $\mathcal{C}_\text{transposed}$ & Z-curve & Manhattan & Random\\ 
\hline
DeiT-S & 79.8 & \textbf{80.7 \textcolor{g}{(+0.9)}} & 80.3 \textcolor{g}{(+0.5)} & 80.4 \textcolor{g}{(+0.6)} & 80.5 \textcolor{g}{(+0.7)} & 80.4 \textcolor{g}{(+0.6)} & \xmark \\
\noalign{\hrule height 1.5pt}
\end{tabular}%
\end{table}

Moreover, we define relative distances using a Manhattan mask, inspired by RMT~\citep{fan2024rmt}. Lastly, we experiment with a set of randomized SFCs, where the flattened image is shuffled with a random fixed order across all layers and heads. This model fails to converge to a meaningful accuracy. This further emphasizes the importance of a \textit{structured} SFC as the unstructured curves do not allow model to capture meaningful information from the data.

\subsection{Alternative masking strategies} \label[appendix]{ap:masking}
Another critical design choice is the masking strategy. We compare \ours, which follows the structure $ S(\mathbf{A'} \odot \mathbf{M}) $, where $ S $ denotes the row-wise softmax operation, $ \mathbf{A'} = \alpha \frac{\mathbf{Q}\mathbf{K}^\top}{\sqrt{d}} $, and $ \mathbf{M} = \mathbf{M}_\text{\ours} $ for a cleaner notation. Our findings indicate that the $ S(\mathbf{A'} \odot \mathbf{M}) $ configuration outperforms all other masking alternatives.  

\begin{table}[h]
\centering
\caption{\textit{Ablation on masking strategies:} The performance of the baseline model is compared against \ours with different masking methods: $S(\mathbf{M}+\mathbf{A'})$, $S(\mathbf{A'}) + \mathbf{M}$, $S(\mathbf{A'}) \odot \mathbf{M}$, and $S(\mathbf{A'}\odot(\mathbf{I}+\mathbf{M}))$. Changes with respect to the baseline are shown inside $(\cdot)$ next to the accuracies.}
\label{tab:masks}
\begin{tabular}{l>{\columncolor{gray!20}}c>{\columncolor{orange!30}}ccccc}
\noalign{\hrule height 1.5pt}
\multirow{2}{*}{Model}  & \multicolumn{6}{c}{Top-1 Accuracy ($\%$)} \\
& Baseline & \ours & $S(\mathbf{M}+\mathbf{A'})$ & $S(\mathbf{A'}) + \mathbf{M}$ & $S(\mathbf{A'}) \odot \mathbf{M}$ & $S(\mathbf{A'}\odot(\mathbf{I}+\mathbf{M}))$\\ 
\hline
DeiT-S & 79.8 & \textbf{80.7 \textcolor{g}{(+0.9)}} & 80.1 \textcolor{g}{(+0.3)} & 80.5 \textcolor{g}{(+0.7)} & 80.5 \textcolor{g}{(+0.7)} & 79.1 \textcolor{red}{(-0.7)}\\
\bottomrule
\end{tabular}%
\end{table} 

\subsection{Other design elements} \label[appendix]{ap:other_ablations}
Furthermore, in \cref{tab:design}, we illustrate the impact of additional design choices described in \cref{sec:design}, such as initialization and the scaling parameter \( \alpha \). Additionally, we assess the effect of fixing \( \gamma_c \) at a constant value of 0.9996 instead of learning it. The results indicate that proper initialization and a learnable \( \gamma_c \) are essential for achieving accuracy gains, while the scaling parameter \( \alpha \) primarily contributes to training stability, particularly in larger models. We have also tried using learned per-curve weights instead of averaging, which did not improve the performance. We believe that since the $\gamma_c$ values act as a selection mechanism (see previous discussions), the added learnable weight makes the optimization harder without additional benefits.

\begin{table}[h]
\centering
\caption{\textit{Ablation on other elements of \ours:} The performance of the baseline model is compared against \ours with and without certain design elements: initialization, scaling factor $\alpha$ and learned $\gamma_c$. \gcheck ~indicates it is included in the model, whereas \xmark ~means it is not. Changes with respect to the baseline are shown inside $(\cdot)$ next to the accuracies. }
\label{tab:design}
\begin{tabular}{lcccl}
\toprule
Model  & Initialization & Scaling & Learned $\gamma_c$ & Top-1 Acc ($\%$)\\
\midrule
\rowcolor{gray!20}
DeiT-S (Baseline) & \xmark & \xmark & \xmark  & 79.9 \\
 & \xmark & \gcheck & \gcheck  & 80.0 {\textcolor{g}{(+0.1)}} \\
 & \gcheck& \xmark  & \gcheck  & \textbf{80.7 {\textcolor{g}{(+0.8)}}} \\
 & \gcheck& \gcheck & \xmark   & 80.3 {\textcolor{g}{(+0.4)}} \\
 \rowcolor{orange!30}
\ours DeiT-S (Ours) & \gcheck& \gcheck& \gcheck&  \textbf{80.7 \textcolor{g}{(+0.8)}} \\
\noalign{\hrule height 1.5pt}
\end{tabular}
\end{table}

\subsection{Global Average Pooling (GAP)} \label[appendix]{abl:gap}

Considering the output of the attention mechanism for each token in the last layer, we can write
\begin{align}
\mathbf{y}_i 
= \sum_{j=1}^N 
\frac{\exp\!\left(\mathbf{q}_i^\top \mathbf{k}_j\right)}
{\sum_{j'=1}^N \exp\!\left(\mathbf{q}_i^\top \mathbf{k}_{j'}\right)}
\, \mathbf{v}_j.
\end{align}
When the classification (CLS) token is used, the sequrnce length becomes $N+1$ where the first token is the CLS. When comparing the use of a global average pooling (GAP)~\citep{lin2013network,lu2022bridging} head versus a CLS head with a decay mask, the attention outputs are extracted as follows
\begin{align}
\mathbf{y}_\text{CLS} 
&= \sum_{j=1}^{N+1} 
\frac{\exp\!\left((\mathbf{q}_{CLS}^\top \mathbf{k}_j) \mathbf{M}[CLS,j] \right)}
{\sum_{j'=1}^{N+1} \exp\!\left((\mathbf{q}_{CLS}^\top \mathbf{k}_{j'}) \mathbf{M}[CLS,j']\right)}
\, \mathbf{v}_j, \\[6pt] 
\mathbf{y}_\text{GAP} 
&= \frac{1}{N}\sum_{i=1}^N  \sum_{j=1}^N 
\frac{\exp\!\left((\mathbf{q}_i^\top \mathbf{k}_j)\mathbf{M}[i,j] \right)}
{\sum_{j'=1}^N \exp\!\left((\mathbf{q}_i^\top \mathbf{k}_{j'}) \mathbf{M}[i,j'] \right)}
\, \mathbf{v}_j.
\end{align}
As shown, in the case of the CLS token, the model only requires the attention distribution and relative distances with respect to the CLS token. 
In our setup, this reduces to $\mathbf{M}[CLS,j] = 1$, (or a a learned parameter $\beta_{CLS}$. By contrast, the GAP formulation is more expressive, as it aggregates attention information across all tokens. 
Importantly, the inclusion of the relative distance decay mask $\mathbf{M}[i,j]$ for all tokens makes GAP more effective in constructing the final representation. Therefore, similar to Vision SSMs such as Vision LSTM and Hydra~\citep{alkin2024visionlstm,hydra}, pooling-based outputs align naturally with spatially informed attention. Note that this calculation holds for the last layer only, the remaining layers utilize the mask fully.

\ours attention supports both the use of a classification token and GPA. To assess the role of the classification token versus GAP with the \ours mask, we pretrain all three scales of DeiT and report results in \cref{tab:cls}. While GAP often yields slightly better compatibility with \ours, the improvements cannot be attributed to pooling alone, the gains are additive. 

Most importantly, \ours is \textit{not dependent on GAP}. In DINO pretraining and VTAB-1K fine-tuning, where the cls\_token is retained, \ours still improves performance. This confirms that the benefits arise from the spatial priors introduced by \ours, not from the choice of pooling strategy.

\begin{table}[h]
\centering
\caption{\textit{Ablation on GAP:} The performance of baseline model and \ours is compared when they both have CLS or uses GAP. Baseline$^\dagger$ indicates results taken from~\cite{chu2023conditional}. Changes with respect to the baseline, original model with CLS, are shown inside $(\cdot)$ next to the accuracies. }
\label{tab:cls}
\begin{tabular}{l>{\columncolor{gray!20}}c>{\columncolor{orange!30}}c>{\columncolor{gray!20}}c>{\columncolor{orange!30}}c}
\noalign{\hrule height 1.5pt}
\multirow{3}{*}{Model}  & \multicolumn{4}{c}{Top-1 Accuracy ($\%$)} \\
& \multicolumn{2}{c}{CLS} & \multicolumn{2}{c}{GAP} \\ 
& Baseline & \ours  & Baseline$^\dagger$ & \ours \\ 
\hline
DeiT-T & 72.2  & 72.3 \textcolor{g}{(+0.2)} & 72.6 & \textbf{73.0} \textcolor{g}{(\textbf{+0.8)}}\\
DeiT-S & 79.8 & 80.1 \textcolor{g}{(+0.3)} & 80.2& \textbf{80.7} \textcolor{g}{(\textbf{+0.9})} \\
\noalign{\hrule height 1.5pt}
\end{tabular}
\end{table}

\section{Additional results}\label{ap:additional}
\subsection{Pretraining of larger models} \label[appendix]{ap:base}
As discussed in \cref{bckg:vit}, when both model capacity and training data are sufficiently large, ViTs can implicitly learn spatial biases directly from data. In such scenarios, the relative contribution of \ours is naturally smaller, as seen in the DeiT and DINO base scale pretraining results in \cref{tab:base}, which show only marginal gains. This is expected and lies beyond the primary scope of our work, which focuses on small models and data-scarce settings where inductive biases are most impactful.  

It is important to note that smaller gains at scale do not diminish the relevance of \ours for larger models. In fact, our fine-tuning experiments (\cref{sec:finetune}, \cref{tab:finetune1}) demonstrate that when data is limited, spatial priors provided by \ours substantially improve performance, even for models with hundreds of millions of parameters. This highlights that \ours remains valuable in practice, not by competing with scale, but by enhancing efficiency and adaptability in data-constrained regimes.

\begin{figure*}[h]
    \centering
    \captionof{table}{\textit{Pretraining results of larger models on ImageNet-1K}: 
    Comparison of the top-1 accuracies of \colorbox{gray!20}{baseline models} with their \colorbox{orange!30}{\ours} counterparts. The values in parentheses $(\cdot)$ indicate the accuracy difference compared to the baseline. The best performance between each pair of models is highlighted in \textbf{bold}. For DINO models, both KNN and linear probe evaluations are reported and (300) indicate the number of training epochs. \textbf{(Left)} Supervised, \textbf{(Right)} Self-supervised training.}
    \label{tab:base}
\begin{minipage}[t]{0.44\textwidth}
    \resizebox{\textwidth}{!}{\begin{tabular}{lc>{\columncolor{gray!20}}c>{\columncolor{orange!30}}c}
         \noalign{\hrule height 1.5pt}
        \multirow{2}{*}{Model} & \multirow{2}{*}{$\#$ Param.} & \multicolumn{2}{c}{Top-1 Accuracy ($\%$)}  \\ 
        & & Baseline & \textbf{\ours}  \\ 
        \hline
        DeiT-B  & 86M & 81.8 &\textbf{ 81.9  \textcolor{g}{(+0.1)}}  \\
         \noalign{\hrule height 1.5pt}
    \end{tabular}}
\end{minipage}
\hfill
\begin{minipage}[t]{0.54\textwidth}
        \resizebox{\textwidth}{!}{\begin{tabular}{lcc>{\columncolor{gray!20}}c>{\columncolor{orange!30}}c}
         \noalign{\hrule height 1.5pt}
        \multirow{2}{*}{Model} &  \multirow{2}{*}{} & \multirow{2}{*}{$\#$ Param.} & \multicolumn{2}{c}{Top-1 Accuracy ($\%$)}  \\ 
        & & & Baseline & \textbf{\ours}  \\ 
        \hline
        \multirow{2}{*}{DINO-B (300)}             & KNN    & \multirow{2}{*}{86M}& 76.1 & \textbf{76.1 \textcolor{g}{(-----)}}   \\
                                            & Linear & & 78.2 & \textbf{78.4  \textcolor{g}{(+0.2)}}  \\
         \noalign{\hrule height 1.5pt}
    \end{tabular}}
\end{minipage}
\end{figure*}

\subsection{Fine-tuning of \ours pretrained models} \label[appendix]{ap:finetune}
We fine-tune the \ours DeiT, and DINO pretrained models from \cref{sec:pretrain,ap:base} on the VTAB-1K dataset. The accuracies for each category and the overall average are presented in \cref{tab:finetune1}, alongside the baseline accuracies of the baseline fine-tuned models. We observe that \ours increases the performance across all models and scales compared to original baselines. DeiT,and DINO models achieve impressive improvements of up to $1.92\%$ with up to $2.87\%$ improvement in individual categories. We note that similar to \cref{tab:finetune2} in this setting, Structured group shows the highest accuracy gain. This further shows the broad applicability of \ours, enhancing diverse architectures with close to zero computational overhead.

Notably, we compare \cref{tab:finetune2} and \cref{tab:finetune1}, fine-tuning with an mask learned only during fine-tuning for all models yields better performance in different tasks compared to pretraining with it. We hypothesize that this is because the model starts with generic pretrained representations and gains additional flexibility by learning spatial structure tailored specifically to the downstream task. This is particularly advantageous when the target task differs substantially from the pretraining domain.

\begin{table}[h]
    \centering
    \caption{\textit{Fine-tuning results of pretrained \ours models on VTAB-1K}: 
    Comparison of the top-1 accuracies of \colorbox{gray!20}{baseline models} and their pretrained \colorbox{orange!30}{\ours} counterparts across the VTAB-1K benchmark. 
    The three task groups are abreviated as NAT. = Natural, SPE. = Specialized, and STR. = Structured. The values in parentheses $(\cdot)$ indicate the accuracy difference compared to the baseline. The best performance within each model pair is highlighted in \textbf{bold}.
    }
    \label{tab:finetune1}
    \resizebox{0.95\textwidth}{!}{
    \begin{tabular}{lc>{\columncolor{gray!20}}c>{\columncolor{gray!20}}c>{\columncolor{gray!20}}c>{\columncolor{gray!20}}c>{\columncolor{orange!30}}c>{\columncolor{orange!30}}c>{\columncolor{orange!30}}c>{\columncolor{orange!30}}c}
         \noalign{\hrule height 1.5pt}
        \multirow{3}{*}{Model} & \multirow{3}{*}{Param.} & \multicolumn{8}{c}{Top-1 Accuracy ($\%$)}  \\ 
        & & \multicolumn{4}{{>{\columncolor{gray!20}}c}}{Baseline} & \multicolumn{4}{{>{\columncolor{orange!30}}c}}{\textbf{\ours}} \\ 
        & & NAT. & SPE. & STR. & Avg.& NAT. & SPE. & STR. & Avg.\\
        \hline
        DeiT-T & 5M  & 69.56 & 82.34 & 53.57 & 65.52 & \textbf{70.71} \textbf{\textcolor{g}{(+1.15)}} & \textbf{82.64} \textbf{\textcolor{g}{(+0.30)}} & \textbf{54.52} \textbf{\textcolor{g}{(+0.95)}} & \textbf{66.41} \textbf{\textcolor{g}{(+0.89)}} \\
        DeiT-S & 22M & 73.64 & 84.30 & 53.44 & 67.38 & \textbf{75.24} \textbf{\textcolor{g}{(+1.60)}} & \textbf{84.87} \textbf{\textcolor{g}{(+0.57)}} & \textbf{56.31} \textbf{\textcolor{g}{(+2.87)}} & \textbf{69.30} \textbf{\textcolor{g}{(+1.92)}} \\
        DeiT-B & 86M & \textbf{76.93} & \textbf{85.52} & 57.00 & 70.35 & {76.54} {\textcolor{red}{(-0.39)}} & {85.44} {\textcolor{red}{(-0.08)}} & \textbf{58.90} \textbf{\textcolor{g}{(+1.90)}} & \textbf{70.99} \textbf{\textcolor{g}{(+0.64)}} \\
        \hline
        DINO-S & 22M & 75.35 & 85.09 & \textbf{60.65} & 71.21 & \textbf{76.29} \textbf{\textcolor{g}{(+0.94)}} & \textbf{85.75} \textbf{\textcolor{g}{(+0.66)}} & {60.61} {\textcolor{red}{(-0.04)}} & \textbf{71.68} \textbf{\textcolor{g}{(+0.47)}} \\
        DINO-B & 86M & 77.50 & 85.77 & 58.47 & 71.23 & \textbf{77.82} \textbf{\textcolor{g}{(+0.32)}} & \textbf{85.83} \textbf{\textcolor{g}{(+0.06)}} & \textbf{58.77} \textbf{\textcolor{g}{(+0.30)}} & \textbf{71.49} \textbf{\textcolor{g}{(+0.26)}} \\

         \noalign{\hrule height 1.5pt}
    \end{tabular}}
\end{table}

\subsection{Multi-resolution classification} \label[appendix]{abl:multi_res}
Following \citet{heo2024ropevit}, we test the resolution scalability of \ours models. We present the top-1 accuracies for DeiT-S, and DeiT-B models across input resolutions ranging from 144 to 512 in \cref{fig:resolution}. We use bicubic interpolation for all positional embeddings~\citep{heo2024ropevit}. In the top plot, we observe that although \ours without positional embeddings performs slightly worse than the baseline at the training resolution (224), it begins to outperform the baseline at higher resolutions. In the second and third plots, where \ours is combined with positional embeddings, for most resolutions, \ours preserves or expands the performance gap compared to baselines. These results suggest that the decay mask used in \ours generalizes effectively to higher resolutions, making it a resolution-robust enhancement for ViTs.

Another interesting application of context extrapolation is video understanding. Following~\citet{dino}, we generate a segmentation video using \ours DINO-B model. While the training resolution is 224, for video, \ours extends to $768\times432$ resolution. Some frames are provided in \cref{fig:dindinvid} and the full video can be found in our GitHub repository.
\begin{figure}[h] 
    \centering
    \includegraphics[width=0.5\linewidth]{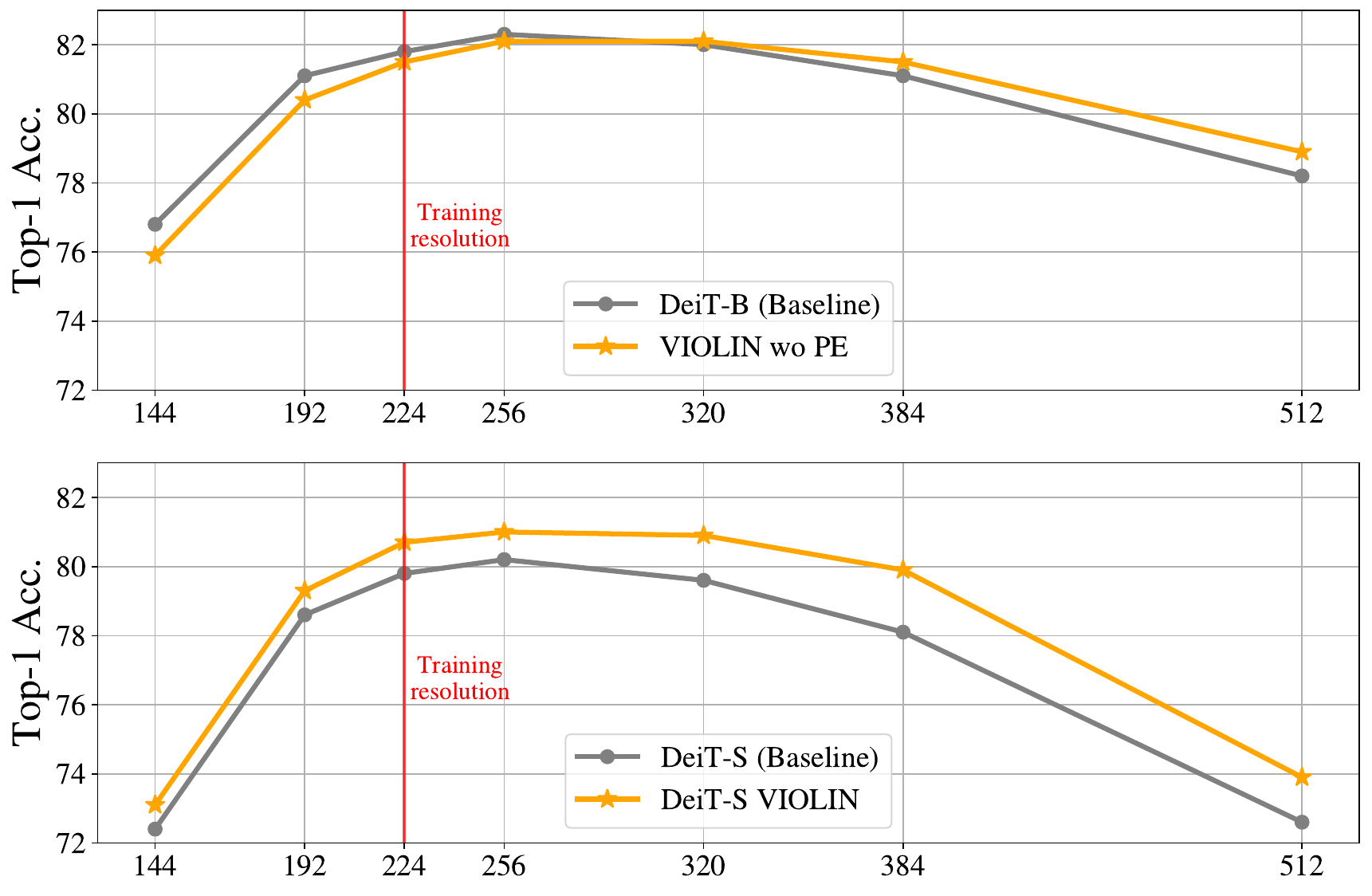}
    \caption{\textit{Resolution expansion:} Top-1 accuracies of DeiT-B (top), DeiT-S (middle) and DeiT-III-S (bottom) models and their \ours counterparts at different resolutions on ImageNet. Training resolution of 224 is highlighted in red.}
    \label{fig:resolution}
\end{figure}
\begin{figure}[h] 
    \centering
    \includegraphics[width=0.9\linewidth]{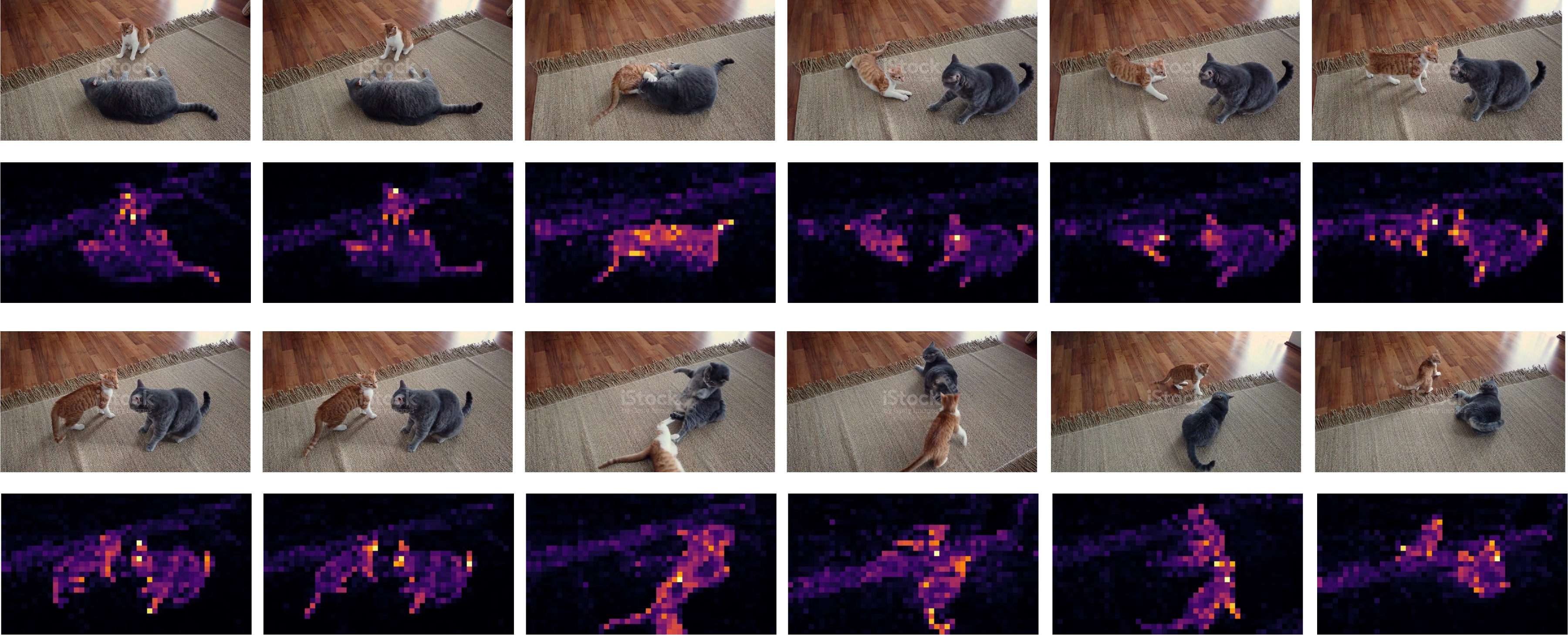}
    \caption{\textit{Video undertanding}: Frame by frame video understanding of \ours-DINO in base scale. The full video and generation codes are also included in the github repository of \ours.}
    \label{fig:dindinvid}
\end{figure}
\subsection{Additional visualizations} \label[appendix]{ap:visuals}
In \cref{fig:flat_golge}, we present the 1D flattened sequences of the patched image \textbf{(a)}, corresponding to the curves illustrated in \cref{fig:golge_curves}. In~\cref{fig:gamma_values} we visualize the mask pattern for a middle pixel under the snake curve for different values of $\gamma$. As expected, when $\gamma \approx 1$, the head attends broadly across the entire image, whereas smaller $\gamma$ values produce a much more localized receptive field, emphasizing spatial neighbors. \cref{fig:attentionabs} compares attention heatmaps of DeiT and \ours models, fine-tuned on Structured group datasets. \cref{fig:violinattvis} visualizes the attention heatmaps of the \ours DeiT-B model using various images. We adopt the average diagonal visualization strategy as proposed in~\citep{liu2024vmamba}. 
\begin{figure}[h] 
    \centering
    \includegraphics[width=0.9\linewidth]{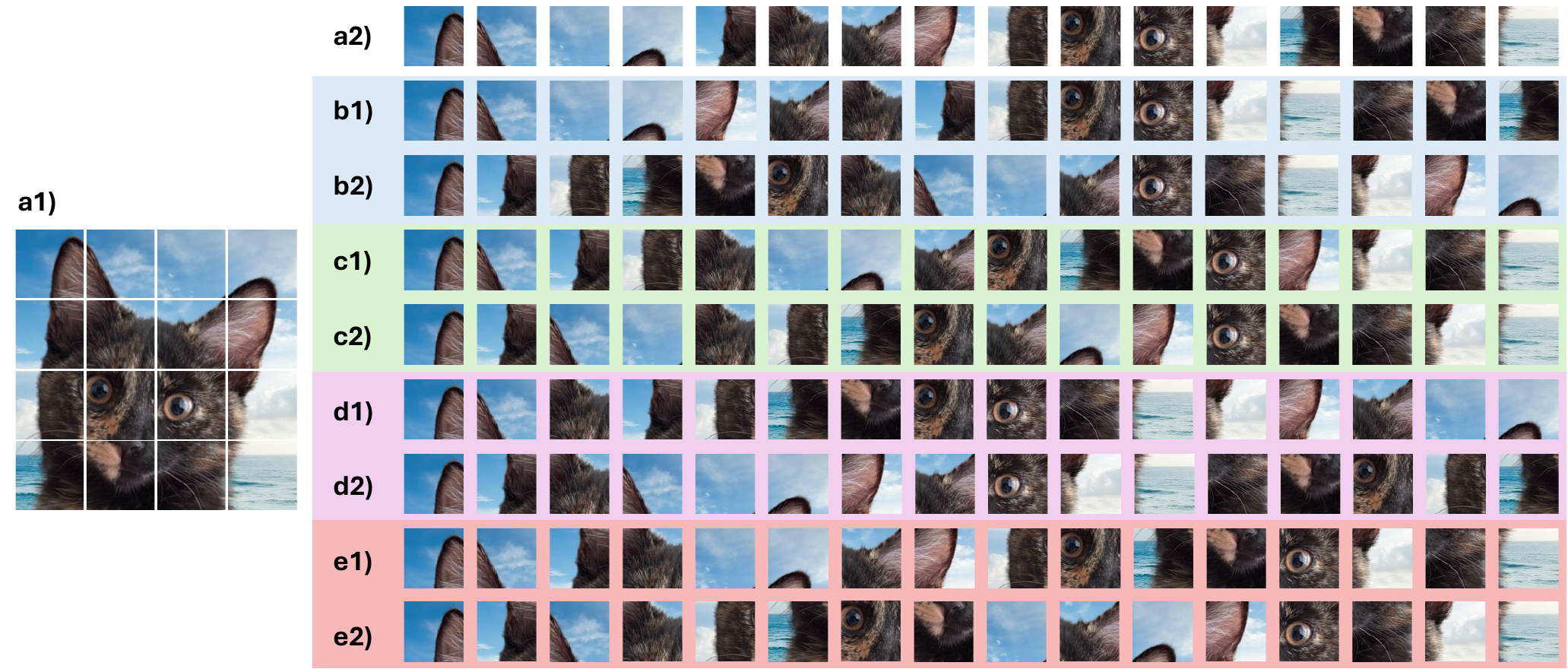}
    \caption{\textit{Flattened Space Filling Curve paths:} Examples of flattened images with different traversal paths followed in \ours. \textbf{(a1)} Original patchedimage. \textbf{(a2)} Z-curve 
        \textbf{(b1)} Snake curve, 
    \textbf{(b2)} Transposed Snake curve, 
        \textbf{(c1)} Zig-zag curve, 
    \textbf{(c2)} Transposed Zig-zag curve, 
        \textbf{(d1)} Hilbert curve, 
    \textbf{(d2)} Transposed Hilbert curve, 
        \textbf{(e1)} Peano curve, 
    \textbf{(e2)} Transposed Peano curve.}
    \label{fig:flat_golge}
\end{figure}
\begin{figure}[h] 
    \centering
    \includegraphics[width=0.9\linewidth]{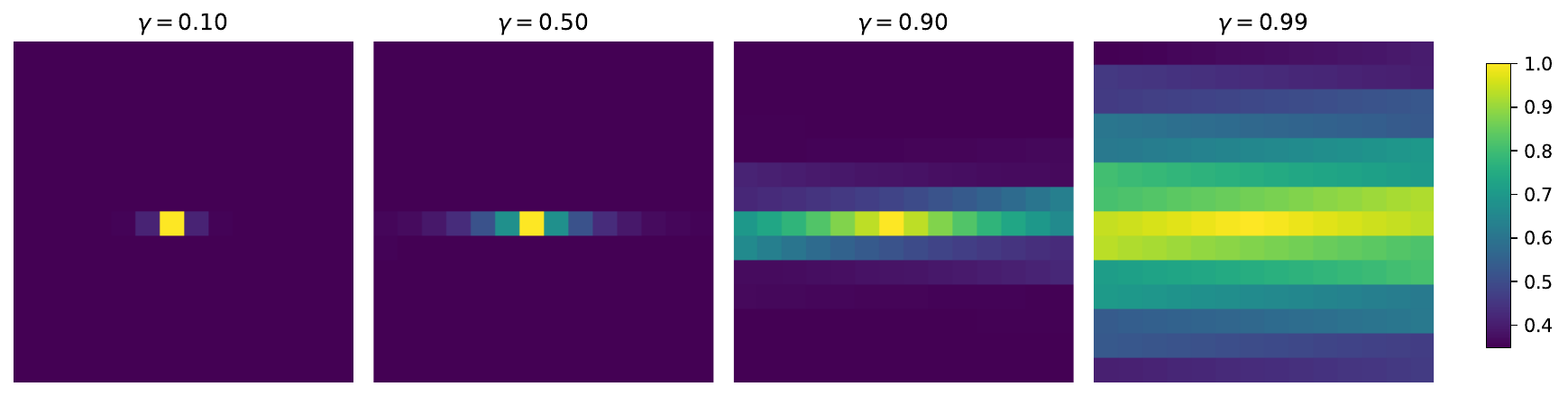}
    \caption{\textit{Effect of $\gamma$ on the decay mask:} Visualization of the decay mask for a central pixel under the Snake curve for different values of $\gamma$. Larger $\gamma$ values yield more global attention, while smaller $\gamma$ restrict the effective receptive field to local regions.}
    \label{fig:gamma_values}
\end{figure}
\begin{figure}[h] 
    \centering
    \includegraphics[width=1\linewidth]{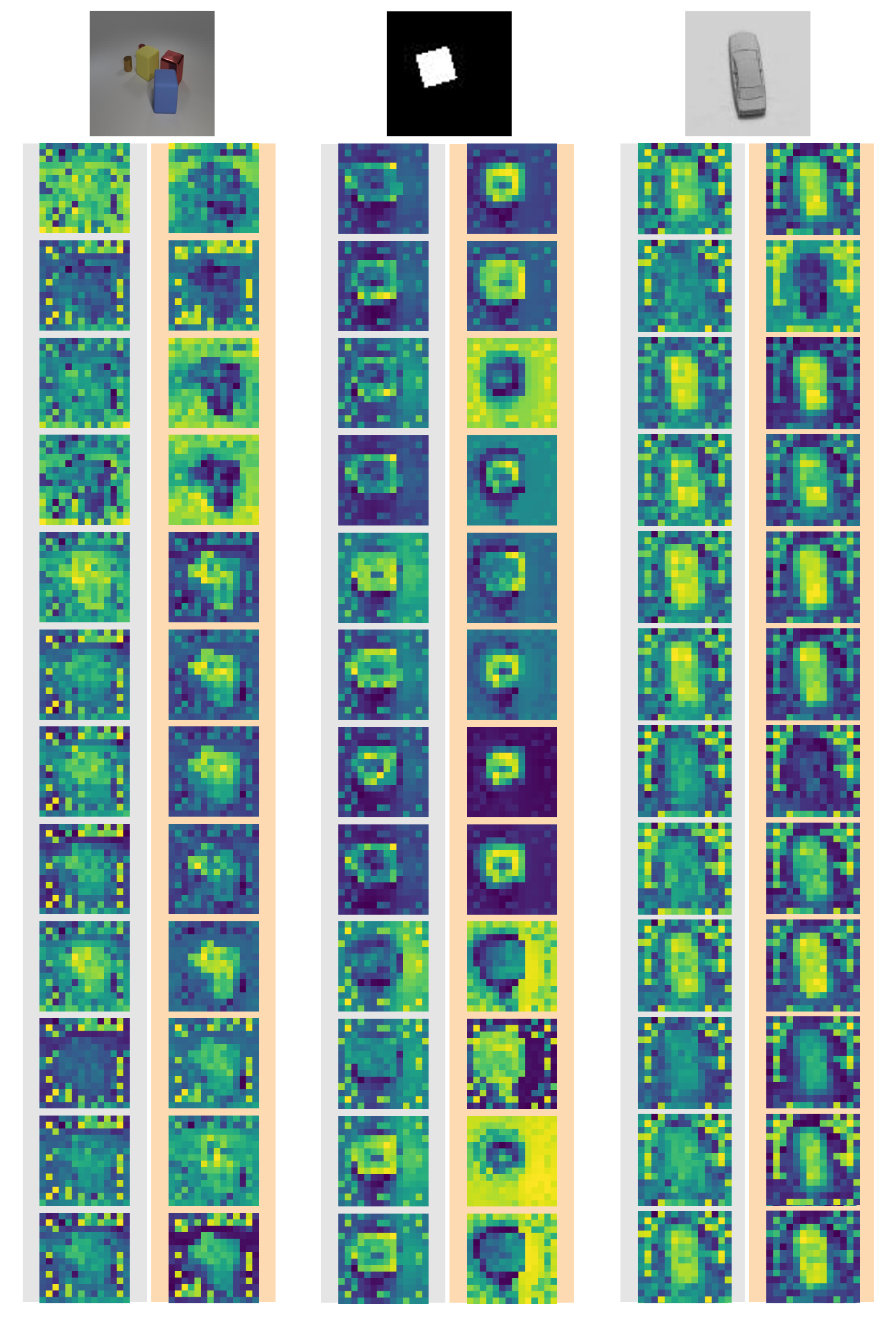}
    \caption{\textit{Attention heatmaps on Structured tasks:} Examples are taken from three datasets in the Structured group: CLEVR-Count, dSprites-Location, and SmallNORB-Azimuth. We compare attention scores of \colorbox{gray!20}{DeiT-B} (left) and \colorbox{orange!30}{\ours} (right), fine-tuned on the corresponding dataset. Visualizations are from layer 12, with rows showing heads 1–12. Since both models share the same pretrained initialization, attention heads are initially identical before fine-tuning. After fine-tuning, \ours produces more accurate and focused heads, with better object coverage and more uniform color outside the objects, indicating reduced attention to irrelevant regions.}
    \label{fig:attentionabs}
\end{figure}

\begin{figure}[h] 
    \centering
    \includegraphics[width=0.9\linewidth]{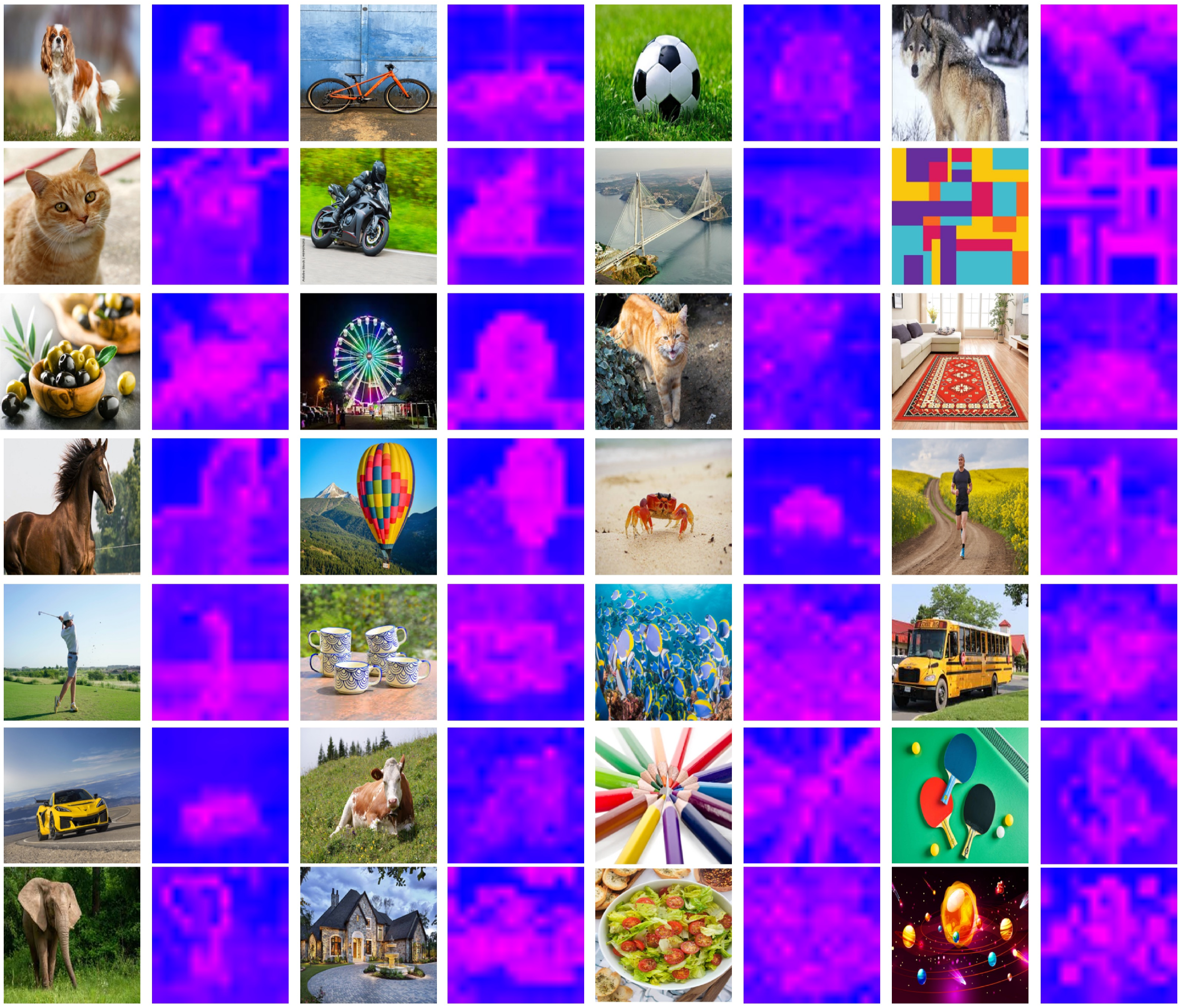}
    \caption{\textit{Attention heatmap visualization of \ours DeiT-B}: The average diagonal of the masked attention is visualized followed by~\citep{liu2024vmamba}.}
    \label{fig:violinattvis}
\end{figure}

\clearpage
\subsection{Details and individual results on VTAB-1K dataset}\label[appendix]{ap:vtab}
VTAB~\citep{vtab} contains 19 tasks which cover a broad spectrum of domains and semantics that are grouped into three sets: NATURAL, SPECIALIZED,
and STRUCTURED.

The NATURAL group represents  natural images and classical vision problems. The group includes Caltech101, CIFAR-100, DTD, Flowers102, Pets, Sun397, and SVHN datasets.

The SPECIALIZED group also contains images of the world, but they are captured through specialist equipment. These images have different invariances to those in the NATURAL tasks. It includes Resisc45 and EuroSAT, Patch Camelyon, and Diabetic Retinopathy datasets.

The STRUCTURED group assesses comprehension of the structure of a scene, for example, object counting, or 3D depth prediction. Most of the tasks are generated from simulated environments, whose structure is easy for a human, but their domain differs greatly to datasets like ImageNet. It includes Clevr count and distance, dSprites location and orientation, SmallNORB, DMLab, and KITTI. In \cref{tab:vtab_natural,tab:vtab_specialized,tab:vtab_structured}, we present the accuracy scores of each model on all VTAB-1K datasets. 
\begin{table}[h]
\centering
\caption{\textit{VTAB Results-Natural Subset:} Individual scores for each dataset.}
\label{tab:vtab_natural}
\begin{tabular}{cl ccccccc}
\noalign{\hrule height 1.5pt}
& \textbf{Model} & CIFAR & Caltech101 & DTD & Flowers102 & Pets  & SVHN & Sun397\\
\midrule
\parbox[t]{3mm}{\multirow{21}{*}{\rotatebox[origin=c]{90}{\textbf{Natural}}}} 
& DeiT-T & 48.36 &	86.9 &	63.97 &	86.43 &	87.14 &	78.28 &	35.87   \\
& \ours DeiT-T & 51.21 &	86.48 &	64.75 &	87.24 &	86.77 &	83.16 &	35.38  \\
& DeiT-T $\odot \mathbf{M}_\text{\ours}$& 51.17 &	87.8 &	65.43 &	89.17 &	86.75 &	85.78 &	37.17  \\
\cline{2-9}
& DeiT-S & 57.38 &	89.06 &	68.83 &	91.09 &	91.13 &	75.82 &	42.19  \\
& \ours DeiT-S & 60.71 &	88.06 &	68.33 &	91.12 &	91.19 &	85.38 &	41.93  \\
& DeiT-S $\odot \mathbf{M}_\text{\ours}$& 59.6 &	89.78 &	69.08 &	92.5 &	91.89 &	86.15 &	43.45  \\
\cline{2-9}
& DeiT-B & 61.38 &	90.33 &	69.06 &	93.73 &	92.43 &	85.95 &	45.59  \\
& \ours DeiT-B & 63.32 &	89.55 &	68.37 &	92.1 &	92.04 &	86.22 &	44.15  \\
& DeiT-B $\odot \mathbf{M}_\text{\ours}$& 61.99 &	91.07 &	70.14 &	93.97 &	92.75 &	90.22 &	45.56  \\
& DeiT-B LoRA & 62.37 &	90.07 &	69.27 &	93.26 &	92.3 &	90.58 &	44.35  \\
& DeiT-B $\odot \mathbf{M}_\text{\ours}$ LoRA &  65.36 &	90.92 &	70.62 &	93.57 &	92.37 &	91.86 &	45.19\\
& DeiT-B DoRA &  63.81 &	90.78 &	69.29 &	91.79 &	89.95 &	88.75 &	44.12 \\
& DeiT-B $\odot \mathbf{M}_\text{\ours}$ DoRA & 66.38 &	90.97 &	69.82 &	92.77 &	91.71 &	90.26 &	44.64 \\
\cline{2-9}
\cline{2-9}
& DeiT-III-S & 59.08 &	88.53 &	67.09 &	91.13 &	91.85 &	84.65 &	43.57  \\
& DeiT-III-S $\odot \mathbf{M}_\text{\ours}$& 62.18 &	88.78 &	69.4 &	93.92 &	91.35 &	89.98 &	43.6  \\
\cline{2-9}
& DeiT-III-B & 64.39 &	89.56 &	70.8 &	94.63 &	93.38 &	87.28 & 47.28  \\
& DeiT-III-B $\odot \mathbf{M}_\text{\ours}$& 66.77 &	89.97 &	71.38 &	95.53 &	93.61 &	91.24 & 46.19 \\
\cline{2-9}
& DeiT-III-L &  65.16 &	87.89 &	71.58 &	94.39 &	93.23 &	71.17 &	48.65 \\
& DeiT-III-L $\odot \mathbf{M}_\text{\ours}$& 66.74 &	87.67 &	72.34 &	95.01 &	93.28 &	78.7 &	48.58\\
\cline{2-9}
& DeiT-III-H &  64.34 &	88.2 &	71.22 &	94.95 &	92.96 &	68.76 &	48.46 \\
& DeiT-III-H $\odot \mathbf{M}_\text{\ours}$& 65.16 &	88.18 &	71.35 &	95.18 &	93.33 &	72.72 &	48.7\\
\cline{2-9}
\cline{2-9}
& DINO-S & 54.32 &	93.95 &	68.12 &	91.28 &	88.62 &	90.24 &	40.93  \\
& \ours DINO-S & 56.05 &	91.95 &	69.33 &	95.26 &	89.62 &	91.65 &	40.2  \\
& DINO-S $\odot \mathbf{M}_\text{\ours}$& 57.38 &	90.92 &	68.88 &	95.18 &	89.44 &	90.61 &	41.45  \\
\cline{2-9}
& DINO-B & 58.57 &	93.7 &	70.64 &	95.84 &	90.21 &	89.69 &	43.86  \\
& \ours DINO-B & 59.96 &	92.13 &	71.84 &	95.69 &	90.49 &	90.78 &	43.83  \\
& DINO-B $\odot \mathbf{M}_\text{\ours}$& 62.21 &	93.32 &	71.58 &	96.1 &	90.74 &	91.74 &	44.87  \\
\noalign{\hrule height 1.5pt}
\end{tabular}
\end{table}

\begin{table}[h]
\centering
\caption{\textit{VTAB Results-Structured Subset:} Individual scores for each dataset. SN refers to SmallNorm, and dS represents dSprites.}
\label{tab:vtab_structured}
\resizebox{\textwidth}{!}{
\begin{tabular}{cl cccccccc}
\noalign{\hrule height 1.5pt}
& Model & CLEVR Count & CLEVR Dist & DMLab & KITTI & dS Loc & dS Ori & SN Azi & SN Ere \\
\midrule
\parbox[t]{3mm}{\multirow{21}{*}{\rotatebox[origin=c]{90}{\textbf{Structured}}}} 
& DeiT-T &  71.37 &	60.37 &	44.26 &	78.81 &	69.04 &	41.86 &	30.28 &	32.57 \\
& \ours DeiT-T & 72.73 &	61.7 &	47.98 &	79.7 &	68.7 &	46.11 &	25.31 &	33.96  \\
& DeiT-T $\odot \mathbf{M}_\text{\ours}$& 74.41 &	59.84 &	46.37 &	80.78 &	78.32 &	50.91 &	31.33 &	38.05  \\
\cline{2-10}
& DeiT-S & 75.08 &	58.15 &	45.74 &	78.43 &	63.3 &	48.13 &	26.24 &	32.48  \\
& \ours DeiT-S & 78.26 &	59.25 &	49.91 &	81.29 &	64.63 &	53.16 &	27.37 &	36.59  \\
& DeiT-S $\odot \mathbf{M}_\text{\ours}$& 78.87 &	59.2 &	50.59 &	80.4 &	73.52 &	53.44 &	32.48 &	37.62   \\
\cline{2-10}
& DeiT-B & 79.01 &	60.1 &	47.03 &	82.61 &	66.7 &	53.38 &	30.87 &	36.32  \\
& \ours DeiT-B & 82.6 &	61.72 &	52.84 &	80.97 &	68.44 &	55.47 &	31.72 &	37.45  \\
& DeiT-B $\odot \mathbf{M}_\text{\ours}$& 81.33 &	61.31 &	53.93 &	83.22 &	81.72 &	57.28 &	35.37 &	40.98  \\
& DeiT-B LoRA & 79.1 &	60.15 &	51.93 &	81.25 &	78.53 &	53.71 &	28.28 &	32.12 \\
& DeiT-B $\odot \mathbf{M}_\text{\ours}$ LoRA & 82.36 &	63.46 &	52.86 &	82.18 &	78.52 &	55.25 &	32.21 &	39.79 \\
& DeiT-B DoRA & 76.97 &	60.62 &	50.37 &	81.34 &	73.34 &	54.11 &	28.69 &	39.43 \\
& DeiT-B $\odot \mathbf{M}_\text{\ours}$ DoRA & 81.64 &	63.29 &	51.06 &	82.42 &	78.65 &	56.14 &	27.62 &	38.89 \\
\cline{2-10}
\cline{2-10}
& DeiT-III-S & 76.53 &	57.29 &	46.23 &	81.81 &	58.12 &	50.48 &	26.33 &	26.57  \\
& DeiT-III-S $\odot \mathbf{M}_\text{\ours}$& 77.78 &	61.9 &	54.84 &	83.17 &	85.91 &	59.78 &	33.45 &	36.07  \\
\cline{2-10}
& DeiT-III-B & 80.54 &	61.82 &	50.95 &	82.7 &	60.75 &	55.35 &	30.36 &	31.18  \\
& DeiT-III-B $\odot \mathbf{M}_\text{\ours}$& 84.51 &	61.92 &	55.64 &	82.79 &	84.06 &	60.34 &	36.59 &	38.4  \\
\cline{2-10}
& DeiT-III-L &  72.99 &	53.23 &	47.59 &	80.78 &	50.19 &	50.72 &	25.21 &	30.51\\
& DeiT-III-L $\odot \mathbf{M}_\text{\ours}$& 76.66 &	55.64 &	50.03 &	81.86 &	55.42 &	57.35 &	28.69 &	33.91\\
\cline{2-10}
& DeiT-III-H &  75.17 &	55.24 &	48.66 &	81.11 &	41.57 &	46.99 &	25.15 &	31.74 \\
& DeiT-III-H $\odot \mathbf{M}_\text{\ours}$& 77.89 &	55.96 &	50.96 &	81.9 &	47.85 &	55.07 &	26.57 &	33\\
\cline{2-10}
\cline{2-10}
& DINO-S & 83.29 &	65.03 &	53.44 &	80.03 &	78.72 &	48.61 &	34.23 &	41.87  \\
& \ours DINO-S & 84.19 &	63.35 &	55.72 &	81.43 &	75.82 &	49.37 &	32.92 &	42.06   \\
& DINO-S $\odot \mathbf{M}_\text{\ours}$& 83.69 &	64.23 &	55.35 &	79.98 &	79.42 &	49.18 &	36.43 &	41.61  \\
\cline{2-10}
& DINO-B &    80.93 &	62.76 &	52.17 &	79.23 &	69.22 &	48.39 &	33.73 &	41.34 \\
& \ours DINO-B & 81.96 &	63.04 &	53.45 &	79 &	72.12 &	49.59 &	30.29 &	40.76 \\
& DINO-B $\odot \mathbf{M}_\text{\ours}$& 83.87 &	63.65 &	55.66 &	81.2 &	74.14 &	54.18 &	34.79 &	39.27  \\
\noalign{\hrule height 1.5pt}
\end{tabular} }
\end{table}

\begin{table}[h]
\centering
\caption{\textit{VTAB Results-Specialized Subset:} Individual scores for each dataset.}
\label{tab:vtab_specialized}
\begin{tabular}{cl cccc}
\noalign{\hrule height 1.5pt}
& Model &  Patch Camelyon & EuroSAT & Resisc45 & Diabetic Retinopathy \\
\midrule
\parbox[t]{3mm}{\multirow{21}{*}{\rotatebox[origin=c]{90}{\textbf{Specialized}}}} 
& DeiT-T & 82.79 &	93.53 &	80.98 &	72.05 \\
& \ours DeiT-T & 82.47 &	93.35 &	81.3 &	73.43 \\
& DeiT-T $\odot \mathbf{M}_\text{\ours}$& 84.04 &	93.88 &	83.23 &	73.87 \\
\cline{2-6}
& DeiT-S & 84.08 &	94.4 &	84.01 &	74.72 \\
& \ours DeiT-S &  85.36 &	95.41 &	83.86 &	74.85\\
& DeiT-S $\odot \mathbf{M}_\text{\ours}$& 85.19 &	95.02 &	85.68 &	74.32 \\
\cline{2-6}
& DeiT-B &  85.74 &	95.38 &	86.37 &	74.6 \\
& \ours DeiT-B & 85.62 &	95.44 &	85.68 & 75.02 \\
& DeiT-B $\odot \mathbf{M}_\text{\ours}$& 86.74 &	95.91 &	87.31 &	75.2 \\
& DeiT-B LoRA &  86.2 &	95.46 &	85.72 &	75.09 \\
& DeiT-B $\odot \mathbf{M}_\text{\ours}$ LoRA & 85.9 &	95.66 &	86.71 &	73.73 \\
& DeiT-B DoRA &  85.53 &	95.39 &	85.21 &	74.8 \\
& DeiT-B $\odot \mathbf{M}_\text{\ours}$ DoRA & 85.92 &	95.56 &	84.98 &	73.35 \\
\cline{2-6}
\cline{2-6}
& DeiT-III-S &  84.57 &	93.33 &	82.68 &	73.94\\
& DeiT-III-S $\odot \mathbf{M}_\text{\ours}$& 85.76 &	94.98 &	86.43 &	74.67 \\
\cline{2-6}
& DeiT-III-B &  86.4 &	94.47 &	85.83 &	74.33 \\
& DeiT-III-B $\odot \mathbf{M}_\text{\ours}$&  87.77 &	95.8 &	87.57 &	74.73\\
\cline{2-6}
& DeiT-III-L & 84.5 &	93.28 &	84.47 &	75.28 \\
& DeiT-III-L $\odot \mathbf{M}_\text{\ours}$& 84.54 &	94.11 &	85.24 &	74.83 \\
\cline{2-6}
& DeiT-III-H &  84.64 &	92.64 &	84.99 &	74.46 \\
& DeiT-III-H $\odot \mathbf{M}_\text{\ours}$& 84.81 &	93.3 &	84.66 &	74.93 \\
\cline{2-6}
\cline{2-6}
& DINO-S &  86.82 &	94.29 &	86.13 &	73.14 \\
& \ours DINO-S &  87.7 &	94.76 &	86.59 &	73.96 \\
& DINO-S $\odot \mathbf{M}_\text{\ours}$& 85.94 &	94.9 &	86.17 &	74.26  \\
\cline{2-6}
& DINO-B & 87.02 &	94.45 &	87.05 &	74.55  \\
& \ours DINO-B &  87.57 &	94.46 &	87.25 &	74.03  \\
& DINO-B $\odot \mathbf{M}_\text{\ours}$& 87.81 &	95.44 &	87.96 &	74.54 \\
\noalign{\hrule height 1.5pt}
\end{tabular}
\end{table}

\clearpage
\subsection{Comparison against other locality methods.}\label[appendix]{ap:locality}

There are many methods for enhancing locality in plain ViTs. To compare these approaches with \ours, we start from the same pretrained DeiT-B model, add each locality mechanism on top of it, and fine-tune all models under the exact same protocol. This ensures that every method begins from an identical initialization. The results show that while all methods offer some improvement, \ours achieves the strongest gains. Below, we detail how each method is incorporated and initialized to preserve the pretrained model at the start of fine-tuning, and we report results in \cref{tab:vtab_natural_local,tab:vtab_structured_local,tab:vtab_specialized_local}. 

\paragraph{Swin RPB} Swin transformers~\citep{liu2021swin} introduces locality two ways, by partitioning the feature map into shifted windows, and with relative position biases (RPB) that encode spatial offsets inside each window. These biases give the attention mechanism information about relative spatial relationships within a window, improving performance on vision tasks where nearby pixels are correlated. To incorporate RPB into a pretrained global-attention ViT, we add a learnable bias term $\mathbf{B}\in\R^{N\times N}$ as in \cref{eq:swin} where $\mathbf{B}[i,j]$ depends on the relative position of the tokens $i$ and $j$.
\begin{align} \label{eq:swin}
        \mathbf{Y} = \text{Softmax}\left( \frac{\mathbf{QK}^\top}{\sqrt{d}} + \mathbf{B} \right) \mathbf{V}.
    \end{align}
By initializing $\mathbf{B}$ with zeros, the modified attention reduces exactly to the original attention. This guarantees that adding the Swin-RPB  does not alter the model’s capabilities and new positional biases can be learned during fine-tuning.

\paragraph{2D Relative Position Encoding (iRPE)} iRPE~\citep{wu2021rethinking} add locality into attention, by adding learnable bias terms based on the 2-D relative position of tokens.  For any pair of tokens \( (i,j) \), the offset \( \Delta p_{ij} \) is mapped through a bucketing function to an index \( b_{ij} \), which selects a bias embedding from a table \( R \in \mathbb{R}^{B \times H} \). Depending on the chosen attachment mode, this embedding is added to queries, keys or values (e.g., \( \hat{k}_j = k_j + R_{b_{\cdot j}} \)) and the attention scores are calculated using this new parameters. To integrate iRPE into a pretrained ViT without disturbing its learned representations, we initialize all bucket embeddings
to zero,\[R_{b} = 0 \;\; \forall\, b\] so that the queries/keys/values are not changed at the start of finetuning. This ensures that the model initially behaves exactly like the pretrained backbone, while the RPE
parameters gradually learn non-zero spatial biases during training.  

\paragraph{LocalVit} LocalVit~\citep{li2021localvit} enhances locality inside the feed-forward network (FFN) rather than attention. It replaces the MLP with a depthwise-convolutional residual branch. This allows each token to mix information with its spatial neighbors, giving the transformer an inductive bias similar to CNNs while preserving the global interactions of self-attention. For LocalViT, we gate the convolutional branch with a learnable scalar initialized to zero, and initialize the depthwise conv as an identity kernel (center=1, others=0). This allows the modified architecture to behave exactly the same as the pretrained model at the first step, enabling smooth fine-tuning and gradual learning of locality information.

\paragraph{\ours variations} Additionally, we evaluate several ablations discussed in previous sections, including an additive version of $\mathbf{M}{\ours}$, Manhattan-distance masking, a single-curve variant ($\mathbf{M}{\text{Peano}}$), and random-curve masking ($\mathbf{M}_{\text{Random}}$), under the same finetuning protocol for completeness. These results further highlight the contributions of using multiple SFCs rather than relying on any single locality pattern.

\begin{table}[h]
\centering
\caption{\textit{VTAB Results-Natural Subset:} Individual scores for each dataset for different locality-enforcing methods.}
\label{tab:vtab_natural_local}
{
{%
\begin{tabular}{l ccccccc}
\noalign{\hrule height 1.5pt}
 Model & CIFAR & Caltech101 & DTD & Flowers102 & Pets  & SVHN & Sun397\\
\midrule
Additive $\mathbf{M}_{\ours}$ & 63.64 & 91.11 &	69.27 &	93.6 &	92.6 &	90.46 &	44.28  \\
Swin RPB& 63.72 & 90.75 &	70.16 &	94.15 &	92.66 &	90.21 &	45.82  \\
i-RPE-QKV & 65.03 &	90.94 &	70.12 &	93.97 &	92.63 &	90.32 &	45.66  \\
LocalVit & 65.17 &	91.13 &	69.57 &	93.85 &	92.56 &	90.26 &	45.63 \\
Manhattan & 59.62 &	90.78 &	68.03 &	92.07 &	91.47 &	89.81 &	42.13 \\
 $\mathbf{M}_{\text{Peano}}$ & 65.04 &	90.78 &	69.18 &	94.11 &	92.61 &	90.14 &	45.89 \\
 $\mathbf{M}_{\text{Random}}$ & 65.02 &	90.78 &	69.02 &	94.09 &	92.6 &	89.74 &	45.91 \\
\noalign{\hrule height 1.5pt}
\end{tabular}}}
\end{table}
\begin{table}[h]
\centering
\caption{\textit{VTAB Results-Structured Subset:} Individual scores for each dataset for different locality-enforcing methods. SN refers to SmallNorm, and dS represents dSprites.}
\label{tab:vtab_structured_local}
{%
\begin{tabular}{l cccccccc}
\noalign{\hrule height 1.5pt}
 Model & CLEVR Count & CLEVR Dist & DMLab & KITTI & dS Loc & dS Ori & SN Azi & SN Ere \\
\midrule
Additive $\mathbf{M}_{\ours}$ & 81.08 & 62.12 &	51.95 &	83.26 &	80.95 &	57.25 &	34.76 &	39.38 \\
Swin RPB & 81.42 & 61.67 &	53.83 &	83.17 &	81.39 &	56.81 &	35.42 &	38.9 \\
i-RPE-QKV & 81.25 & 61.58 &	53.42 &	83.12 &	81.49 &	57.28 &	35.13 &	38.34 \\
LocalVit & 81.28 &	61.53 &	53.43 &	82.56 &	81.38 &	57.6 &	35.5 &	38.71 \\
Manhattan &	76.74 &	60.73 &	50.16 &	82.51 &	74.69 &	55.03 &	32.49 &	34.57  \\
$\mathbf{M}_{\text{Peano}}$ & 81.45 &	61.4 &	53.59 &	83.17 &	81.09 &	56.98 &	34.53 &	40.84 \\
$\mathbf{M}_{\text{Random}}$ & 81.45 &	61.33 &	53.36 &	82.84 &	80.21 &	56.98 &	34.5 &	40.76 \\
\noalign{\hrule height 1.5pt}
\end{tabular} }%
\end{table}
\begin{table}[h!]
\centering
\caption{\textit{VTAB Results-Specialized Subset:} Individual scores for each dataset for different locality-enforcing methods.}
\label{tab:vtab_specialized_local}
{%
\begin{tabular}{l cccc}
\noalign{\hrule height 1.5pt}
Model &  Patch Camelyon & EuroSAT & Resisc45 & Diabetic Retinopathy \\
\midrule
Additive $\mathbf{M}_{\ours}$ & 86.84 &96.07 &	87.62 &	74.93  \\
Swin RPB & 86.17 & 95.66 &	87.47 &	75.37  \\
i-RPE-QKV & 86.76 & 95.72 &	87.51 &	74.91  \\
LocalVit & 86.55 &	95.85 &	87.58 &	75.4 \\
Manhattan &	86.44 &	94.93 &	86.29 &	74.21 \\
$\mathbf{M}_{\text{Peano}}$ & 87.13 &	95.93 &	87.71 &	75.56 \\
$\mathbf{M}_{\text{Random}}$ & 86.8 &	95.57 &	87.63 &	75.26 \\
\noalign{\hrule height 1.5pt}
\end{tabular}}
\end{table}

\subsection{Learned curve order}\label{ap:learned}
Motivated by recent work on learned patch orderings~\citep{kutscher2025REOrder}, we implemented a learned ordering variant within our framework and trained a DeiT-Tiny model using this learned sequence. The results are shown in \cref{tab:learned_order}. Although the learned variant underperforms the original VIOLIN mask in this initial experiment, it highlights several promising research directions, such as jointly learning multiple traversal curves, exploring task-adaptive orderings, and studying how different datasets induce specialized spatial structures—all of which may further improve performance and interpretability.
\begin{table}[h!]
\centering
\caption{\textit{Comparison of DeiT-Tiny, VIOLIN, and a learned patch-ordering variant:} learned patch orderings~\citep{kutscher2025REOrder} is adapted to \ours framework.}
{%
\begin{tabular}{lc}
\noalign{\hrule height 1.5pt}
Model & Accuracy ($\%$) \\
\midrule
\rowcolor{gray!20}DeiT-T & 72.2 \\
\rowcolor{orange!30}\ours & 73.0 \\
\ours w learned order & 70.1 \\
\noalign{\hrule height 1.5pt}
\end{tabular}}
\label{tab:learned_order}
\end{table}

\subsection{Comparison with relative positional encodings in pretraining}\label{ap:rpe_comp}

\ours and relative positional encodings (RPEs) introduce spatial inductive bias through different mechanisms. As described in \cref{ap:relpos}, \ours applies a lightweight multiplicative decay mask, whereas modern RPEs add learned pairwise positional terms to the attention logits and often require additional parameters or architecture-specific modifications. To assess their relationship, in addition to the fine-tuning experiments in \cref{ap:locality}, we include comparisons with several RPE-based locality baselines in both the pretraining settings.

On ImageNet-1K supervised pretraining, VIOLIN achieves competitive performance to several RPE variants while adding significantly fewer FLOPs. For example, on DeiT-S, \ours introduces 5$\times$ fewer FLOPs than Transformer-XL and 1.3$\times$ fewer FLOPs than iRPE-QK, while obtaining comparable accuracy.

\begin{table}[h]
\centering
\caption{\textit{Comparison of \ours and RPE variants:} on DeiT-S pretraining in ImageNet-1K. Results are taken from respective papers of i-RPE~\citep{wu2021rethinking} and Transformer-XL~\citep{dai2019transformerxl}.}
{%
\begin{tabular}{lcc}
\noalign{\hrule height 1.5pt}
Model & {Additional FLOPs (\%)} & {Top-1 Acc. (\%)} \\
\midrule
\rowcolor{gray!20} DeiT-S       & -   & 79.9 \\
\rowcolor{orange!30}\ours & \textbf{0.7}  & 80.7 \\
Transformer-XL & 4.3  & 80.8 \\
iRPE-K         & 0.9  & \textbf{80.9} \\
iRPE-QK        & 2.2  & 81.1 \\
iRPE-QKV       & 5.9  & 81.4 \\
\noalign{\hrule height 1.5pt}
\end{tabular}}
\label{tab:rpe_pretrain}
\end{table}

\ours can also be combined with RPEs. On DeiT-T, adding VIOLIN to iRPE-K yields an additional accuracy gain, indicating that the methods introduce complementary inductive information.

\begin{table}[h]
\centering
\caption{\textit{Combination of \ours with RPEs:} pretraining results on DeiT-T model as baseline, with PRE and with RPE+\ours.}
{%
\begin{tabular}{lcc}
\noalign{\hrule height 1.5pt}
{Model} & {Additional FLOPs (\%)} & {Top-1 Acc. (\%)} \\
\midrule
\rowcolor{gray!20} DeiT-T            & --   & 72.2 \\
iRPE-K              & 1.7  & 73.7 \\
\rowcolor{orange!30}iRPE-K + \ours & 2.3  & \textbf{73.9} \\
\noalign{\hrule height 1.5pt}
\end{tabular}}
\label{tab:rpe_combine}
\end{table}

\newpage

\section{Codes and implementation details} \label{ap:codes}
\subsection{Compute resources} \label{ap:compute}
\begin{table}[h]
\centering
\caption{\textit{Compute resources for pertaining:} The number of GPUS and approximate training time for each model and scale are provided.}
\label{tab:compute}
\begin{tabular}{lcc}
\noalign{\hrule height 1.5pt}
Model & \# GPUs & Training time \\
\midrule
DeiT-T & 4 & $\approx 17 \text{ Hour}$\\
DeiT-S & 4 & $\approx 23 \text{ Hour}$ \\
DeiT-B & 16 & $\approx 1.7 \text{ Day}$\\
\hline
DINO-S & 16  & $\approx 3.2 \text{ Days}$\\
DINO-B & 16 &  $\approx 7 \text{ Days}$\\
\noalign{\hrule height 1.5pt}
\end{tabular}
\end{table}

In \cref{tab:compute}, we report the compute resources required for each of the evaluated models. These numbers also apply to the models used for ablation experiments.

For fine-tuning, we performed 30 runs per dataset for each model (25 for validation and 5 for final evaluation). Each run took between 2 to 10 minutes, and the complete fine-tuning evaluation was completed in approximately 10 days.

All experiments were conducted using a mix of NVIDIA A100 SXM4 80GB, NVIDIA GH200 96GB, and NVIDIA H100 SXM5 80GB GPUs, used interchangeably depending on availability.

\subsection{VTAB-1K hyperparameters} \label[appendix]{ap:vtab_hyper}
To determine optimal learning rates, we use the VTAB-1K-pytorch repository~\citep{vtab1k-pytorch} and conduct a grid search. Following the original implementation, every dataset is first split into a $800/200$ train/validation partition to select the optimal learning rate per dataset using 5 seeds. We then train on the full dataset using 5 random seeds. For each model, we average the top 3 runs to report the final accuracy. The complete list of hyperparameters is provided in \cref{tab:vtab_hyper}. For parameter-efficient fine-tuning, we again use the same set of hyperparameters and grid search over ranks [2,4,8,16].

\begin{table}[h]
\begin{center}
\caption{\textit{Hyperparameters for fine-tuning on VTAB-1K}: The same hyperparameters are used for all models, following~\citep{vtab1k-pytorch}.}
\label{tab:vtab_hyper}
\begin{tabular}{lc}
\noalign{\hrule height 1.5pt}
Parameter & Value  \\
\hline
Epochs & 50 \\
Batch size & 64 \\
Seeds & 5 \\
Optimizer & AdamW \\
\quad Learning rate & [1e-3, 7.5e-4, 5.0e-4, 2.5e-4, 1.0e-4] \\
\quad Layer-wise lr deca & 0.65* \\
\quad Weight decay & 0.05\\
\quad Momentum & $\beta_1=0.9,\beta_2=0.999$ \\
Learning rate schedule & linear warmup $\rightarrow$ cosine decay \\
\quad Warmup epochs & 5 \\
Precision & mixed \texttt{bfloat16} \\
\quad Backend & \texttt{torch.autocast} \\
Data Augmentation & \\
\quad \tt{Resize} &  \\
\qquad \tt{interpolation} & bicubic \\
\qquad \tt{size} & 224x224 \\
\quad \tt{Normalize} & ImageNet-1K statistics \\
\noalign{\hrule height 1.5pt}
\end{tabular}
\end{center}
\end{table}

\subsection{Codes for curves} \label[appendix]{curve_codes}

In this section, we provide the codes used to create the permutation orders of each SFC in basis of Z-curve. In other words, we define efficiency the indexing needed for the permutation $\pi_c(.)$ for each curve $c$ used in our study. \\

\textbf{Snake curve} \\ \\
\hspace{5mm}\begin{minipage}[h]{0.95\columnwidth} 
    \centering
    \begin{python}[framerule=0.0
    mm , rulecolor=\color{black} ,frame=single , backgroundcolor = \color{cyan!10} ]
def snake_curve(grid):
    """Returns the elements of the grid in snake order."""
     n_rows, n_cols = grid.shape
    order = []
    for y in range(n_rows):
        if y %
            # Left-to-right for even rows
            order.extend((x, y) for x in range(n_cols))
        else:
            # Right-to-left for odd rows
            order.extend((x, y) for x in reversed(range(n_cols)))
    return order
\end{python}
\end{minipage}

\textbf{Zig-zag curve} \\ \\
\hspace{5mm}
\begin{minipage}[h]{0.95\columnwidth} 
    \centering
    \begin{python}[framerule=0.0
    mm , rulecolor=\color{black} ,frame=single , backgroundcolor = \color{green!10}]
def zigzag_curve(grid):
    """Returns the elements of the grid in diagonal zig-zag order."""
    n_rows, n_cols = grid.shape
    order = []
    for d in range(n_rows + n_cols - 1):
        if d %
            r = min(d, n_rows - 1)
            c = d - r
            while r >= 0 and c < n_cols:
                order.append((r, c))
                r -= 1
                c += 1
        else:
            c = min(d, n_cols - 1)
            r = d - c
            while c >= 0 and r < n_rows:
                order.append((r, c))
                c -= 1
                r += 1
    return order
\end{python}
\end{minipage}

\newpage
\textbf{Hilbert curve } Adapted from~\citep{gilbert}. \\ \\
\hspace{5mm}\begin{minipage}[h]{0.95\columnwidth} 
    \centering
    \begin{python}[framerule=0.0
    mm , rulecolor=\color{black} ,frame=single , backgroundcolor = \color{violet!7}]
def hilbert_curve(grid):
    rows = len(grid)
    cols = len(grid[0]) if rows > 0 else 0
    return [(x, y) for x,y in gilbert2d(rows, cols)]
    
def gilbert2d(width, height):
    """
    Generalized Hilbert ('gilbert') space-filling curve for arbitrary-sized
    2D rectangular grids. Generates discrete 2D coordinates to fill a rectangle
    of size (width x height).
    """
    if width >= height:
        yield from generate2d(0, 0, width, 0, 0, height)
    else:
        yield from generate2d(0, 0, 0, height, width, 0)

def sgn(x):
    return -1 if x < 0 else (1 if x > 0 else 0)

def generate2d(x, y, ax, ay, bx, by):
    w = abs(ax + ay)
    h = abs(bx + by)
    (dax, day) = (sgn(ax), sgn(ay)) # unit major direction
    (dbx, dby) = (sgn(bx), sgn(by)) # unit orthogonal direction
    if h == 1:
        # trivial row fill
        for i in range(0, w):
            yield(x, y)
            (x, y) = (x + dax, y + day)
        return
    if w == 1:
        # trivial column fill
        for i in range(0, h):
            yield(x, y)
            (x, y) = (x + dbx, y + dby)
        return
    (ax2, ay2) = (ax//2, ay//2)
    (bx2, by2) = (bx//2, by//2)
    w2 = abs(ax2 + ay2)
    h2 = abs(bx2 + by2)
    if 2*w > 3*h:
        if (w2 %
            # prefer even steps
            (ax2, ay2) = (ax2 + dax, ay2 + day)
        # long case: split in two parts only
        yield from generate2d(x, y, ax2, ay2, bx, by)
        yield from generate2d(x+ax2, y+ay2, ax-ax2, ay-ay2, bx, by)
    else:
        if (h2 %
            # prefer even steps
            (bx2, by2) = (bx2 + dbx, by2 + dby)
        # standard case: one step up, one long horizontal, one step down
        yield from generate2d(x, y, bx2, by2, ax2, ay2)
        yield from generate2d(x+bx2, y+by2, ax, ay, bx-bx2, by-by2)
        yield from generate2d(x+(ax-dax)+(bx2-dbx), y+(ay-day)+(by2-dby),
                              -bx2, -by2, -(ax-ax2), -(ay-ay2))
\end{python}
\end{minipage}
\newpage
\textbf{Peano curve } Adapted from~\citep{rmrschub_2021_zCurve,prater_pymorton}. \\ \\
\hspace{5mm}\begin{minipage}[h]{0.95\columnwidth} 
    \centering
    \begin{python}[framerule=0.0
    mm , rulecolor=\color{black} ,frame=single , backgroundcolor = \color{red!7}]
def interleave_bits(x, y):
    """
    Interleave the bits of two integers (x, y) to compute Morton order.
    """
    def split_bits(value):
        result = 0
        for i in range(32):  # Support up to 32-bit integers
            result |= ((value >> i) & 1) << (2 * i)
        return result

    return split_bits(x) | (split_bits(y) << 1)

def peano_curve(grid):
    """Returns the elements of the grid in diagonal morton/peano order."""
    n_rows, n_cols = grid.shape
    order = []

    for y in range(n_rows):
        for x in range(n_cols):
            morton_key = interleave_bits(x, y)
            order.append((morton_key, x, y))

    # Sort by Morton key to achieve the Morton curve order
    order.sort(key=lambda pair: pair[0])
    return [(x, y) for _, x, y in order]
\end{python}
\end{minipage}

\subsection{Code of efficient decay mask} \label{ap:decay_code}
\hspace{5mm}\begin{minipage}[h]{0.95\columnwidth} 
    \centering
    \begin{python}[framerule=0.0
    mm , rulecolor=\color{black} ,frame=single , backgroundcolor = \color{gray!7}]
def Casual_Decay_Mask(b_i , N):
    idx = torch.arange(N,device=b_i.device)
    I, J = torch.meshgrid(idx, idx, indexing='ij')
    E = (torch.abs((I-J)).float().view(1,1,N,N))
    M = torch.sigmoid(b_i).view(1,-1,1,1)**E
    return M
\end{python}
\end{minipage}

\end{document}